\documentclass[lettersize,journal]{IEEEtran}

\usepackage{textcomp}
\usepackage{stfloats}
\usepackage{url}
\usepackage{verbatim}
\usepackage{cite}
\usepackage{amsmath,amsfonts,amssymb,amsthm,graphicx,booktabs,caption}
\usepackage{algorithm,algorithmic}
\usepackage{enumitem}

\usepackage{hyperref}
\usepackage{multirow}
\usepackage{makecell}
\newtheorem{proposition}{Proposition}
\usepackage{xcolor}
\usepackage{tabularx}
\newtheorem{theorem}{Theorem}[section]
\newtheorem{lemma}{Lemma}[section]
\usepackage{array}

\usepackage{placeins}
\begin{document}

\title{Neural Uncertainty Principle: A Unified View of Adversarial Fragility and LLM Hallucination}

\author{Dong-Xiao~Zhang,~Hu~Lou,~Jun-Jie~Zhang,~Jun~Zhu,~Deyu~Meng%
\thanks{Dong-Xiao Zhang, Hu Lou, and Jun-Jie Zhang contributed equally to this work.}%
\thanks{Corresponding authors: Jun-Jie Zhang and Deyu Meng (e-mail: zjacob@mail.ustc.edu.cn; dymeng@mail.xjtu.edu.cn).}%
\thanks{Dong-Xiao Zhang, Hu Lou, and Jun-Jie Zhang are with the Northwest Institute of Nuclear Technology, Xi'an, Shaanxi 710024, China.}%
\thanks{Jun Zhu is with Department of Computer Science and Technology, Institute for AI, BNRist Center, THBI Lab, Tsinghua-Bosch Joint Center for ML, Tsinghua University, Beijing 10084, China (e-mail: dcszj@mail.tsinghua.edu.cn).}%
\thanks{Deyu Meng is with the School of Mathematics and Statistics and the Ministry of Education Key Lab of Intelligent Networks and Network Security, Xi'an Jiaotong University, Xi'an, Shaanxi 710049, China.}%
}



\maketitle

\begin{abstract}
Adversarial vulnerability in vision and hallucination in large language models are conventionally viewed as separate problems, each addressed with modality-specific patches. Here we argue that they can be interpreted through a common geometric mechanism: under a loss-induced weighting that emphasizes boundary-relevant samples, the input projection and its directional loss gradient form a conjugate pair whose dispersions obey a Robertson--Schr\"odinger-type lower bound---a Neural Uncertainty Principle (NUP).
This bound implies that a model cannot simultaneously achieve arbitrarily sharp boundary discrimination and uniformly low sensitivity to input perturbations. Vision adversarial fragility corresponds to a near-bound regime where boundary compression inflates gradient dispersion; LLM hallucination corresponds to the opposite extreme, where weak prompt--gradient coupling leaves the continuation space under-constrained before decoding. The bound is modulated by an input--gradient correlation channel that admits exact reduction to real-gradient statistics, yielding a single-backward probe (CC-Probe) applicable to both modalities. In vision, suppressing dominant coupling components during training improves robustness without adversarial training; in language, the same prefill-stage probe detects hallucination risk before generating any answer tokens. NUP thus reframes two isolated reliability failures as opposite sides of a single conjugate trade-off.
\end{abstract}


\section{Introduction} 
\label{sec:intro}

Modern neural systems have evolved far beyond their initial roles as pattern recognizers. They have become integral infrastructure in scientific discovery, industrial automation, and daily digital interaction~\cite{krizhevsky2012imagenet,simonyan2014vgg,he2016resnet,vaswani2017attention,brown2020gpt3,radford2018gpt,radford2019gpt2,kaplan2020scaling}. As these models extend across critical domains, from predicting protein structures to generating legal or medical advice, the requirement to understand their \emph{boundary behaviors} shifts from a theoretical curiosity to an engineering necessity~\cite{Yang2022benchmark, hendrycks2021robustness, hendrycks2020pretrained}. This fragility is most clearly manifested in two distinct yet pervasive phenomena: \emph{adversarial vulnerability} in vision models, where imperceptible perturbations flip confident predictions~\cite{szegedy2013intriguing,goodfellow2014explaining,madry2017towards,moosavi2016deepfool,carlini2017towards,athalye2018obfuscated,croce2020autoattack,croce2021robustbench}, and \emph{hallucination}, especially in large language models (LLMs), where generation drifts fluently from fact to fabrication~\cite{maynez2020faithfulness, truthfulqa2021,rawte2023siren,farquhar2024semantic,sahoo2024survey,kossen2024cheap,OpenAI2025WhyLMHallucinate,KarbasiEtAl2025Impossible}.

Currently, the dominant responses to these failures have been domain- and modality-specific. 
\textit{Adversarial vulnerability} is typically formalized as worst-case risk under norm-bounded perturbations, and the dominant mitigation is adversarial training (AT) and its variants (e.g., TRADES), which improve robustness but are computationally expensive and tightly coupled to specific threat models~\cite{madry2017towards,zhang2019trades}.
\textit{Hallucination in LLMs} is usually framed as a faithfulness/factuality failure in generation.
Mitigation is commonly pursued through alignment and instruction tuning~\cite{ouyang2022training},
retrieval-augmented generation (RAG) that augments parametric memory with external evidence~\cite{lewis2020rag},
and post-hoc verification/self-check pipelines~\cite{manakul-etal-2023-selfcheckgpt,gao-etal-2023-rarr,dhuliawala-etal-2024-chain}.

Overall, current practice still largely reflects a patchwork of modality-specific solutions, rather than converging toward a unified \emph{observable} of boundary behavior that is comparable across perception and generation.

Here, we take a unified view and argue that both failure modes share a geometric signature: the model enters a boundary-stressed region of the loss landscape.
In adversarial attacks, small input perturbations can traverse locally sharp loss contours; in LLMs, weak prompt conditioning leaves a large set of continuations similarly compatible \emph{before decoding}, enabling prior-driven drift.
We show that these are two opposite ways of mismanaging the same ``uncertainty budget''—an imbalance between input-space localization and gradient-space sensitivity.

In summary, this paper mainly makes the following contributions:
\begin{itemize}[leftmargin=*]
\item \textbf{Neural Uncertainty Principle (NUP).}
We formalize a Neural Uncertainty Principle (NUP) for neural models with a well-defined scalar loss: under a boundary-emphasized (loss-induced) state, an input projection operator and its directional-derivative conjugate obey a Robertson--Schr\"odinger--type constraint (Theorem~\ref{thm:rs_form_of_nup}).
The constraint captures a principled trade-off: a model cannot be made simultaneously arbitrarily accurate on boundary-relevant samples and uniformly robust to small input perturbations.

\item \textbf{From operator constraint to a computable correlation channel.}
A key technical contribution is that, under the loss-phase construction, the covariance term in the Robertson–Schrödinger (RS) inequality admits an exact reduction to loss-weighted statistics of the \emph{real} loss gradient field (Theorem~\ref{thm:cov_reduction}, Lemma~\ref{lem:cos_dir_corr}).
This reduction exposes an explicit \emph{input--gradient coupling} channel that controls proximity to the constraint bound, and yields a practical single-backward per-sample proxy:
the \textit{Conjugate Correlation Probe (CC-Probe)}, defined as the absolute input--gradient cosine (mean-centered for prompt embeddings).

\item \textbf{Two opposite anomaly regimes under the same principle.}
Using the CC-Probe, we can empirically characterize two opposite boundary-anomaly regimes:
in vision, a persistent \emph{high}-CC-Probe marks boundary-stress and concentrates adversarial fragility;
in LLM prompting, \emph{anomalously low} CC-Probe marks under-conditioning at prefill and correlates with elevated hallucination risk.
Across settings, reliable behavior typically lies between these extremes, which we summarize as an intermediate ``Goldilocks'' band (Fig.~\ref{fig:nup_roadmap}) used as a qualitative regime indicator.
\end{itemize}

To validate the aforementioned results, we design six experiments that diagnose the two regimes and use two lightweight interventions, ConjMask (masking high-contribution input components) and LogitReg (logit-side regularization), as \emph{mechanistic tests} that manipulate the coupling channel and produce objective-specific shifts under controlled stress tests.

Our goal is to establish NUP as a unified and \emph{testable} lens for boundary anomalies, by turning its correlation channel into a single-backward observable and validating the resulting regime predictions with controlled mechanistic interventions, rather than simply optimizing an end-to-end adversarial defense or hallucination detector for leaderboards.

\section{Related Work}
\label{sec:related}

\subsection{Adversarial Robustness and Accuracy–Robustness Frontier}

Deep neural networks achieve remarkable accuracy on clean data~\cite{krizhevsky2012imagenet, he2016resnet, vaswani2017attention}, but remain vulnerable to imperceptible adversarial perturbations~\cite{szegedy2013intriguing, goodfellow2014explaining}. 
In robustness evaluation, a standard first-order white-box adversary is \emph{Projected Gradient Descent (PGD)}~\cite{madry2017towards}, which iteratively maximizes the training loss---typically the \emph{cross-entropy (CE)} loss for classification.
More recently, standardized suites such as AutoAttack~\cite{croce2020autoattack} include Auto-PGD (APGD) variants that optimize either CE (\emph{APGD-CE}) or the \emph{Difference of Logits Ratio (DLR)} loss (\emph{APGD-DLR}), offering complementary stress tests beyond plain CE optimization.

Adversarial training (AT) formulates robustness as a min--max optimization problem~\cite{madry2017towards}, and TRADES~\cite{zhang2019trades} further decomposes the objective to balance natural and robust accuracy. Despite progress, a persistent accuracy--robustness tension remains, attributed variously to sample complexity~\cite{schmidt2018robustgeneralization}, non-robust features~\cite{tsipras2018robustness}, and intrinsic margin constraints~\cite{fawzi2016robustness, stutz2019disentangling}. Standardized benchmarks such as RobustBench~\cite{croce2021robustbench} and strong attacks like AutoAttack~\cite{croce2020autoattack} now enable rigorous evaluation.

Our work offers a complementary perspective: Proposition~\ref{prop:exp1_sep} frames this tension as a \emph{conjugate trade-off} arising from the NUP. As models compress features to sharpen decision boundaries, the effective feasible volume shrinks, pushing the system toward bound saturation.

\subsection{Gradient-Based Attribution, Regularization, and Masking}

Gradient-based attribution methods explain model predictions by analyzing input sensitivities. Early saliency maps~\cite{simonyan2014saliency} visualize $\nabla_x f$, while Gradient$\times$Input (GxI)~\cite{shrikumar2017deeplift, ancona2018unified} weights gradients by input magnitudes—a quantity superficially similar to our probe (presented in Sec.~\ref{sec:nup_formalism}). Integrated Gradients~\cite{sundararajan2017integrated} and SmoothGrad~\cite{smilkov2017smoothgrad} improve attribution stability. Recent work connects gradient geometry to robustness: Input-Gradient Regularization~\cite{ross2018inputgradreg} penalizes gradient norms, while Perceptually Aligned Gradients~\cite{ganz2023pag} show that robust models exhibit human-interpretable gradients.

Masking-based regularization offers another avenue. Dropout~\cite{srivastava2014dropout} and DropBlock~\cite{ghiasi2018dropblock} randomly mask activations to prevent co-adaptation, while adversarial masking strategies~\cite{chen2021adversarialmask} target vulnerable features. Our \emph{ConjMask} (presented in Sec.~\ref{sec:exp3_results}) differs fundamentally: rather than random or adversarially-driven selection, we mask input components with large normalized interaction scores
$|\tilde x_{c,j}\tilde p_{c,j}|$ (Eq.~\eqref{eq:exp3_score}),
i.e., a channel-wise normalized version of the $x_i p_i$-type coupling, directly targeting the input--gradient coupling channel that appears in the NUP via the correlation coefficient $\rho_c(u)$ (Eq.~\eqref{eq:rho_formalism}) and the covariance reduction (Theorem~\ref{thm:cov_reduction}).
Meanwhile, prior GxI-based methods use $x \odot \nabla_x \mathcal{L}$ for \emph{post-hoc explanation}; we instead treat $|\cos(x,p)|$ as a \emph{predictive state variable} motivated by the NUP (Theorem~\ref{thm:rs_form_of_nup}) and its mixed-axis corollary (Lemma~\ref{lem:mixed_axis_nup}).

\subsection{Uncertainty Quantification and Conjugate Perspectives}

Predictive uncertainty estimation enables models to flag unreliable outputs. Deep Ensembles~\cite{lakshminarayanan2017ensembles} remain a strong baseline, while Monte Carlo Dropout~\cite{gal2016dropout} provides a Bayesian approximation. Calibration techniques~\cite{guo2017calibration} align confidence with accuracy, critical for out-of-distribution detection~\cite{hendrycks2016baseline, Yang2022benchmark}. These methods typically require multiple forward passes.

A separate line explores trade-off structures echoing conjugate relationships. The Information Bottleneck~\cite{tishby2015deep} formalizes compression--prediction trade-offs in representation space; Spectral Normalization~\cite{miyato2018spectral} and gradient penalties~\cite{gulrajani2017improved} implicitly constrain sensitivity dispersion. Logit regularization and label smoothing~\cite{szegedy2016rethinking, muller2019labelsmoothing} stabilize output distributions against overconfidence.

Our NUP framework differs in grounding uncertainty in \emph{input-space conjugate geometry} rather than representation-level compression or output-level statistics. The Robertson--Schr\"odinger form (Theorem~\ref{thm:rs_form_of_nup}) yields a single-backward probe that complements expensive sampling-based estimators.

\subsection{Hallucination Detection and Mitigation in LLMs}

Large language models generate fluent but factually unsupported text—a phenomenon termed hallucination~\cite{maynez2020faithfulness, truthfulqa2021, sahoo2024survey}. Recent theoretical analyses suggest hallucinations may be structurally unavoidable under realistic conditions~\cite{OpenAI2025WhyLMHallucinate, KarbasiEtAl2025Impossible}.

Detection approaches are predominantly \emph{post-hoc} and \emph{sampling-based}. Semantic entropy methods~\cite{farquhar2024semantic, kossen2024cheap} estimate uncertainty via multiple generations, while LLM-as-a-judge~\cite{zheng2024judging} uses auxiliary models to evaluate outputs. Prompt engineering strategies such as Chain-of-Thought~\cite{wei2022chain, kojima2022cot} reshape the generation process itself. Self-consistency decoding~\cite{wang2023selfconsistency} aggregates multiple reasoning paths but requires substantial sampling overhead.

These methods share a common limitation: they operate \emph{during or after decoding}. Early-exit mechanisms~\cite{schuster2022earlyexit} and confidence-based filtering~\cite{kadavath2022languagemodelconfidence} reduce inference cost but still require partial generation. In contrast, our prefill-stage probe (presented in Sec.~\ref{sec:exp5}--\ref{sec:exp6}) computes $|\cos(\bar{x}, \bar{p})|$ on prompt embeddings \emph{before generating any answer (output) tokens}, enabling decoding-free and sampling-free risk assessment. Proposition~\ref{prop:exp56_llm} interprets this through the NUP: anomalously low correlation indicates a high-slack regime (i.e., $S_c(u)\gg 1$ in Eq.~\eqref{eq:slack_formalism_repl}), where the prompt fails to constrain the feasible continuation space, amplifying prior-driven drift.

\subsection{Positioning of This Work}
\label{sec:related:positioning}

Existing approaches address adversarial robustness and hallucination 
as separate phenomena, employing modality-specific defenses 
(adversarial training, prompt engineering) or post-hoc diagnostics 
(semantic entropy, LLM-as-judge). Our work departs from this paradigm 
in the following ways:

\begin{enumerate}[label=(\roman*),nosep]
\item \textbf{Unified mechanism.} We ground both failure modes in a 
single operator-theoretic constraint (Theorem~\ref{thm:rs_form_of_nup}), 
where the input--gradient correlation \(\rho_c(u)\) governs proximity 
to the uncertainty bound.

\item \textbf{Prefill signal.} Unlike sampling-based uncertainty 
estimators, the CC-Probe requires only a single backward pass and 
operates \emph{before} decoding (for LLMs) or attack evaluation 
(for vision).

\item \textbf{Intervention design.} Rather than adversarially 
augmenting the training set, we directly target the conjugate coupling channel implicated by the NUP in boundary fragility.
\end{enumerate}

\section{Formalism of NUP}
\label{sec:nup_formalism}

% =========================
% Nutshell Roadmap (Theory -> Experiments)
% Put this at the beginning of the theory section (right before formal operator derivations).
% =========================
\subsection{NUP in a Nutshell: Roadmap From Geometry to Experiments}
\label{sec:nup_nutshell}

\begin{figure}[t]
    \centering
    % rename the file as you like, e.g., ./Plots/nup_roadmap.png
    \includegraphics[width=0.5\textwidth]{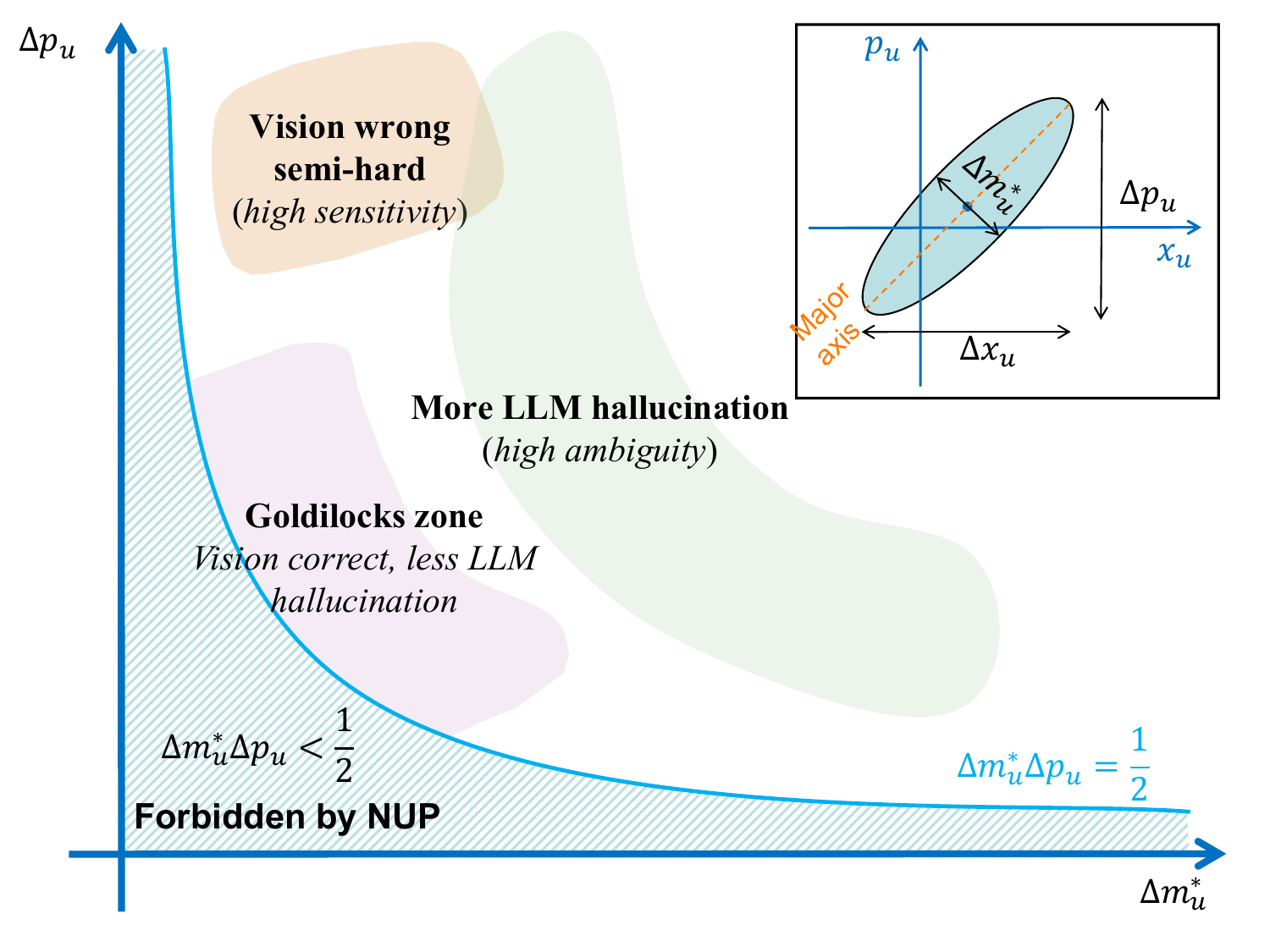}
    \vspace{-2mm}
    \caption{\textbf{Illustration of the Neural Uncertainty Principle (NUP) in the $(\Delta \hat m_u^\star,\Delta \hat p_u)$ plane.}
    Here $\Delta \hat m_u^\star$ is the \emph{minimum} dispersion achievable by a mixed observable
    $\hat m_u(\lambda)=\hat x_u+\lambda \hat p_u$, and $\Delta \hat p_u$ is the dispersion of the conjugate operator $\hat p_u$
    under the loss-phase (boundary-emphasized) state.
    NUP implies the forbidden region $\Delta \hat m_u^\star \Delta \hat p_u < \tfrac{1}{2}$.
    Vision boundary-stress concentrates toward the upper-left (small $\Delta \hat m_u^\star$, large $\Delta \hat p_u$),
    whereas LLM under-conditioning / hallucination occupies a larger region away from the uncertainty boundary (large $\Delta \hat m_u^\star$, uncontrolled $\Delta \hat p_u$).
    Reliable behavior lies in an intermediate ``Goldilocks'' band where neither term is extreme. 
    \textbf{Note.} Strictly speaking, $\Delta \hat m_u^\star$ and $\Delta \hat p_u$ denote the dispersions of the operators $\hat m_u$ and $\hat p_u$ under the loss-induced state (hence the ``hat'' notation at the operator level). In the figure, to emphasize the phase-plane geometry and the experimentally observable quantities, we plot numerical estimates of these dispersions. We therefore omit the hats in the figure labels for notational simplicity; the plotted quantities correspond one-to-one to $\Delta \hat m_u^\star$ and $\Delta \hat p_u$ as defined in the equations.
}
    \label{fig:nup_roadmap}
\end{figure}

\vspace{2pt}
\noindent\textbf{Phase-plane geometry under a loss-induced state.}
Fig.~\ref{fig:nup_roadmap} provides the minimal geometric picture needed for the rest of the paper.
Let $x$ and $\mathcal L_c(x)$ denote the input and the per-sample loss under condition $c$
(e.g., the cross-entropy loss for a ground-truth class in vision, or a shifted negative log-likelihood (NLL) objective in language).
We study one-dimensional projections along a unit direction $u$:
$x_u(x)=u^\top x$ and $p_u(x)=\partial_u \mathcal L_c(x)$.
Under the loss-phase construction (Sec.~\ref{subsec:state_and_ops}), loss-weighted samples form an effective
``boundary-relevant'' population in the $(x_u,p_u)$ plane whose second-order structure is summarized by an ellipse (the upper right corner of Fig.~\ref{fig:nup_roadmap}).
Two dispersions from this ellipse are the only quantities we need in this roadmap\footnote{$\hat O$ represents an ``operator'', a rule that takes a function $\psi(x)$ and produces another function $(\hat O\psi)(x)$. Here $\hat p_u$ and $\hat x_u$ correspond to the operators induced by $p_u$ and $x_u$ respectively. Please see more details in Sec.~\ref{subsec:state_and_ops}.}:

\begin{itemize}[leftmargin=1.2em,itemsep=1pt,topsep=2pt]
\item $\Delta \hat p_u$ (\textbf{sensitivity dispersion}): an operator dispersion whose real drift term is driven by the loss gradient. Operationally, it quantifies how widely the boundary-emphasized population spreads in
loss-sensitivity: larger $\Delta \hat p_u$ indicates that small input perturbations can more frequently induce large loss changes,
which is precisely the regime exploited by gradient-based attacks. 

\item $\Delta \hat m_u^\star$ (\textbf{best mixed-axis thickness}): define the mixed observable
$\hat m_u(\lambda)=\hat x_u+\lambda \hat p_u$ and let $\Delta \hat m_u(\lambda)$ be its dispersion under the same state.
Then $\Delta \hat m_u^\star:=\min_{\lambda}\Delta \hat m_u(\lambda)$ is the \emph{minimum} achievable thickness of the
boundary-emphasized ellipse along any mixed (tilted) axis. Intuitively, it measures how \emph{consistent} the boundary layer is across samples:
a smaller $\Delta \hat m_u^\star$ means the hard population collapses into a thinner ambiguity band in the phase plane,
which typically accompanies higher clean accuracy but also sharper boundary behavior.
\end{itemize}

\vspace{2pt}
\noindent\textbf{Mixed-axis NUP constraint.}
The full NUP is a Robertson--Schr\"odinger uncertainty relation for the operator pair $(\hat x_u,\hat p_u)$ (Theorem~\ref{thm:rs_form_of_nup}).
A key corollary (Lemma~\ref{lem:mixed_axis_nup}) eliminates the covariance term by passing to the optimal mixed axis and yields
\begin{equation}
\label{eq:nutshell_nup}
\Delta \hat m_u^\star\, \Delta \hat p_u \;\ge\; \tfrac{1}{2}.
\end{equation}
This inequality is the only constraint required to interpret Fig.~\ref{fig:nup_roadmap}:
the lower-left region is \textbf{forbidden}, i.e., one cannot simultaneously make the boundary-emphasized population
\emph{extremely thin} (very small $\Delta \hat m_u^\star$) and \emph{uniformly insensitive} (very small $\Delta \hat p_u$).
This provides the geometric intuition through which NUP relates adversarial vulnerability and hallucination.

\vspace{2pt}
\noindent\textbf{Two opposite failure modes on the same plane.}
The same $(\Delta \hat m_u^\star,\Delta \hat p_u)$ plane explains two ``boundary anomalies'' as opposite corners:

\paragraph{Vision boundary-stress}
Discriminative training tends to squeeze errors and near-margin samples into a narrow boundary layer,
which corresponds to pushing $\Delta \hat m_u^\star$ downward.
Based on Eq. \eqref{eq:nutshell_nup}, we know that this compression must be compensated by inflating $\Delta \hat p_u$:
the hard subset becomes increasingly sensitive.
Empirically, this is where \emph{wrong / semi-hard} samples concentrate---they are highly correctable by small perturbations,
hence adversarially fragile.

\paragraph{LLM under-conditioning / hallucination}
In prompting, a common failure is not an overly sharp boundary but an under-constrained continuation space.
If the prompt does not sufficiently couple to loss-sensitive directions at prefill, the effective phase-plane cloud is thick
(large $\Delta \hat m_u^\star$) while sensitivity dispersion is unbounded (uncontrolled $\Delta \hat p_u$),
corresponding to high slack: many continuations are similarly compatible before decoding begins.
This ``under-conditioning'' corner aligns with elevated hallucination risk, where generation can drift into fluent but weakly grounded outputs.

\paragraph{Intermediate ``Goldilocks'' zone}
Reliable behavior lies between extremes: $\Delta \hat m_u^\star$ is not excessively squeezed (avoiding boundary-stress),
and $\Delta \hat p_u$ is not vanishingly small under weak conditioning (avoiding unconstrained drift).
We will repeatedly interpret experiments as moving populations toward this intermediate band in Fig.~\ref{fig:nup_roadmap}.

\vspace{2pt}
\noindent\textbf{Practical probe.}
The Robertson--Schr\"odinger form contains an explicit correlation term $\rho_c(u)$ and satisfies
$\Delta \hat m_u^\star=\Delta \hat x_u\sqrt{1-\rho_c(u)^2}$ (Eq.~\ref{eq:min_var_m_formalism}), where
$\rho$ controls the \emph{tilt/coupling} of the ellipse,
while the primary ``budget'' relevant to failure modes is captured by $(\Delta \hat m_u^\star,\Delta \hat p_u)$ through Eq.
\eqref{eq:nutshell_nup}.

To localize \emph{individual samples} on this plane without estimating operator moments, we use a single-backward proxy for the
coupling channel: \textit{the input--gradient cosine (CC-Probe)}, i.e.,
$c_{\text{img}}$ in vision and $c_{\text{prompt}}$ in LLM prefill
(Eqs.~\eqref{eq:c_img}--\eqref{eq:c_prompt}).
It tracks the same input--gradient coupling that enters the Robertson--Schr\"odinger covariance term under the loss-phase construction,
and empirically identifies boundary-stress in vision and under-conditioning prompts in language.

\vspace{2pt}
\noindent\textbf{Preview our experiment objectives using Fig.~\ref{fig:nup_roadmap}.}
\begin{itemize}[leftmargin=1.2em,itemsep=1pt,topsep=2pt]
\item \textbf{Exp.~1--2 (diagnosis):} CC-Probe separates easy/correct vs.\ wrong/semi-hard samples in vision and responds causally to $\pm$FGSM (Fast Gradient Sign Method \cite{goodfellow2014explaining}),
consistent with moving subsets toward the intermediate band in Fig.~\ref{fig:nup_roadmap}.

\item \textbf{Exp.~3--4 (vision intervention):} ConjMask suppresses dominant $|x_i p_i|$ couplings during training (relieving boundary stress);
LogitReg complements it against score-space attacks. Net effect: shift the stressed population away from the upper-left corner.

\item \textbf{Exp.~5--6 (LLM intervention):} low $c_{\text{prompt}}$ indicates under-conditioning at prefill; selecting higher-$c_{\text{prompt}}$
prompt variants moves behavior towards the Goldilocks Zone.
\end{itemize}

\noindent
\textbf{In short,}
NUP forbids the lower-left: one cannot simultaneously minimize boundary ambiguity and sensitivity.
Vision failures arise when training pushes the hard subset toward small $\Delta \hat m_u^\star$ (hence large $\Delta \hat p_u$),
whereas LLM hallucinations arise under high slack with large $\Delta \hat m_u^\star$ and uncontrolled $\Delta \hat p_u$.
Our probes and interventions are designed to steer behavior toward the intermediate band in Fig.~\ref{fig:nup_roadmap}.

\textbf{A note to the reader:} The following subsections present an operator/commutator formulation (borrowed from the Robertson--Schr\"odinger uncertainty relation) to justify the correlation channel behind our probe.
Readers could safely skip ahead to Secs.~\ref{sec:exp:protocol}--\ref{sec:experiments} for the experimental results and practical algorithms, and return to this theory section for context as needed.

\subsection{Loss-weighted state and canonical directional operators}
\label{subsec:state_and_ops}

Let $f_\theta$ be a differentiable model and $\mathcal L(\cdot)$ a scalar loss.
For a condition $c$, define the per-sample loss
\begin{equation}
\label{eq:loss_def_formalism}
\mathcal L_c(x) := \mathcal L\big(f_\theta(x), c\big).
\end{equation}
The real input-gradient field (loss sensitivity) is
\begin{equation}
\label{eq:real_grad_formalism}
p(x):=\nabla_x \mathcal L_c(x)\in\mathbb R^d .
\end{equation}
For small $\delta$, $\mathcal L_c(x+\delta)-\mathcal L_c(x)=p(x)^\top\delta+o(\|\delta\|_2)$, so $p(x)$ is the standard first-order
sensitivity object.

\vspace{2pt}
\noindent\textbf{Loss-induced weighting and state.}
To emphasize boundary-relevant regions in operator moments, we evaluate second-order statistics under a loss-induced
analysis weighting
\begin{equation}
\label{eq:weight_def_formalism}
w_c(x) := \frac{\mathcal L_c(x)^2}{\int_{\mathcal X}\mathcal L_c(\xi)^2\,d\xi},
\qquad
\mathbb E_c[g] := \int_{\mathcal X} g(x)\,w_c(x)\,dx.
\end{equation}
In practice, all expectations are instantiated either over an empirical dataset distribution or over the prompt set used in evaluation. 

It is convenient to encode $w_c$ as a normalized state in $\mathcal H:=L^2(\mathcal X)$:
\begin{equation}
\label{eq:state_def_formalism}
\begin{aligned}
    \psi_c(x) := &A_c(x)\exp\!\big(i\alpha\,\mathcal L_c(x)\big),
\\
A_c(x):= &\frac{|\mathcal L_c(x)|}{\sqrt{\beta_c}},
\\
\beta_c:= &\int_{\mathcal X}\mathcal L_c(x)^2\,dx,
\end{aligned}
\end{equation}
so that $|\psi_c(x)|^2=A_c(x)^2=w_c(x)$. We fix the phase scale $\alpha=1$ throughout. It only sets the phase convention of the auxiliary state and is not treated as a tunable parameter.
For a (densely defined) operator 
$\hat O$ on $L^2(\mathcal X)$, we write its expectation \cite{Sakurai2020} as \footnote{The expectation $\langle \hat O\rangle_c$ is the average value of this quantity in the state $\psi_c$, computed via the usual $L^2$ inner product.}
\begin{equation}
\label{eq:expectation_integral_formalism}
\langle \hat O\rangle_c \;:=\; \langle \psi_c,\hat O\psi_c\rangle
\;=\;\int_{\mathcal X}\psi_c(x)^{*}\,(\hat O\psi_c)(x)\,dx,
\end{equation}
where $(\cdot)^{*}$ denotes complex conjugation.

\vspace{2pt}
\noindent\textbf{Directional projection.}
To avoid ambiguity in vector-valued uncertainty statements, we introduce a unit direction $u\in\mathbb S^{d-1}$ and study
one-dimensional projections:
\begin{equation}
\label{eq:directional_projection_formalism}
\begin{aligned}
x_u:=&u^\top x,
\\
p_u(x):=&u^\top p(x)=\partial_u \mathcal L_c(x),
\qquad
\partial_u:=u^\top\nabla_x .
\end{aligned}
\end{equation}

\vspace{2pt}
\noindent\textbf{Canonical operators.} 
Define the multiplication and directional derivative operators as follows:
\begin{equation}
\label{eq:ops_formalism}
(\hat x_u g)(x):=(u^\top x)\,g(x),
\qquad
(\hat p_u g)(x):=-i\,\partial_u g(x),
\end{equation}
on a common dense domain $\mathcal D\subset L^2(\mathcal X)$ where the following algebraic manipulations are valid.
A direct calculation on $\mathcal D$ gives, for any test function $g$,
\begin{align}
[\hat x_u,\hat p_u]g
&=\hat x_u(-i\partial_u g)-(-i\partial_u)((u^\top x)g) \nonumber\\
&=-i(u^\top x)\partial_u g+i\big((u^\top x)\partial_u g+g\big)=ig,
\end{align}
hence
\begin{equation}
\label{eq:comm_formalism}
[\hat x_u,\hat p_u]=i\mathbb I.
\end{equation}
We use the canonical constant
\begin{equation}
\label{eq:kappa_formalism}
\kappa:=\tfrac12\big|\langle[\hat x_u,\hat p_u]\rangle_c\big|=\tfrac12,
\end{equation}
which follows immediately from \eqref{eq:comm_formalism} and $\langle\mathbb I\rangle_c=1$. Eq.~\eqref{eq:comm_formalism} is similar to the canonical position--momentum commutation relation from quantum mechanics \cite{Heisenberg1927,Dirac1958}.

\subsection{RS form of NUP and the mixed-axis corollary}
\label{subsec:rs_and_mixed_axis}

Operator inequalities constrain \emph{fluctuations} rather than means. Accordingly, for an observable $\hat O$ we write its centered version
as $\Delta\hat O:=\hat O-\langle\hat O\rangle_c$.
The dispersions of $\hat x_u$ and $\hat p_u$ under $\psi_c$ are then
\begin{equation}
\label{eq:disp_formalism}
\begin{aligned}
(\Delta\hat x_u)^2:=&\langle(\Delta\hat x_u)^2\rangle_c
=\mathbb E_c\!\big[(x_u-\mathbb E_c[x_u])^2\big],
\\
(\Delta\hat p_u)^2:=&\langle(\Delta\hat p_u)^2\rangle_c
=\big\|\Delta\hat p_u\,\psi_c\big\|_2^2,
\end{aligned}
\end{equation}
and the symmetrized covariance is
\begin{equation}
\label{eq:cov_formalism}
\mathrm{Cov}_c(\hat x_u,\hat p_u)
:=\frac12\Big\langle \Delta\hat x_u\,\Delta\hat p_u+\Delta\hat p_u\,\Delta\hat x_u\Big\rangle_c.
\end{equation}

\begin{theorem}[Neural Uncertainty Relation]
\label{thm:rs_form_of_nup}
For the canonical pair $(\hat x_u,\hat p_u)$ and the state $\psi_c$ with finite second moments,
\begin{equation}
\label{eq:rs_formalism}
(\Delta\hat x_u)^2(\Delta\hat p_u)^2\ \ge\ \kappa^2+\mathrm{Cov}_c(\hat x_u,\hat p_u)^2,
\end{equation}
where $\kappa=\tfrac12\big|\langle[\hat x_u,\hat p_u]\rangle_c\big|=\tfrac12$.
\end{theorem}
\noindent
The inequality is the classical Robertson--Schr\"odinger relation in quantum physics \cite{Robertson1929,Schrodinger1930}; a standard proof is provided in the Supplementary Material.

\vspace{2pt}
\noindent\textbf{Correlation channel.}
Define the (operator) correlation coefficient
\begin{equation}
\label{eq:rho_formalism}
\rho_c(u):=\frac{\mathrm{Cov}_c(\hat x_u,\hat p_u)}{\Delta\hat x_u\,\Delta\hat p_u}\in[-1,1].
\end{equation}
Then Theorem~\ref{thm:rs_form_of_nup} is equivalently
\begin{equation}
\label{eq:rs_rho_formalism}
\big(\Delta\hat x_u\,\Delta\hat p_u\big)^2\bigl(1-\rho_c(u)^2\bigr)\ \ge\ \kappa^2 .
\end{equation}
The factor $(1-\rho^2)$ is the explicit coupling term that will later be approximated by a cheap proxy.

\vspace{4pt}
\noindent\textbf{Mixed-axis elimination of covariance.}
A useful way to ``spend'' the covariance term is to pass to a mixed observable.
For $\lambda\in\mathbb R$, define
\begin{equation}
\label{eq:mixed_obs_formalism}
\hat m_u(\lambda):=\hat x_u+\lambda\,\hat p_u.
\end{equation}
Expanding the variance yields
\begin{equation}
\label{eq:var_m_formalism}
(\Delta\hat m_u(\lambda))^2
=(\Delta\hat x_u)^2+\lambda^2(\Delta\hat p_u)^2+2\lambda\,\mathrm{Cov}_c(\hat x_u,\hat p_u).
\end{equation}
The minimizing coefficient is
\begin{equation}
\label{eq:lambda_star_formalism}
\lambda^\star=-\frac{\mathrm{Cov}_c(\hat x_u,\hat p_u)}{(\Delta\hat p_u)^2},
\end{equation}
and the minimum mixed variance satisfies the identity
\begin{equation}
\label{eq:min_var_m_formalism}
(\Delta\hat m_u^\star)^2
:=\min_{\lambda}(\Delta\hat m_u(\lambda))^2
=(\Delta\hat x_u)^2\bigl(1-\rho_c(u)^2\bigr).
\end{equation}

\begin{lemma}[Mixed-axis form of NUP]
\label{lem:mixed_axis_nup}
With $\Delta\hat m_u^\star$ defined by \eqref{eq:min_var_m_formalism}, Theorem~\ref{thm:rs_form_of_nup} implies
\begin{equation}
\label{eq:nup_direct_formalism}
\boxed{\ \Delta\hat m_u^\star\,\Delta\hat p_u\ \ge\ \kappa=\tfrac12\ }.
\end{equation}
\end{lemma}
\noindent
\textit{Proof.}
Substitute \eqref{eq:min_var_m_formalism} into \eqref{eq:rs_rho_formalism} to obtain
$(\Delta\hat m_u^\star)^2(\Delta\hat p_u)^2\ge \kappa^2$, and take square roots. \hfill$\square$

When $\rho_c(u)=0$, \eqref{eq:min_var_m_formalism} gives $\Delta\hat m_u^\star=\Delta\hat x_u$ and Lemma~\ref{lem:mixed_axis_nup}
reduces to the canonical Heisenberg form
\begin{equation}
\label{eq:heisenberg_degen_formalism}
\Delta\hat x_u\,\Delta\hat p_u\ \ge\ \tfrac12.
\end{equation}
In this regime the mixed axis offers no thinner direction than $\hat x_u$ itself; this is the coupling-free case aligned with
the Heisenberg-type tension studied in \cite{zhang_iscience2025, zhang_nsr2024, zhang_scirep2024, zhang2025analysis}.

\subsection{Interpreting Lemma~\ref{lem:mixed_axis_nup}: ellipse geometry and conjugate trade-off}
\label{subsec:geometry_precision_robustness}

It is convenient to switch between (i) the operator notation $(\hat x_u,\hat p_u)$ defining second moments under the state $\psi_c$,
and (ii) the induced scalar random variables obtained by evaluating these observables on samples.
Concretely, $\hat x_u$ is a multiplication operator, so its ``measurement'' on a sample $x$ is the scalar $x_u=u^\top x$; accordingly,
$\langle \hat x_u\rangle_c=\mathbb E_c[x_u]$ and $(\Delta\hat x_u)^2=\mathbb E_c[(x_u-\mathbb E_c[x_u])^2]$.
For $\hat p_u=-i\partial_u$, the corresponding \emph{real} sensitivity channel is the directional loss gradient
$p_u(x)=\partial_u \mathcal L_c(x)$.
Under the loss-induced state, the symmetrized \emph{operator} covariance reduces exactly to a \emph{scalar} covariance
(Theorem~\ref{thm:cov_reduction}).

\textbf{Boundary layer induced by loss weighting.}
All expectations $\mathbb E_c[\cdot]$ are taken under the analysis weighting $w_c(x)\propto \mathcal L_c(x)^2$
(Section~\ref{subsec:state_and_ops}).
In classification, large loss concentrates on misclassified or small-margin samples, which are typically boundary-adjacent.
Therefore, the loss-weighted population emphasized by $w_c$ can be viewed as an \emph{effective boundary layer}.

\textbf{Ellipse view of the loss-weighted $(x_u,p_u)$ cloud.}
With the above identification, the loss-weighted samples form a 2D cloud of points $(x_u,p_u)$ in the $(x_u,p_u)$ plane.
Its second-order geometry is summarized by the covariance matrix
\begin{equation}
\label{eq:sigma_formalism_repl}
\Sigma(u):=
\begin{pmatrix}
(\Delta\hat x_u)^2 & \mathrm{Cov}_c(\hat x_u,\hat p_u)\\
\mathrm{Cov}_c(\hat x_u,\hat p_u) & (\Delta\hat p_u)^2
\end{pmatrix}.
\end{equation}
Geometrically, $\Sigma(u)$ defines an ellipse: the correlation $\rho_c(u)$ controls its tilt (coupling),
while $\det\Sigma(u)$ controls its area. A short calculation gives
\begin{equation}
\label{eq:det_sigma_formalism_repl}
\begin{aligned}
\det\Sigma(u)
=&(\Delta\hat x_u)^2(\Delta\hat p_u)^2-\mathrm{Cov}_c(\hat x_u,\hat p_u)^2\\
=&\big(\Delta\hat x_u\,\Delta\hat p_u\big)^2\bigl(1-\rho_c(u)^2\bigr).
\end{aligned}
\end{equation}
Combining \eqref{eq:det_sigma_formalism_repl} with \eqref{eq:min_var_m_formalism} yields the factorization
\begin{equation}
\label{eq:area_factorization_formalism_repl}
\sqrt{\det\Sigma(u)}=\Delta\hat m_u^\star\,\Delta\hat p_u,
\end{equation}
so Theorem~\ref{thm:rs_form_of_nup} and Lemma~\ref{lem:mixed_axis_nup} can be read as an \emph{area lower bound}:
\begin{equation}
\label{eq:area_bound_formalism_repl}
\sqrt{\det\Sigma(u)}\ \ge\ \kappa=\tfrac12
\ \Longleftrightarrow \ 
\Delta\hat m_u^\star\,\Delta\hat p_u\ \ge\ \tfrac12.
\end{equation}

\textbf{Stability of the loss landscape.}
The mixed observable $\hat m_u(\lambda)=\hat x_u+\lambda\hat p_u$ corresponds to projecting the $(x_u,p_u)$ cloud onto the line
$m_u(\lambda)=x_u+\lambda p_u$.
The quantity $\Delta\hat m_u^\star:=\min_{\lambda}\Delta\hat m_u(\lambda)$ is the smallest standard deviation achievable among this one-parameter
family of projections, i.e., the thickness of the loss-weighted cloud along its sharpest mixed axis.

Thus, the decrease of $\Delta\hat m_u^\star$ means that within the loss-weighted (boundary-layer) population,
the \emph{position--sensitivity relation becomes more consistent}: similar positions $x_u$ tend to induce similar sensitivities $p_u$,
and conversely, a given sensitivity level $p_u$ corresponds to a narrower range of positions $x_u$.
Equivalently, the local loss landscape seen by boundary-adjacent samples is more ``aligned" across samples, producing a thinner ambiguity band
in the \emph{$(x_u,p_u)$ phase plane}. We emphasize that this is an \emph{effective} thickness (a second-order spread of a loss-weighted phase-plane
population), not the Euclidean thickness of the decision boundary in input space.

\textbf{Connection to clean accuracy.}
Since natural errors concentrate near the decision boundary, and $w_c$ upweights precisely these boundary-adjacent/high-loss samples,
reducing this boundary-layer ambiguity band (smaller $\Delta\hat m_u^\star$) typically reduces near-boundary natural misclassifications.
Consequently, $\Delta\hat m_u^\star$ often co-varies inversely with clean accuracy along training trajectories.

\textbf{Conjugate axis as sensitivity dispersion and robustness.}
In contrast, $\Delta\hat p_u$ measures the dispersion of the derivative operator $\hat p_u$ under $\psi_c$; under the loss-induced state,
$\hat p_u\psi_c$ contains a real drift term proportional to the real gradient $p_u(x)$ (Section~\ref{subsec:bridge_and_proxy}),
so $\Delta\hat p_u$ captures a boundary-weighted \emph{sensitivity scale}.
This connects to robustness via the first-order expansion
\begin{equation}
\label{eq:first_order_loss_repl}
\mathcal L_c(x+\delta)-\mathcal L_c(x)=p(x)^\top\delta+o(\|\delta\|_2),
\end{equation}
which implies that, under an $L_2$-bounded perturbation $\|\delta\|_2\le\epsilon$, the worst-case first-order loss increase scales as
$\epsilon\|p(x)\|_2$.
Therefore, larger boundary-weighted sensitivity dispersion (i.e., larger $\Delta\hat p_u$) indicates that smaller perturbations tend to induce larger loss changes
and more frequent boundary crossings, degrading adversarial robustness under gradient-based threat models.

\textbf{Accuracy--robustness tension as a conjugate constraint.}
Eq.~\eqref{eq:area_bound_formalism_repl} is the central geometric message: when the ellipse area is close to its RS lower bound,
reducing the thin-axis thickness $\Delta\hat m_u^\star$ (sharper boundary-layer consistency / thinner ambiguity band) necessarily enlarges the conjugate
sensitivity dispersion $\Delta\hat p_u$.
This is the sense in which the two quantities are \emph{conjugate}: compressing the boundary-weighted $(x_u,p_u)$ cloud along its sharpest mixed axis
must be compensated by expansion along the sensitivity axis so that the area constraint remains satisfied.
This provides a principled geometric vocabulary for the empirical accuracy--robustness tension.

It is sometimes convenient to define
\begin{equation}
\label{eq:slack_formalism_repl}
V_c(u):=\sqrt{\det\Sigma(u)}=\Delta\hat m_u^\star\,\Delta\hat p_u,
\ \ 
S_c(u):=\frac{V_c(u)}{\kappa}\ge 1.
\end{equation}
Here $S_c(u)\approx 1$ indicates a near-bound regime (limited slack), where the conjugate trade-off is most pronounced,
while $S_c(u)\gg 1$ indicates substantial slack in the second-order geometry.

\subsection{From operator moments to real gradients and a practical proxy}
\label{subsec:bridge_and_proxy}

The NUP above is stated at the operator level.
To connect it to computable neural quantities, we first show that the operator covariance
$\mathrm{Cov}_c(\hat x_u,\hat p_u)$ reduces \emph{exactly} to a scalar covariance involving the real gradient $p_u(x)$ under the loss-phase
state \eqref{eq:state_def_formalism}. We then introduce a cheap proxy for the coupling channel.

\paragraph{Operator covariance reduction}
\label{subsubsec:cov_reduction}
We first present the following theorem.

\begin{theorem}[Covariance reduction under the loss-induced state]
\label{thm:cov_reduction}
Let $\psi_c(x)=A_c(x)e^{i\alpha\mathcal L_c(x)}$ with $A_c^2=w_c$ as defined in Eq. \eqref{eq:state_def_formalism}.
Assume that boundary terms vanish for integration by parts along direction $u$, and
then the symmetrized operator covariance satisfies
\begin{equation}
\label{eq:cov_reduction_formalism}
\mathrm{Cov}_c(\hat x_u,\hat p_u)=\alpha\,\mathrm{Cov}^{\mathrm{sc}}_c(x_u,p_u),
\end{equation}
where $\mathrm{Cov}^{\mathrm{sc}}_c(g,h):=\mathbb E_c[gh]-\mathbb E_c[g]\mathbb E_c[h]$ is the scalar covariance under $w_c$ and
$p_u(x)=\partial_u\mathcal L_c(x)$.
\end{theorem}

\noindent\textit{Proof sketch.}
We rewrite the argument in terms of the centered variable $\Delta x_u$ and apply the drift identity for $-i\partial_u\psi_c$, which makes the $\partial_u A_c$ contribution purely imaginary and hence irrelevant after taking the real part in the symmetrized covariance. Please see more proof details in the supplementary material.

\paragraph{CC-Probe as a proxy for the coupling channel}
\label{subsubsec:cc_probe}

The operator correlation $\rho_c(u)$ is defined through dataset-level moments under $w_c$ and depends on the choice of direction $u$.
To obtain a cheap per-sample statistic that tracks the same coupling channel appearing in
$\mathrm{Cov}^{\mathrm{sc}}_c(x_u,p_u)$, we use a geometric identity for random directional probes.

\begin{lemma}[Directional-probe correlation and cosine]
\label{lem:cos_dir_corr}
Fix nonzero vectors $x,p\in\mathbb R^d$, and draw $u\sim\mathrm{Unif}(\mathbb S^{d-1})$.
Let $x_u:=u^\top x$ and $p_u:=u^\top p$. Then
\begin{equation}
\label{eq:cos_dir_corr_formalism}
\mathrm{Corr}_u(x_u,p_u)=\frac{x^\top p}{\|x\|_2\,\|p\|_2}=\cos(x,p).
\end{equation}
\end{lemma}
\noindent\textit{Proof.}
Since $\mathbb E_u[u]=0$ and $\mathbb E_u[uu^\top]=\frac{1}{d}I$,
\begin{equation}
\begin{aligned}
\mathbb E_u[x_u p_u]=&\mathbb E_u[(u^\top x)(u^\top p)]=x^\top \mathbb E_u[uu^\top]p=\tfrac1d x^\top p,
\\
\mathbb E_u[x_u^2]=&\tfrac1d\|x\|_2^2,
\\
\mathbb E_u[p_u^2]=&\tfrac1d\|p\|_2^2.
\end{aligned}
\end{equation}
Because $\mathbb E_u[x_u]=\mathbb E_u[p_u]=0$, the Pearson correlation equals
$\mathrm{Corr}_u(x_u,p_u)=\mathbb E_u[x_u p_u]/\sqrt{\mathbb E_u[x_u^2]\mathbb E_u[p_u^2]}=\cos(x,p)$. \hfill$\square$

Motivated by Lemma~\ref{lem:cos_dir_corr} and Theorem~\ref{thm:cov_reduction}, we provide the
\textbf{Conjugate Correlation Probe (CC-Probe, briefly)} for a sample $x$ as
\begin{equation}
\label{eq:ccprobe_formalism}
c_{\text{probe}}(x):=\big|\cos\!\big(x,\,p(x)\big)\big|
=\frac{\big|x^\top p(x)\big|}{\|x\|_2\,\|p(x)\|_2}.
\end{equation}
The quantity $c_{\text{probe}}$ is a cheap per-sample indicator of directional coupling: averaged over random directions, $\cos(x,p)$ is exactly the correlation
between the projected coordinate and projected gradient (Lemma~\ref{lem:cos_dir_corr}), and the dataset-level covariance channel driving
$\rho_c(u)$ reduces to scalar covariances between $x_u$ and $p_u$ (Theorem~\ref{thm:cov_reduction}).
This relationship is sufficient for $c_{\text{probe}}$ to serve as a practical proxy used in our experiments and interventions.

\subsection{Testable predictions}
\label{subsec:nur_regimes}

The NUP motivates a simple idea: \emph{abnormal input--gradient coupling} is a reliable marker of failure modes. We test this idea with the CC-Probe introduced above.

\begin{proposition}[Exp.~1 prediction: a ``high-cosine tail'' marks hard/fragile vision samples]
\label{prop:exp1_sep}
As training improves clean accuracy, correctly classified images should show \emph{lower} CC-Probe $c_{\text{img}}$, while misclassified/hard images remain at \emph{higher} $c_{\text{img}}$ (a persistent high-cosine tail). Adversarial errors should concentrate on this high-cosine tail.
\end{proposition}

\begin{proposition}[Exp.~2 prediction: $\pm$FGSM changes CC-Probe in the expected direction]
\label{prop:exp2_causal}
For a fixed vision model, a small gradient-aligned perturbation (+FGSM) should \emph{increase} $c_{\text{img}}$ and reduce accuracy, while a sufficiently small anti-aligned perturbation (--FGSM) should \emph{decrease} $c_{\text{img}}$ and can preserve (or slightly improve) accuracy.
\end{proposition}

\begin{proposition}[Exp.~3--4 prediction: training that suppresses strong $x\!\cdot\!p$ coupling becomes more robust]
\label{prop:exp34_mask}
If vulnerability is driven by a few input components with large coupling scores
(implemented in practice as the channel-wise normalized interaction
$|\tilde x_{c,j}\tilde p_{c,j}|$),
then masking these dominant components during training (ConjMask) should improve robustness to standard gradient attacks (e.g., PGD/APGD-CE) without adversarial training (Exp.~3). If this robustness is loss-dependent, adding a logit stabilizer (LogitReg) should restore robustness under stronger loss-optimized attacks (e.g., APGD-DLR) (Exp.~4).
\end{proposition}

\begin{proposition}[Exp.~5--6 prediction: in LLM prefill, low CC-Probe means under-conditioning and higher hallucination risk]
\label{prop:exp56_llm}
Using only prefill (no decoding), prompts with unusually \emph{low} $c_{\text{prompt}}$ should have higher hallucination risk; therefore $-c_{\text{prompt}}$ should predict hallucination above chance (Exp.~5). Among paraphrased prompts, choosing the one with \emph{higher} $c_{\text{prompt}}$ should more often select the judge-preferred prompt (Exp.~6).
\end{proposition}

\section{Experimental Protocol}
\label{sec:exp:protocol}

The preceding section derived four testable predictions from the NUP.
We design Experiments~1--6 to test these predictions across two modalities: discriminative vision (classification) and generative language (reasoning). Detailed training recipes and hyperparameters are deferred to Supplementary Material.

\subsection{Datasets and Model Zoo}
\label{sec:exp:common}

\paragraph{Vision}
We cover three levels of visual complexity:
(i) \textbf{CIFAR-10} ($32{\times}32$, coarse-grained);
(ii) \textbf{Tiny-ImageNet-200} ($64{\times}64$, medium-scale);
and (iii) \textbf{ImageNet-100} ($224{\times}224$, a class-balanced subset of ILSVRC2012).
Across these benchmarks, we instantiate a diverse ``architecture zoo'' to ensure architecture-agnostic validity:
\begin{itemize}[leftmargin=*]
    \item \textbf{CNNs:} ResNet-18/50, DenseNet-121, EfficientNet-B0/B4.
    \item \textbf{Transformers:} Vision Transformer ViT-Tiny and Swin-Tiny.
    \item \textbf{State-Space Models:} Vision Mamba (Vim-Tiny) \cite{zhu2024visionmamba}.
\end{itemize}

\paragraph{Language}
We employ a publicly available model, \textit{deepseek-coder-7b-instruct-v1.5}, as the subject model. The evaluation focuses on undergraduate-level math reasoning tasks where hallucination is prevalent. We curated two specific datasets:
\begin{itemize}[leftmargin=*]
    \item \textbf{Benchmark-500:} A collection of 500 unique, undergraduate-level mathematics problems.
    \item \textbf{Perturbation-100:} A focused subset of 100 problems sampled from Benchmark-500. For each problem, we generated 4 additional semantically equivalent but syntactically distinct rephrasings (5 variations total per problem).
\end{itemize}

These datasets are designed for controlled falsification of Proposition~\ref{prop:exp56_llm}: math reasoning reduces confounds from retrieval and enables more stable verification.

Ground truth for these generative tasks is established via a ``Panel-of-Judges'' consensus mechanism involving five external LLMs (\textit{Claude-4.5-Opus-Think}, \textit{DeepSeek-v3.2-Think}, \textit{Gemini-3-Pro}, \textit{GPT-5.2-Thinking}, and \textit{Grok-4.1-Think}).

\subsection{Operationalizing the Conjugate Probe}
\label{sec:exp:probe}

A core contribution of this work is translating the theoretical operator correlation \(\rho_c\) into a computable per-sample observable.

\paragraph{Vision}
For an image input \(X \in \mathbb{R}^{C \times H \times W}\), we compute the gradient of the task loss \(\mathcal{L}\) with respect to the input, \(P = \nabla_X \mathcal{L}\). Both \(X\) and \(P\) are flattened into 1-D vectors \(x\) and \(p\). Because the practical vision probe is computed on the normalized input tensor fed to the model, we use the uncentered image-space cosine; the sample-level conjugate correlation and its dataset-level average over \(N\) samples are defined as:
\begin{equation}
\label{eq:c_img}
c_{\text{img}} = \frac{|x^\top p|}{\|x\|_2 \|p\|_2}, \qquad \bar{c}_{\text{img}} = \frac{1}{N} \sum_{i=1}^{N} c_{\text{img}}^{(i)},
\end{equation}
where $c_{\text{img}}^{(i)}$ is calculated on the $i$-th sample.
In practice, unless explicitly stated, we report simple (uniform) averages over a given subset (e.g., correct vs.\ incorrect samples) rather than explicitly reweighting by \(\mathcal L^2\). The loss-induced weighting is the theoretical device used to derive the operator statistics; empirically, the per-sample cosine serves as a lightweight proxy for input--gradient coupling.

\textit{Remark.} The NUP is expressed via an operator correlation $\rho_c(u)$ under a loss-induced state, while we use the practical proxy $c_{\text{img}}$.
$c_{\text{img}}$ is intended to track the \emph{input--gradient coupling channel} that enters the Robertson--Schr\"odinger covariance term:
under the loss-phase construction, $\mathrm{Cov}_c(\hat x_u,\hat p_u)$ reduces exactly to a (loss-weighted) scalar covariance between $x_u$ and $p_u$ (Theorem~\ref{thm:cov_reduction}),
and the dot-product $x^\top p$ is the direction-averaged form of this coupling.

\paragraph{Language}
For language models, we extract the signal during the prompt processing stage (prefill).
We compute a mean-centered cosine between the prompt embeddings and their input gradients.
Let $X \in \mathbb{R}^{T \times d}$ be the embedding tensor for the $T$ valid tokens, and
$P \in \mathbb{R}^{T \times d}$ be the corresponding gradient tensor. We use a shifted loss, i.e., the next-token negative log-likelihood. Unless otherwise stated, the shifted NLL loss is computed on the non-system prefill tokens of the chat template (i.e., the user prompt together with template tokens available before answer generation), so the score is available at prefill time without generating any answer tokens; it still requires one backward pass with respect to the prompt embeddings. 
We flatten these tensors into 1-D vectors $x,p \in \mathbb{R}^{Td}$ and define the mean-centered vectors
\[
\bar{x}:=x-\mu_x\mathbf{1},\qquad \bar{p}:=p-\mu_p\mathbf{1},
\]
where $\mu_x=\mathrm{mean}(x)$ and $\mu_p=\mathrm{mean}(p)$ are the corresponding scalar means.
The prompt-level conjugate score and its corpus-level average over $M$ prompts are computed as:
\begin{equation}
\label{eq:c_prompt}
c_{\text{prompt}} = \frac{| \bar{x}^\top \bar{p} |}{\| \bar{x} \|_2 \| \bar{p} \|_2}, \qquad
\bar{c}_{\text{prompt}} = \frac{1}{M} \sum_{j=1}^{M} c_{\text{prompt}}^{(j)}.
\end{equation}
where $c_{\text{prompt}}^{(j)}$ is calculated on the $j$-th prompt.

\subsection{Evaluation Metrics}

\paragraph{Adversarial Robustness (Vision)}
To assess robustness under complementary white-box threat models, we evaluate three widely used $L_\infty$ attacks at the target budget:
(i) \textbf{Iterative baseline attack:} \textbf{PGD-20} as the standard multi-step first-order attack.
(ii) \textbf{AutoAttack component (CE loss):} \textbf{APGD-CE} (Auto-PGD optimizing cross-entropy) to provide a stronger, well-tuned first-order baseline.
(iii) \textbf{AutoAttack component (DLR loss):} \textbf{APGD-DLR} (Auto-PGD optimizing the Difference of Logits Ratio loss) to stress-test robustness beyond CE-optimized gradients.

\paragraph{Hallucination Analysis (Language)}
We treat Exp.~5 hallucination detection as a binary classification task ($1=$ hallucination, $0=$ non-hallucination), quantified by the \textbf{Area Under the ROC Curve (AUROC)} with bootstrap resampling ($N_{\text{boot}}=2000$). For Exp.~6 prefill-stage prompt selection, the \textbf{primary metrics} are \textbf{Top-1 Hit Rate} (frequency with which the selected prompt belongs to the judge-best set) and \textbf{Mean Regret} (the average quality gap between the selected prompt and the best judge score within the same five-prompt cluster). Deterministic decoding is used for reproducibility.

% =========================
% Experiments and Evidence
% =========================
\section{Experiments and Evidence}
\label{sec:experiments}

\subsection{Experimental Roadmap}
\label{sec:roadmap}

Our experiments test the NUP predictions using a simple, computable proxy---the CC-Probe---across vision classification and LLM prompting. We organize Exp.~1--6 into three phases\footnote{To reproduce Exp.~1--6, please refer to the released artifacts at \url{https://doi.org/10.57760/sciencedb.29613}.}:

\begin{itemize}[leftmargin=*,itemsep=2pt,topsep=2pt]
    \item \textbf{Phase I: Diagnosis in vision (Secs.~\ref{sec:exp1}--\ref{sec:exp2}).}
    We test whether input--gradient coupling, measured by $c_{\text{img}}$, behaves as a per-sample indicator of boundary difficulty.
    Exp.~1 checks that training separates samples into a low-$c_{\text{img}}$ bulk (mostly correct) and a high-$c_{\text{img}}$ tail (hard/incorrect), as predicted by Proposition~\ref{prop:exp1_sep}.
    Exp.~2 then applies $\pm$FGSM as a controlled perturbation and verifies the directional prediction that $c_{\text{img}}$ increases under +FGSM and decreases under --FGSM (Proposition~\ref{prop:exp2_causal}). These experiments locate the ``wrong/semi-hard samples" region of Fig.~\ref{fig:nup_roadmap} (small boundary thickness, large sensitivity).

    \item \textbf{Phase II: Vision intervention for robustness (Secs.~\ref{sec:exp3_results}--\ref{sec:exp4_results}).}
    We test whether reducing dominant input--gradient couplings during training improves robustness.
    Exp.~3 introduces ConjMask, which masks large $|x_i p_i|$ components on triggered samples, and evaluates robustness under $L_\infty$ attacks: PGD-20, APGD-CE, and APGD-DLR.
    Exp.~4 adds LogitReg to broaden robustness beyond the CE channel, targeting cases where performance is strong under CE-based attacks but weaker under DLR-based attacks (Proposition~\ref{prop:exp34_mask}). ConjMask aims to push these samples towards the central ``Goldilocks" band in Fig.~\ref{fig:nup_roadmap}.

    \item \textbf{Phase III: Prefill-only LLM evaluation (Secs.~\ref{sec:exp5}--\ref{sec:exp6}).}
    We test the complementary claim that unusually low prompt-level coupling $c_{\text{prompt}}$ indicates weak conditioning and higher hallucination risk.
    Exp.~5 evaluates $-c_{\text{prompt}}$ as a decoding-free hallucination risk score (AUROC).
    Exp.~6 applies the same prefill score for prompt selection among paraphrases, measuring whether choosing higher-$c_{\text{prompt}}$ prompts improves judge-preferred selection (Proposition~\ref{prop:exp56_llm}). Hallucination-prone prompts reside in the high ambiguity region of Fig.~\ref{fig:nup_roadmap} (large boundary thickness, uncontrolled sensitivity).
\end{itemize}

\subsection{Exp.~1: CC-Probe Separates Correct and Incorrect Samples During Training}
\label{sec:exp1}

\begin{figure}[t!]
    \centering
    \includegraphics[width=0.48\textwidth]{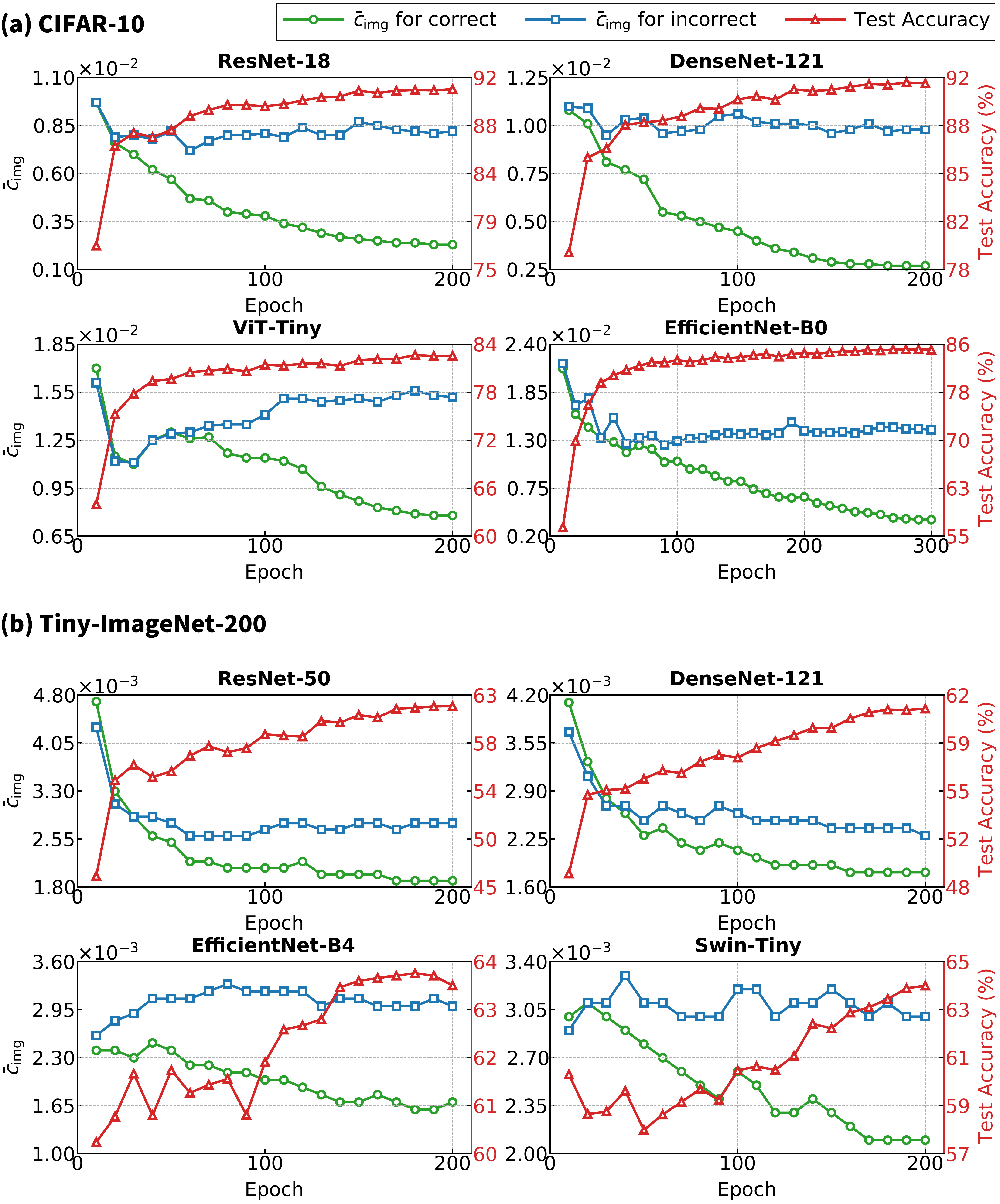}
    \caption{\textbf{Evolution of the CC-Probe and accuracy during training.}
We plot evaluation accuracy (red, right axis) and the mean absolute input--gradient cosine $\bar{c}_{\text{img}}$ (left axis; Eq.~\eqref{eq:c_img}), computed on the held-out evaluation split every 10 epochs and reported separately for correctly classified versus misclassified samples (green vs.\ blue).
Gradients are taken w.r.t.\ the standard cross-entropy loss using the ground-truth label.
\textbf{(a)} CIFAR-10 (ResNet-18, DenseNet-121, ViT-Tiny, EfficientNet-B0).
\textbf{(b)} Tiny-ImageNet-200 (ResNet-50, DenseNet-121, EfficientNet-B4, Swin-Tiny).}
    \label{fig:cosine_sim_evolution}
\end{figure}

For each test input $x$ we compute the per-sample CC-Probe
$c_{\text{img}}(x)$ from Eq.~\eqref{eq:c_img}, where
$p(x)=\nabla_x \mathcal{L}(f_\theta(x),y_{\text{true}})$ is the input gradient of the standard cross-entropy.
We then average $c_{\text{img}}(x)$ over the subset of correctly classified samples and over the subset of misclassified samples at each epoch.
Since $c_{\text{img}}$ is scale-free, the trends below are not driven by gradient magnitudes; in implementation we stabilize the cosine computation by adding a small value to the input/gradient norm denominators, rather than dropping low-gradient samples.

We tracked $\bar{c}_\text{img}$ during training across CNN and Transformer architectures on
CIFAR-10 and Tiny-ImageNet-200, recording the statistic separately for correctly and incorrectly
classified samples on the evaluation split.
As shown in Fig.~\ref{fig:cosine_sim_evolution}, two consistent phenomena emerged:

\begin{enumerate}[label=(\roman*)]
\item \textbf{Correct samples decouple.}
As training progresses and accuracy improves, $\bar{c}_\text{img}$ for correctly classified
samples steadily decreases toward zero, indicating that these images tend to decouple with their gradients.

\item \textbf{A persistent ``boundary-stress''.}
In contrast, $\bar{c}_\text{img}$ for incorrectly classified samples remains elevated,
creating a stable gap between the two populations across checkpoints.
\end{enumerate}

A small $\bar{c}_\text{img}$ implies weak average conjugate alignment between $x$ and
$p(x)=\nabla_x\mathcal L_c(x)$, suggesting that the NUP correlation channel is relatively unsaturated
for the corresponding sample set. From the NUP viewpoint, training tends to move correctly classified
samples into a high-slack regime (low loss with low residual correlation), whereas misclassified
samples remain in a boundary-stress where correlation stays elevated, indicating persistent boundary
sensitivity. This pattern supports the redistribution picture in Proposition~\ref{prop:exp1_sep}:
conjugate alignment concentrates on a small hard subset while most correct samples exhibit low coupling.

\subsection{Exp.~2: Stepping Along/Against the Gradient Moves the CC-Probe as Predicted}
\label{sec:exp2}

\FloatBarrier
\begin{figure}[!t]
    \centering
    \includegraphics[width=0.48\textwidth]{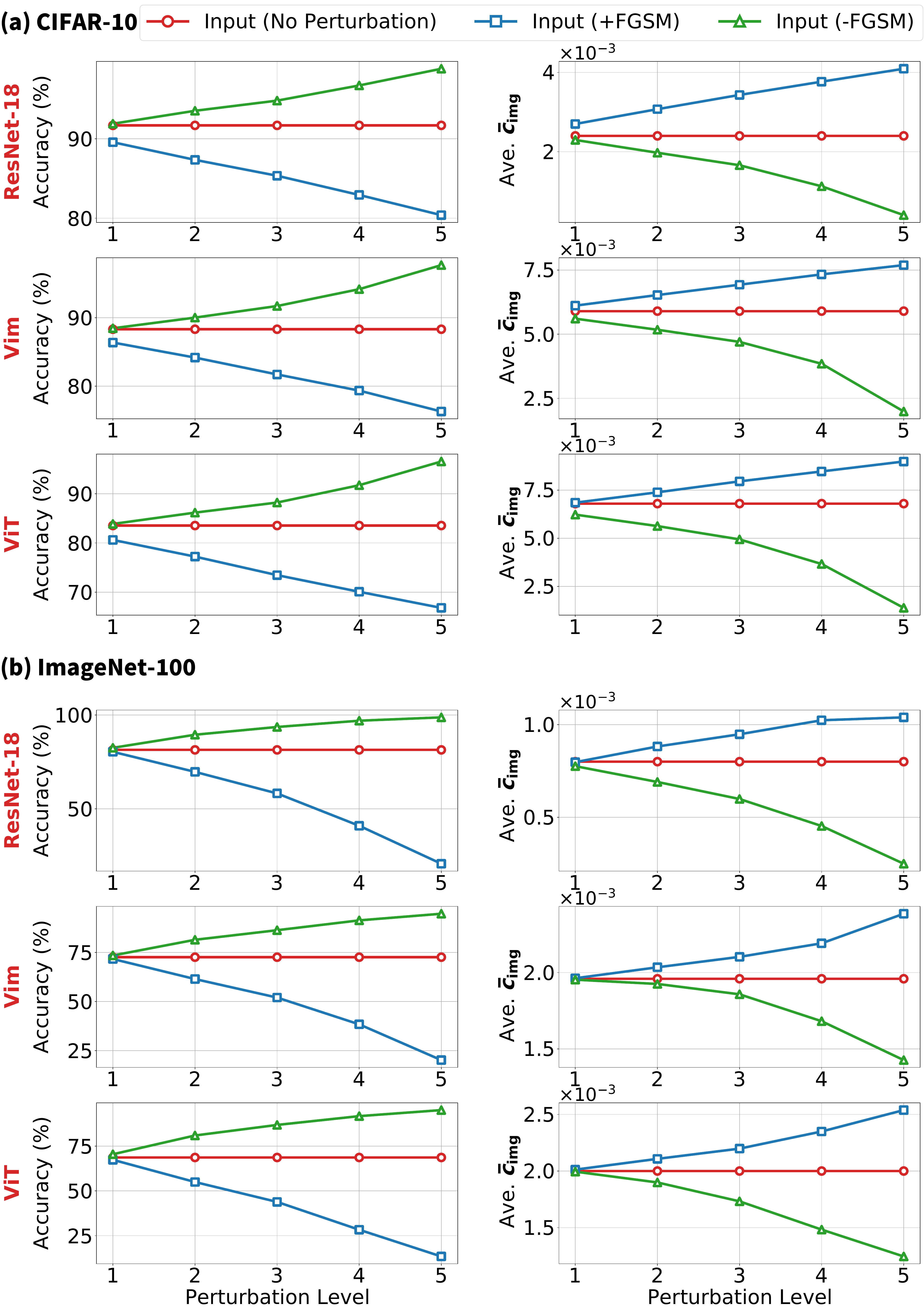}
    \caption{\textbf{Effect of gradient-aligned perturbations on accuracy and \(\bar{c}_\text{img}\).}
    We evaluate ResNet-18, ViT-Tiny, and Vim-Tiny on CIFAR-10 and ImageNet-100.
    Left column: accuracy; right column: \(\bar{c}_\text{img}\).
    We compare \textbf{Clean}, \textbf{+FGSM} and \textbf{--FGSM} across perturbation levels \(\epsilon\)
    (see Supplementary Material for exact values).}
    \label{fig:Cifar10_cosine_sim}
\end{figure}

To test whether $\bar c_{\text{img}}$ responds \emph{causally} to boundary-directed input changes, we evaluate the dedicated Exp.~2 checkpoints (trained with clean data) for ResNet-18, ViT-Tiny, and Vim-Tiny on CIFAR-10 and ImageNet-100. We then perturb each input by a small $L_\infty$ step along the sign of the input gradient:
\begin{equation}
\begin{aligned}
x^{+}=& \mathrm{clip}\!\left(x + \epsilon\,\mathrm{sign}(\nabla_x \mathcal{L}_c(x))\right),\\
x^{-}=& \mathrm{clip}\!\left(x - \epsilon\,\mathrm{sign}(\nabla_x \mathcal{L}_c(x))\right),
\end{aligned}
\end{equation}
where $\mathrm{clip}(\cdot)$ enforces valid pixel ranges.
We then recompute $\bar c_{\text{img}}$ on the perturbed inputs (i.e., using $(x^{\pm},\,\nabla_{x^{\pm}}\mathcal L_c)$)
and evaluate accuracy under the same perturbations.

As illustrated in Fig.~\ref{fig:Cifar10_cosine_sim}, we observe a consistent bidirectional response:

\begin{enumerate}[label=(\roman*)]
\item \textbf{Gradient-aligned perturbations increase coupling and degrade accuracy.}
As $\epsilon$ increases, +FGSM typically raises $\bar c_{\text{img}}$ while reducing accuracy,
consistent with pushing samples toward locally more loss-sensitive regions.

\item \textbf{Anti-aligned perturbations reduce coupling (when $\epsilon$ is small).}
For sufficiently small $\epsilon$, --FGSM often lowers $\bar c_{\text{img}}$ and can preserve
(or slightly improve) accuracy, reflecting a move toward locally flatter loss regions.
For larger $\epsilon$, the perturbation may leave the semantic manifold and accuracy can still drop;
our claim here is about the \emph{directional} effect on the coupling statistic.
\end{enumerate}

In the NUP, proximity to the constraint is governed by the effective conjugate volume
$V_c(u)=\Delta \hat m_u^\star\,\Delta \hat p_u$ and the slack ratio
$S_c(u)=V_c(u)/\kappa\ge 1$ (Eq.~\eqref{eq:slack_formalism_repl}).
Here, the CC-Probe $\bar c_{\text{img}}$ is used as a \emph{proxy} for the input--gradient coupling channel entering $\rho_c(u)$ via the loss-induced statistics; it is not itself an exact estimator of $V_c(u)$ or $S_c(u)$.
Accordingly, Fig.~\ref{fig:Cifar10_cosine_sim} is interpreted as an empirical probe of boundary stress:
+FGSM increases loss and empirically raises $\bar c_{\text{img}}$, which is consistent with stronger correlation pressure and higher boundary stress,
whereas sufficiently small --FGSM steps reduce the measured coupling and are consistent with a relaxation toward a less stressed regime.

\subsection{Exp.~3: ConjMask improves robustness under CE-based first-order attacks}
\label{sec:exp3_results}

\begin{table*}[!t]
\centering
\caption{\textbf{Robust accuracy (\%)} under $L_\infty$ attacks at $\epsilon=8/255$ on CIFAR-10. Attacks: PGD-20, Auto-PGD with CE loss (APGD-CE), Auto-PGD with DLR loss (APGD-DLR). We use APGD-CE/APGD-DLR as objective-specific stress tests (components of AutoAttack), rather than claiming evaluation under the full AutoAttack suite. Timing is reported as provided (seconds $\times 200$), otherwise ``--''. Results illustrate NUP-predicted, objective-specific shifts; interventions are demonstrations (not tuned for standardized leaderboard performance). TRADES-PGD is shown only as a reference.}
\label{tab:merged_robustness_results_eps8}
\setlength{\tabcolsep}{6pt}
\renewcommand{\arraystretch}{1.15}
\small
\begin{tabular}{llccccc}
\toprule
\textbf{Model} & \textbf{Method} & \textbf{Clean} & \textbf{PGD-20} & \textbf{APGD-CE} & \textbf{APGD-DLR} & \shortstack{\textbf{Timing} \\ \textbf{(Sec. $\times$ 200)}} \\
\midrule
\multirow{4}{*}{ResNet-18} & Base              & 91.52 & 0.01  & 0.02  & 0.10  & -- \\
& TRADES-PGD            & 74.57          & 44.14 & 43.92 & 42.27 & $\sim$ 100 \\
& ConjMask (Exp.~3) & 85.27       & 77.55 & 77.02 & 1.28  & $\sim$ 20 \\
& ConjMask+LogitReg (Exp.~4) & 87.64       & 62.71 & 36.92 & 43.29 & $\sim$ 30 \\
\midrule
\multirow{4}{*}{ViT-Tiny} & Base              & 84.10 & 0.01  & 0.01  & 0.02  & -- \\
& TRADES-PGD            & 71.95          & 32.95 & 32.46 & 35.51 & $\sim$ 120 \\
& ConjMask (Exp.~3) & 71.49       & 39.19 & 29.49 & 13.96  & $\sim$ 45 \\
& ConjMask+LogitReg (Exp.~4) & 71.24       & 43.15 & 14.59 & 15.91 & $\sim$ 60 \\
\midrule
\multirow{4}{*}{EfficientNet-B0} & Base              & 88.14 & 0.24  & 0.17  & 0.00  & -- \\
& TRADES-PGD            & 56.88 & 30.65 & 30.48 & 28.05 &  $\sim$ 260 \\
& ConjMask (Exp.~3) & 83.91 & 69.10 & 68.28 & 0.20    & $\sim$ 50\\
& ConjMask+LogitReg (Exp.~4) & 81.18 & 48.48 & 21.89 & 28.36 & $\sim$ 65 \\
\bottomrule
\end{tabular}
\vspace{2pt}
\end{table*}

Following Proposition~\ref{prop:exp34_mask}, we study whether a training-time suppression of highly conjugated input components can improve adversarial robustness \emph{without adversarial training}. Here, ConjMask is an objective-specific test of NUP, rather than a solution of comprehensive adversarial robustness.

We start from a separately trained \emph{clean CIFAR-10 checkpoint} for the same architecture. For each mini-batch, we first compute a forward pass and mark samples as \emph{triggered} if either the current prediction is incorrect, or the prediction confidence is below a fixed threshold (default $0.2$). For each triggered image $X$ (flattened as $x$) we compute the input gradient $P=\nabla_X \mathcal{L}$ (flattened as $p$). In the released implementation, this probing step intentionally switches the model to \texttt{eval} mode; moreover, this mode is kept for the subsequent forward passes of the same optimization step. Since the goal of ConjMask is to reduce the CC-Probe $c_{\text{img}}=|\cos(x,p)|$ (Eq.~\eqref{eq:c_img}), we focus on the components that contribute most to this coupling statistic.

To do so, we build a per-pixel importance score by taking the elementwise product between \emph{unit-normalized} image and gradient (done \emph{per channel} in the implementation):
\begin{equation}
\label{eq:exp3_score}
s_{c,j}=\big|\tilde x_{c,j}\,\tilde p_{c,j}\big|,\qquad
\tilde x_c=\frac{x_c}{\|x_c\|_2+\epsilon},\ \tilde p_c=\frac{p_c}{\|p_c\|_2+\epsilon}.
\end{equation}
We then select the \textbf{top-$k$ pixels per channel} according to $s_{c,j}$ and first form a binary mask.
In the released implementation, this mask is subsequently smoothed with a Gaussian kernel to obtain a soft mask
\(M\in[0,1]^{C\times H\times W}\), and the triggered input is replaced by \emph{soft stochastic interpolation}
between the original image and a per-channel Gaussian replacement sampled from CIFAR-10 statistics:
\[
X_{\mathrm{mask}} = X\odot(1-M) + R\odot M .
\]
Thus, ConjMask is a conjugation-guided \emph{soft stochastic replacement}.

We report robust accuracy under $L_\infty$ attacks at $\epsilon=8/255$, using PGD-20 and Auto-PGD variants (CE and DLR losses), together with standard clean accuracy in Table~\ref{tab:merged_robustness_results_eps8}. The following observations can be made:

\begin{enumerate}[label=(\roman*)]
\item \textbf{ConjMask induces a large regime shift under CE-based first-order attacks.}
On CIFAR-10 at $\epsilon=8/255$ under PGD-20 and APGD-CE, clean baselines largely collapse, while ConjMask (Exp.~3) substantially improves robustness on CNNs and without any test-time masking. Importantly, the robustness profile is \emph{objective-dependent}: the same models can remain vulnerable under APGD-DLR, which we interpret as evidence that ConjMask primarily manipulates the CE-gradient coupling channel predicted by the NUP, rather than constituting a universally strong defense across objectives.

\item \textbf{CE robustness can remain objective-dependent.}
The same models that are strong on PGD/APGD-CE can still degrade sharply under APGD-DLR: for ConjMask (Exp.~3), APGD-DLR robust accuracy is $1.28\%$ on ResNet-18, $13.96\%$ on ViT-Tiny, and $0.20\%$ on EfficientNet-B0 (Table~\ref{tab:merged_robustness_results_eps8}).
This ``strong on CE, weaker on DLR'' pattern highlights objective-specific geometry channels, which is consistent with the NUP view that manipulating one coupling channel may not cover all attack objectives.

\item \textbf{ViT benefits less from ConjMask than CNNs.}
For ViT, Exp.~3 is stronger than the clean baseline under CE-based attacks, but it does not surpass TRADES-PGD. We attribute this gap to architectural differences: compared with CNNs, ViTs may exhibit weaker feature compression and preserve more non-pixel-level information pathways, which can reduce the effectiveness of a pixel-level masking intervention.
\end{enumerate}

Importantly, all masking is used only during fine-tuning; at test time, the deployed model always receives the original clean input, with no masking, no stochastic preprocessing, and no gradient-obfuscation mechanism.

\subsection{Exp.~4: LogitReg complements ConjMask beyond the gradient channel}
\label{sec:exp4_results}

Exp.~3 demonstrates that conjugation-guided masking can produce strong resistance to standard gradient attacks, but its near-zero APGD-DLR accuracy indicates a remaining vulnerability that is not explained by the CE-gradient channel alone.

In Exp.~4, we keep \emph{the same ConjMask masking procedure} as in Exp.~3 and add an output-side regularization term, denoted by \emph{LogitReg}, together with an optional auxiliary-view consistency term. The purpose is to complement the CE-gradient coupling channel targeted by ConjMask with score-space stabilization, so that robustness is not restricted to CE-like attack directions.

Concretely,  for Exp.~4 the objective is
\begin{equation}
\label{eq:exp4_main_total}
\mathcal{L}_{\text{Exp4}}
=
\mathcal{L}_{\text{mask}}
+
\mathcal{L}_{\text{logit}}
+
\mathcal{L}_{\text{cons}},
\end{equation}
where \(\mathcal{L}_{\text{mask}}\) is the same masked-input cross-entropy term as in Exp.~3:
\begin{equation}
\label{eq:exp4_main_mask}
\mathcal{L}_{\text{mask}}
=
\mathrm{CE}\!\big(f_\theta(X_{\mathrm{mask}}), y\big).
\end{equation}

The released script uses a shared logit-side stabilizer across ResNet-18, ViT-Tiny, and EfficientNet-B0, namely a centered-logit variance penalty. Let
\begin{equation}
\label{eq:exp4_main_varc}
\bar z = z - \frac{1}{K}\sum_{j=1}^{K} z_j,
\qquad
\mathrm{VarC}(z)=\frac{1}{K}\sum_{j=1}^{K}\bar z_j^2.
\end{equation}
Then
\begin{equation}
\label{eq:exp4_main_logit}
\mathcal{L}_{\text{logit}}
=
\alpha_{\mathrm{logit}}\cdot \frac{1}{2}
\Big(
\mathrm{VarC}(f_\theta(X))
+
\mathrm{VarC}(f_\theta(X_{\mathrm{mask}}))
\Big),
\end{equation}
where \(\alpha_{\mathrm{logit}}=0\) at epoch \(0\), and \(\alpha_{\mathrm{logit}}=\texttt{penalty\_weight}\) thereafter.

When the consistency branch is activated, the auxiliary view is a per-sample mean-baseline image:
\begin{equation}
\label{eq:exp4_main_aux}
X_{\mathrm{aux}}=\bar X \mathbf{1},
\qquad
\bar X=\frac{1}{CHW}\sum_{c,h,w} X_{c,h,w}.
\end{equation}
The consistency term is
\begin{equation}
\label{eq:exp4_main_cons}
\mathcal{L}_{\text{cons}}
=
0.8\,
\mathrm{KL}\!\Big(
\log\mathrm{softmax}\!\big(\tfrac{f_\theta(X_{\mathrm{mask}})}{T}\big)
\,\big\|\,
\mathrm{softmax}\!\big(\tfrac{f_\theta(X_{\mathrm{aux}})}{T}\big)
\Big),
\end{equation}
with \(T=1\). In the released implementation, this branch is enabled for epochs greater than 5.

We evaluate under the same $L_\infty$ setting as Exp.~3 ($\epsilon=8/255$) using PGD-20, APGD-CE, and APGD-DLR. The following observations can be made:

\begin{enumerate}[label=(\roman*)]
\item \textbf{Lightweight stabilizers restore DLR robustness with practical efficiency.}
Adding shared logit regularization and shared auxiliary-view consistency on top of the same ConjMask procedure substantially improves APGD-DLR robustness where Exp.~3 remains weak. On ResNet-18, APGD-DLR increases from $1.28\%$ (Exp.~3) to $43.29\%$ (Exp.~4), close to TRADES-PGD’s $42.27\%$, while maintaining high clean accuracy ($87.64\%$) and requiring no adversarial training, hence lower practical cost. The resulting robustness profile is more balanced than Exp.~3, with non-trivial performance retained under PGD-20 and APGD-CE ($62.71\%$ and $36.92\%$, respectively).

\item \textbf{The gain is architecture-dependent in outcome, reflecting CE--DLR direction mismatch with partial overlap.}
For ViT-Tiny and EfficientNet-B0, Exp.~4 trades part of the PGD/APGD-CE robustness of Exp.~3 for improved APGD-DLR robustness (ViT-Tiny: $13.96\%\to15.91\%$; EfficientNet-B0: $0.20\%\to28.36\%$), but the improvement remains less pronounced than on ResNet-18 and is accompanied by a noticeable drop in PGD-20/APGD-CE.

\end{enumerate}

Overall, the shared LogitReg augmentation can deliver non-trivial robustness improvements with much less training time than adversarial training.

\subsection{Exp.~5: Prefill-only hallucination risk prediction on Benchmark-500}
\label{sec:exp5}

Here, we use ``hallucination'' in an \emph{operational} sense: a response is labeled as hallucinated if it contains confident but unsupported/incorrect reasoning or claims, as judged by external verifiers. In math reasoning, this includes fabricated intermediate steps, invalid transformations, or incorrect final answers presented with high confidence. We emphasize that this is a task-specific reliability label used for evaluating the proposed prefill-only risk signal, rather than a universal definition of hallucination.

We intentionally restrict to a strict \emph{prefill-only, single-backward} setting: all risk scores are computed before generating any answer tokens. Under this constraint, we do not aim to compete with post-hoc hallucination detectors that rely on sampling multiple generations, retrieval, tool use, or LLM-as-a-judge at test time; instead we test whether a pre-hoc geometric signal exists as predicted by the NUP.

We test Proposition~\ref{prop:exp56_llm} by testing whether \emph{low} prompt-level conjugate correlation
signals hallucination risk in a trained LLM.
For each of 500 math problems, we compute the prefill-stage probe on the prompt embeddings:
we obtain the embedding tensor $x$ and its input gradient $p=\nabla_x \mathcal L(x)$ from a single backward pass,
and form $c_{\text{prompt}}$ (Eq.~\eqref{eq:c_prompt}), which is computed \emph{before any answer tokens are generated}.

We then generate a single model response only for evaluation, and use a panel of external LLM judges to assign a
hallucination vote probability $P_{\text{vote}}\in[0,1]$ for that response.

\paragraph{Risk scores}
We evaluated four sample-level probes as hallucination risk scores (higher = more likely to hallucinate):
\begin{itemize}
    \item \textbf{Risk-Cos} (ours): $- c_\text{prompt}$;
    \item \textbf{Risk-Entropy}: mean predictive entropy;
    \item \textbf{Risk-Loss}: mean negative log-likelihood (shifted NLL as defined in Sec.~\ref{sec:exp:probe});
    \item \textbf{Risk-Margin}: negative mean logit margin.
\end{itemize}
Note that \textbf{Risk-Loss} here reflects prompt-level self-supervised NLL, not factual correctness of the final answer;
thus it can behave very differently from a hallucination indicator.

\paragraph{Evaluation}
We report AUROC with bootstrap resampling ($N_{\text{boot}}=2000$) for distinguishing strict-consensus hallucinations
($P_{\text{vote}}=1.0$) from strict-consensus clean responses ($P_{\text{vote}}=0.0$),
and also provide a relaxed-consensus analysis ($P_{\text{vote}}\ge 0.8$) as a robustness check.
Results are summarized in Table~\ref{tab:exp5_auc}. We can observe that:
\begin{enumerate}[label=(\roman*)]
\item \textbf{Standard uncertainty metrics fail.}
Risk-Entropy is near chance, while Risk-Loss is often anti-correlated with hallucination labels,
highlighting that hallucinations can be fluent and low-entropy rather than ``uncertain''.

\item \textbf{Risk-Cos is effective as a prefill-stage geometric signal.}
Risk-Cos achieves AUROC $\approx 0.69$ under both strict and relaxed consensus settings.
This supports the slack-regime interpretation of Proposition~\ref{prop:exp56_llm}:
hallucinations correspond to anomalously low prompt--gradient alignment, indicating weak conditioning
and a larger feasible continuation space prior to decoding.
\end{enumerate}

\begin{table}[t]
\centering
\caption{\textbf{AUROC of prefill-stage hallucination risk signals on Benchmark-500.}
All metrics are computed \emph{without generating any answer tokens} (prefill only).
Risk-Cos (ours) and Risk-Margin are informative, while prompt-level entropy and NLL are weak under this constraint.}
\label{tab:exp5_auc}
\resizebox{\columnwidth}{!}{%
\begin{tabular}{l c c c}
\toprule
\textbf{Metric} & \textbf{AUROC Mean} & \textbf{95\% CI} & \textbf{Interpretation} \\
\midrule
\multicolumn{4}{l}{\textit{Strict Consensus ($P_{\text{vote}}=1.0$ vs.\ Clean)}} \\
\textbf{Risk-Cos (Ours)} & \textbf{0.6939} & [0.58, 0.79] & Informative \\
Risk-Margin & 0.6850 & [0.57, 0.79] & Informative \\
Risk-Entropy & 0.5323 & [0.44, 0.63] & Near chance \\
Risk-Loss & 0.3508 & [0.25, 0.46] & Anti-correlated \\
\midrule
\multicolumn{4}{l}{\textit{Relaxed Consensus ($P_{\text{vote}} \ge 0.8$)}} \\
\textbf{Risk-Cos (Ours)} & \textbf{0.6957} & [0.61, 0.78] & Robust \\
Risk-Margin & 0.6730 & [0.59, 0.75] & Robust \\
Risk-Entropy & 0.5020 & [0.42, 0.59] & Near chance \\
Risk-Loss & 0.3260 & [0.25, 0.41] & Anti-correlated \\
\bottomrule
\end{tabular}%
}
\end{table}

\paragraph{Mechanistic interpretation under Proposition~\ref{prop:exp56_llm}}
Under the \emph{prefill-only} constraint, both \textbf{Risk-Entropy} and \textbf{Risk-Loss} are functions of the
model's next-token distribution over the \emph{prompt} itself; thus they primarily quantify the predictability of the
\emph{prompt wording} under the model rather than the correctness of the (future) solution.
Their weak/negative correlation indicates that ``easy-to-model'' prompts (low $\mathrm{NLL}_{\text{prompt}}$ or low entropy)
can still admit confident but incorrect continuations, i.e., prompt predictability does not control downstream factuality.

In contrast, \textbf{Risk-Cos} targets the \emph{conjugate coupling channel} posited in Proposition~\ref{prop:exp56_llm}.
Since $c_{\text{prompt}}$ serves as a
computable surrogate for the input--gradient covariance channel that enters $\rho_c(u)$ in the NUP, a \emph{larger} coupling magnitude (larger $|\rho_c(u)|$) compresses the effective conjugate volume through the factor
$\sqrt{1-\rho_c(u)^2}$, yielding smaller slack $S_c(u)$ (stronger conditioning), whereas a \emph{smaller} coupling corresponds to
a higher-slack regime (weaker conditioning).
Thus, low $c_{\text{prompt}}$ (high \textbf{Risk-Cos}) indicates that the prompt content is weakly coupled to the loss-sensitive directions, consistent with Proposition~\ref{prop:exp56_llm}'s \emph{under-conditioning} hypothesis:
a high-slack state admits a larger set of feasible continuations prior to decoding, increasing the likelihood of
prior-driven drift that manifests as hallucination.

\subsection{Exp.~6: Prefill-stage Prompt Selection}
\label{sec:exp6}

We evaluate whether a conjugate-correlation score can be used as a \emph{prefill-only} quality filter to select a prompt \emph{before} generating any answer tokens.

\paragraph{Exp. Setup}
We use \textbf{Perturbation-100}, consisting of 100 problems with 5 semantically equivalent (paraphrased) prompt variants per problem. For each problem, we compute a \emph{prefill risk score} for all five candidate prompts \emph{without decoding} any answer tokens. The selection rule itself is therefore prefill-only: for each metric, the prompt with the lowest predicted risk is designated as the selected variant. For offline evaluation, we then generate deterministic answers for \emph{all} five prompt variants and score them with the same judge panel as in Exp.~5. The resulting judged answers are used only to identify the \emph{judge-best set} for each problem (allowing ties among highest-scoring variants) and to compute the Top-1 / Mean-Regret metrics reported below. As a sanity check on the reliability of the judge panel, we also report inter-judge agreement via \textbf{Fleiss' Kappa}.

\paragraph{Prefill-only risk scores}
Each candidate prompt is assigned one scalar risk value (higher $\rightarrow$ worse) computed \emph{before} decoding:
\textbf{Risk-Cos} (ours), \textbf{Risk-Margin}, \textbf{Risk-Entropy}, and \textbf{Risk-Loss (NLL)} (definitions follow Sec.~\ref{sec:exp:probe}).
All reported selection decisions in this experiment use only these prefill scores.

\paragraph{Metrics}
We focus on two selection-oriented metrics that directly reflect the practical utility of a prefill-only prompt filter:

\begin{itemize}[leftmargin=1.2em]
    \item \textbf{Top-1 Hit Rate} ($\uparrow$): the fraction of problems for which the selected prompt belongs to the judge-best set.
    \item \textbf{Mean Regret} ($\downarrow$): the average quality gap between the selected prompt and the best judge score attained within the same five-prompt cluster.
\end{itemize}

In addition, we report \textbf{Fleiss' Kappa} on the judges' binary hallucination annotations as a sanity check on inter-judge consistency.

\begin{table}[t]
\centering
\small
\setlength{\tabcolsep}{8pt}
\caption{Prefill-only prompt selection results on \textbf{Perturbation-100}. For each five-prompt cluster, the prompt with the lowest predicted prefill risk is selected before decoding. \textbf{Risk-Cos} achieves the highest Top-1 Hit Rate and the lowest Mean Regret among all compared methods.}
\label{tab:exp6_results}
\begin{tabular}{lcc}
\toprule
\textbf{Prefill Risk Score} & \textbf{Top-1 Hit Rate} $\uparrow$ & \textbf{Mean Regret} $\downarrow$ \\
\midrule
\textbf{Risk-Cos (ours)} & \textbf{0.83} & \textbf{0.433} \\
Risk-Margin              & 0.81          & 0.591          \\
Risk-Loss (NLL)          & 0.78          & 0.615          \\
Risk-Entropy             & 0.77          & 0.636          \\
\bottomrule
\end{tabular}
\end{table}

\paragraph{Results}
Table~\ref{tab:exp6_results} shows that \textbf{Risk-Cos} achieves the highest \textbf{Top-1 Hit Rate} (\textbf{0.83}) and the lowest \textbf{Mean Regret} (\textbf{0.433}) among all compared prefill-only risk scores. This indicates that, under the prefill-only constraint, Risk-Cos is the strongest criterion on the two primary selection metrics reported in the main text.

\textbf{Risk-Margin} attains a similarly high \textbf{Top-1 Hit Rate} (0.81), but its \textbf{Mean Regret} is larger (0.591). This gap suggests that although Margin often selects a competitive prompt, its failures are more costly when it misses the judge-best variant. By comparison, \textbf{Risk-Loss (NLL)} and \textbf{Risk-Entropy} perform worse on these two primary metrics, with lower hit rates (0.78 and 0.77, respectively) and higher regret (0.615 and 0.636, respectively). 

\paragraph{Judge reliability}
The judge panel exhibits strong agreement on the binary hallucination annotation, with \textbf{Fleiss' Kappa = 0.857}. This serves as a sanity check that the judges are broadly consistent on that binary label and supports the stability of the offline evaluation pipeline. We emphasize, however, that the prompt-ranking score used in Exp.~6 is an aggregate judge score that also depends on additional structured judge fields (see Supplementary Material), so Fleiss' Kappa should not be interpreted as a direct reliability coefficient for that full aggregate score.

Overall, the results of Exp.~6 suggest that \textbf{Risk-Cos} can serve as a useful, decoding-free criterion for selecting among semantically equivalent prompt formulations. MMore importantly, this finding provides empirical evidence consistent with the NUP account developed in this paper: within the prefill stage, stronger prompt--gradient coupling is associated with better-constrained continuations, whereas weaker coupling reflects the under-conditioning regime predicted by Proposition~\ref{prop:exp56_llm}.

\section{Discussion}
\label{sec:discussion}

The primary theoretical value of the NUP is that it provides a principled constraint on what a learned system can simultaneously achieve when its behavior is examined through the joint geometry of inputs and loss sensitivities. In the Robertson--Schr\"odinger form (Theorem~\ref{thm:rs_form_of_nup}), the pair $(\hat x_u,\hat p_u)$ obeys a nontrivial constraint: the boundary-emphasized phase-plane area cannot collapse below a constant, equivalently $\Delta \hat m_u^\star\,\Delta \hat p_u \ge \tfrac12$ (Lemma~\ref{lem:mixed_axis_nup}). At a high level, this resembles a limit statement: within the present operator formulation, there exists a nonzero ``minimum uncertainty budget'' that prevents simultaneously making boundary ambiguity arbitrarily small and boundary sensitivity uniformly low.

\subsection{A conjugate limit underlying boundary behaviors}
Interpreting the operator moments as second-order geometry of the loss-weighted $(x_u,p_u)$ cloud (Sec.~\ref{subsec:geometry_precision_robustness}) turns the inequality into a geometric statement: the effective area $\sqrt{\det \Sigma(u)}=\Delta \hat m_u^\star\,\Delta \hat p_u$ is lower-bounded (Eq.~\eqref{eq:area_bound_formalism_repl}). When the system operates near this limit, improvements in one axis necessarily incur costs in the conjugate axis. This offers a principled way to interpret why optimizing one desideratum (e.g., squeezing boundary errors to raise accuracy) can systematically amplify another failure mode (e.g., sensitivity exploited by adversarial perturbations).

\subsection{Unified interpretation across vision and language}
Within this limit-based view, adversarial fragility and hallucination become two opposite ways of mismanaging the same uncertainty budget. Table~\ref{tab:regime_summary_intro} summarizes how NUP manifests 
differently across modalities.

\begin{table*}[t]
\centering
\caption{\textbf{Two regimes from the NUP.} The NUP constrains an effective conjugate volume \(V=\Delta x\,\Delta p\,\sqrt{1-\rho^2}\ge \kappa\), where \(\rho\) is the \emph{operator} correlation between \(\hat x_u\) and \(\hat p_u\) in the RS inequality. Empirically, we diagnose the \emph{input--gradient coupling channel} entering \(\rho\) using the CC-Probe: in vision we use \(|\cos(x,p)|\) (inputs are normalized), while in LLM prompting we use the centered \(|\cos(\bar x,\bar p)|\).}
\label{tab:regime_summary_intro}
\footnotesize
\setlength{\tabcolsep}{7pt}
\renewcommand{\arraystretch}{1.28}

\begin{tabular}{p{0.20\textwidth} p{0.38\textwidth} p{0.38\textwidth}}
\toprule
 & \textbf{Saturation / Boundary-stress (Vision)} 
 & \textbf{Slack / Under-conditioning (LLM)} \\
\midrule
\textbf{NUP mechanism} 
& \(\ |\rho|\uparrow \Rightarrow \sqrt{1-\rho^2}\downarrow \Rightarrow V\downarrow\ \) (tight, near-bound)
& \(\ |\rho|\downarrow\ (\to 0) \Rightarrow \sqrt{1-\rho^2}\uparrow \Rightarrow V\uparrow\ \) (loose, high slack) \\[1pt]

\textbf{Probe signal} 
& high \(|\cos(x,p)|\)
& low \(|\cos(\bar x,\bar p)|\) \\[1pt]

\textbf{Observed symptom} 
& misclassification \& adversarial fragility
& higher hallucination risk \\[1pt]

\textbf{Actionable interventions} 
& ConjMask + LogitReg
& Prefill Risk Scoring + Prompt Screening/Selection \\
\bottomrule
\end{tabular}
\end{table*}

In discriminative vision, training can push hard samples toward a boundary-stress configuration: the ambiguity thickness becomes small while sensitivity dispersion becomes large, with strong input--gradient coupling concentrating on a stressed tail (Exp.~1--2; Proposition~\ref{prop:exp1_sep}--\ref{prop:exp2_causal}).
In LLM prompting under a prefill-only constraint, the dominant pathology is instead \emph{under-conditioning}: unusually weak prompt--gradient coupling indicates that the prompt does not effectively constrain loss-sensitive directions, leaving high slack and enabling prior-driven drift (Exp.~5--6; Proposition~\ref{prop:exp56_llm}). 

\subsection{From a bound to practical observables and control channels}
A second theoretical contribution is the bridge from the operator inequality to computable quantities. Under the loss-induced state, the symmetrized operator covariance reduces to a scalar covariance of real gradients (Theorem~\ref{thm:cov_reduction}), and random directional probing links projected correlation to cosine (Lemma~\ref{lem:cos_dir_corr}). This justifies the CC-Probe as a single-backward proxy for the coupling channel that modulates proximity to the uncertainty limit. Once coupling is made observable, it can also be used as a practical control target: ConjMask directly suppresses dominant per-component couplings $|x_i p_i|$ and yields strong robustness under PGD-20 and APGD-CE without adversarial training (Exp.~3; Proposition~\ref{prop:exp34_mask}). The CE--DLR gap further indicates that robustness is multi-channel; LogitReg complements ConjMask by stabilizing score-space geometry and broadening robustness under APGD-DLR (Exp.~4).

\subsection{Limitations and future directions}

Our analysis is most informative when there exists a nontrivial boundary-relevant subset that the loss-phase construction can emphasize. 
In near ``pure fitting'' regimes---e.g., when the loss is essentially small on the relevant support, or the input--output mapping behaves close to bijective with little effective compression~\cite{Nikolaou2025LanguageMA}---the weighting $w_c(x)\propto \mathcal L_c(x)^2$ provides little contrast and distinct loss-induced states $\psi_c$ are not meaningfully separated; consequently, the notion of a stressed boundary layer becomes ill-posed and the CC-Probe inclines to lose discriminative power. In addition, the probe is loss-dependent (CE in vision; shifted prompt NLL in language), and Exp.~3 indicates robustness can differ across attack losses (CE vs.\ DLR), motivating multi-loss evaluation (e.g., APGD-CE and APGD-DLR) and defenses that cover multiple geometry channels. Finally, CC-Probe is a proxy rather than a direct estimator of $\rho_c(u)$ or slack $S_c(u)$, and prefill-only LLM use still requires one backward pass on prompt embeddings. 

The NUP framework also suggests a cautionary perspective for future LLM alignment. Currently, LLM hallucinations are predominantly driven by a high-slack (under-conditioned) regime, as detected by the CC-Probe. However, if future training paradigms aggressively reduce this slack to suppress hallucinations---forcing prompts into tight, deterministic couplings with specific continuations---the NUP inequality \(\Delta \hat m_u^\star\,\Delta \hat p_u \ge \tfrac12\) dictates that the system may be pushed toward the opposite extreme. In doing so, we risk inducing the same boundary-stress regime currently plaguing vision models, which would manifest in LLMs as extreme adversarial fragility (e.g., imperceptible prompt perturbations or synonymous substitutions triggering catastrophic output shifts). Ultimately, the NUP suggests that boundary uncertainty may not be fully eliminable within the present formulation; in practice, reducing one form of fragility may increase pressure on another. Sustainable model design might therefore navigate this fundamental trade-off, seeking an intermediate ``Goldilocks'' zone rather than minimizing a single failure axis.

\section{Conclusion}
\label{sec:conclusion}

In this study, we introduced the Neural Uncertainty Principle (NUP), a geometric constraint-based perspective in which input projections and their loss gradients behave as conjugate observables under a loss-induced state. NUP unifies adversarial vulnerability in vision and hallucination in LLMs as opposite extremes on the same conjugate plane—saturation versus slack. 
Based on this principle, we derive a practical single-backward probing method, CC-Probe, which enables the diagnosis of cross-modal boundary anomalies and supports targeted interventions: ConjMask and LogitReg for enhancing robustness in vision tasks, and prefill-stage risk scoring for hallucination detection and prompt selection in language models.
By reframing two isolated reliability challenges as a single conjugate trade-off, the NUP provides a unified perspective for advancing robust and faithful AI.

% =========================
% Reproducibility Checklist
% =========================
\section*{Reproducibility and Artifact Release}
\label{sec:repro}
Training/eval scripts, seeds, logs, and prompt lists are released at \url{https://doi.org/10.57760/sciencedb.29613}. Hardware budgets (GPU-days), wall-clock, memory footprints, exact hyperparameters, and attack settings are included in the Supplementary Material.

\section*{Acknowledgments}
During manuscript preparation, the authors used ChatGPT (OpenAI) for language editing, cross-section consistency checking, and suggestions aimed at improving the clarity and presentation of mathematical exposition. All technical content, including the underlying proofs, derivations, experimental design, implementation, and conclusions, was developed and verified by the authors, who take full responsibility for the final manuscript.

\bibliography{refs} % use the provided BibTeX keys; or paste the .bib content in the submission

 

\begin{IEEEbiography}[{\includegraphics[width=1in,height=1.25in,clip,keepaspectratio]{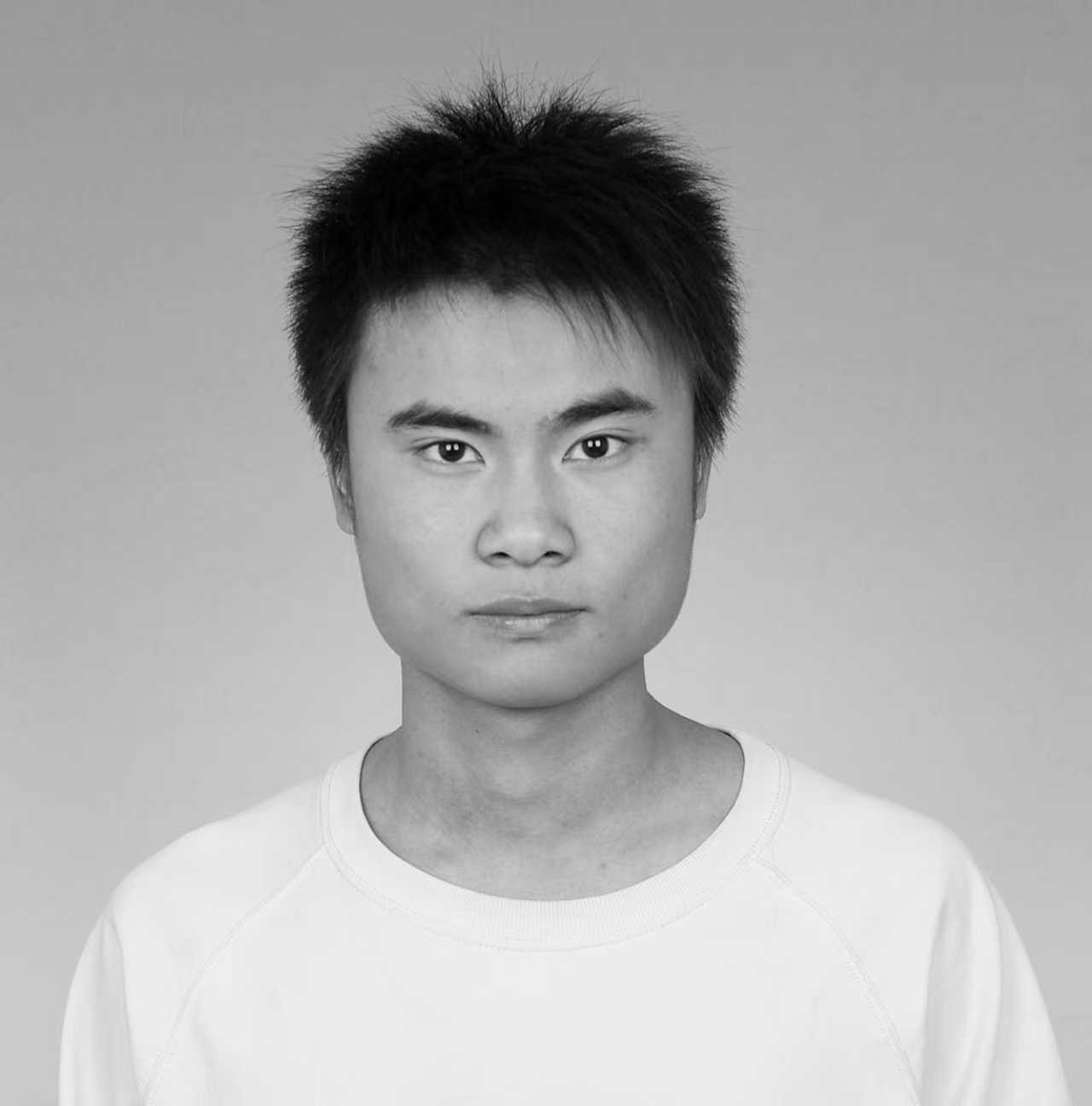}}]{Dong-Xiao Zhang}
Dong-Xiao Zhang received the B.S. degree in Measurement and Control Technology and Instrumentation from the Beijing Institute of Technology in 2013, and the Ph.D. degree in Optical Engineering from the Army Engineering University in 2020. He is currently an Associate Research Fellow with the Northwest Institute of Nuclear Technology, Xi’an, China, where he previously served as an Assistant Research Fellow (2020–2025). His research interests include deep reinforcement learning and decision-making intelligence.
\end{IEEEbiography}

\begin{IEEEbiography}[{\includegraphics[width=1in,height=1.25in,clip,keepaspectratio]{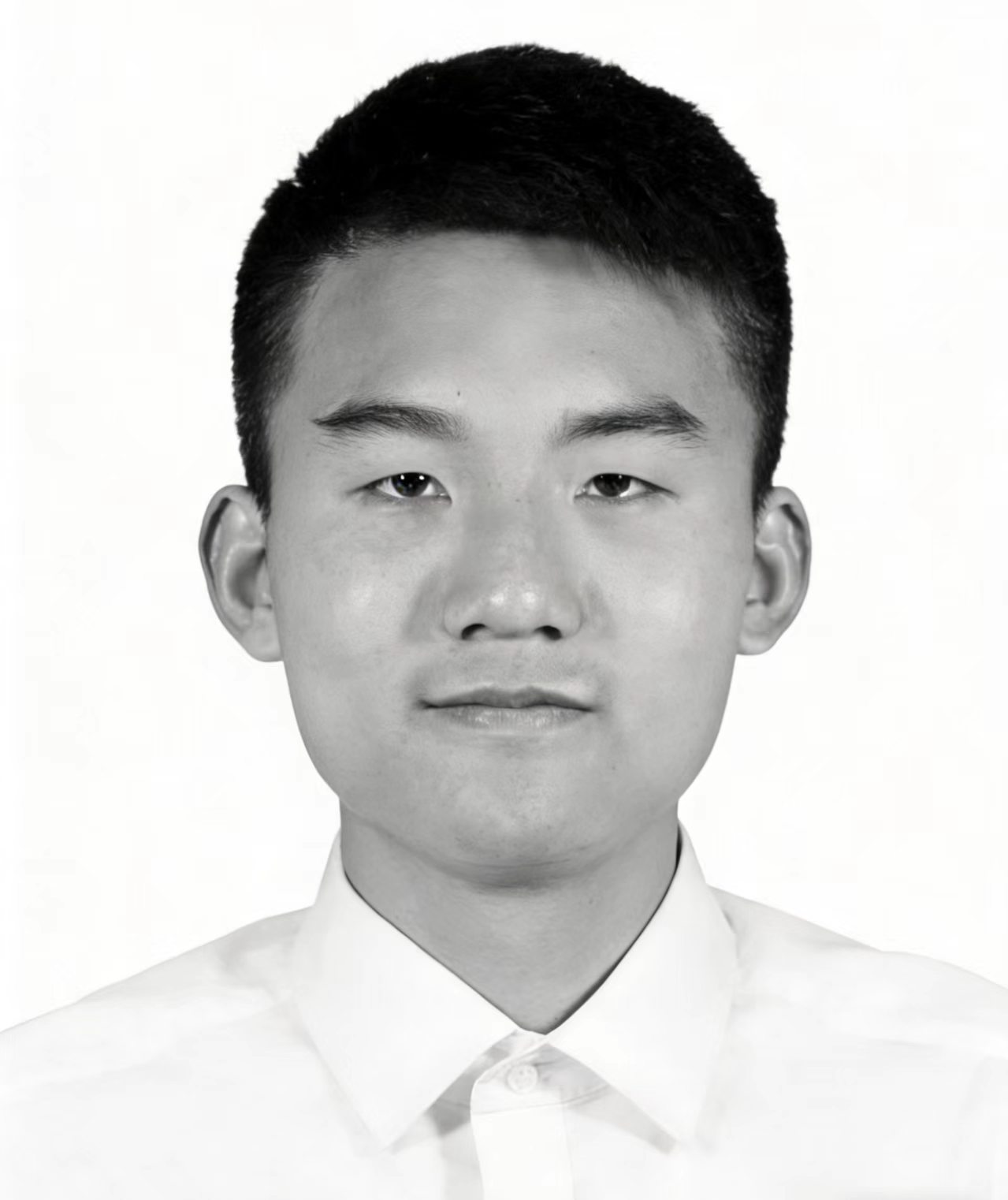}}]{Hu Lou}
Hu Lou received the B.Eng. degree in mechanical and electronic engineering from Jiangsu University of Science and Technology, Zhenjiang, China, in 2020, and the master’s degree in electronic information from Zhejiang University, Hangzhou, China, in 2023. He is currently an Assistant Research Fellow with the Northwest Institute of Nuclear Technology, Xi’an, China, where he previously served as a Research Assistant (2023–2025). His research interests include artificial intelligence and intelligent equipment.
\end{IEEEbiography}

\begin{IEEEbiography}[{\includegraphics[width=1in,height=1.25in,clip,keepaspectratio]{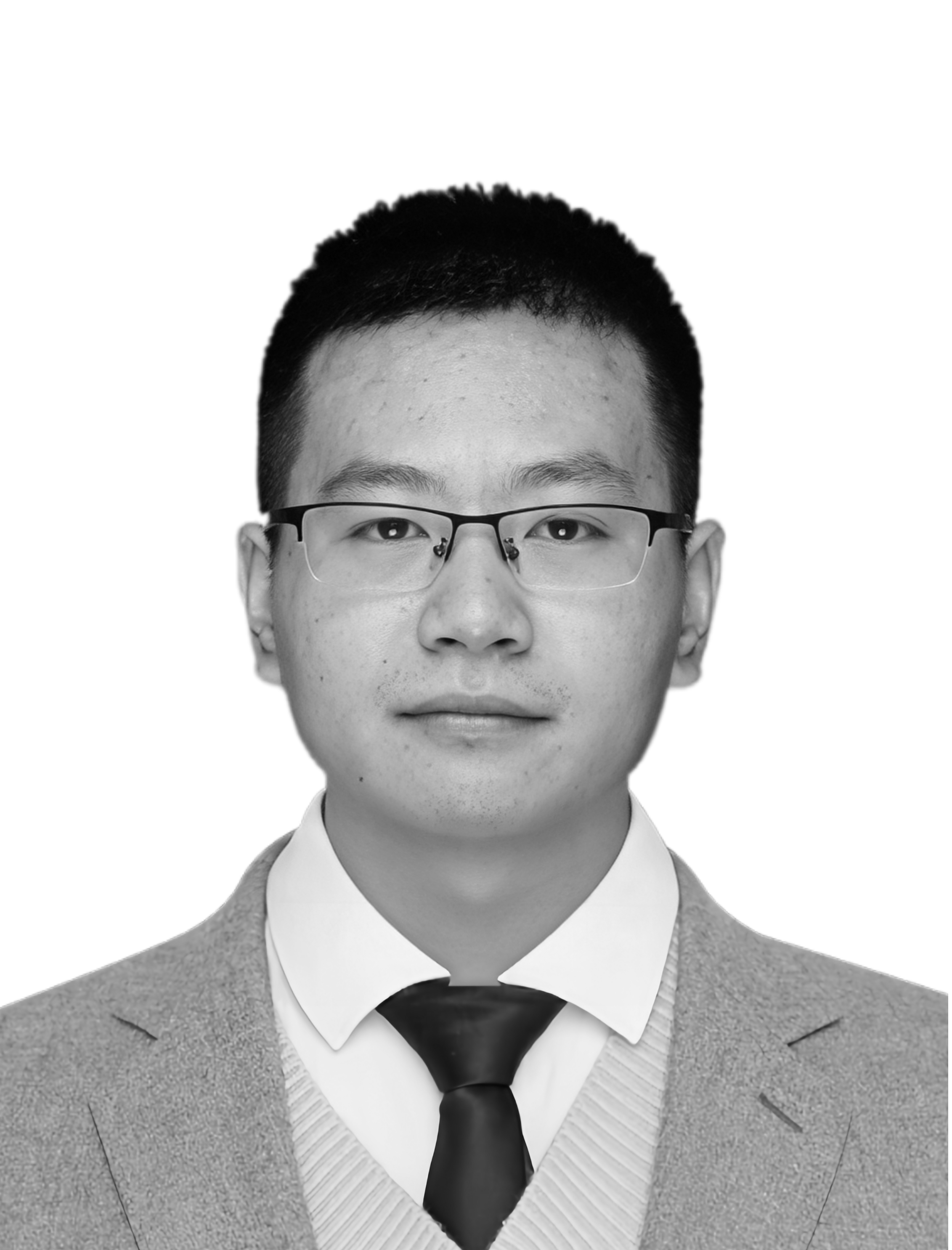}}]{Jun-Jie Zhang}
Jun-Jie Zhang received the B.S. degree in physics from the Second Artillery Engineering University, Xi'an, China, in 2014, and the Ph.D. degree in nuclear science and technology from the Rocket Force University of Engineering, Xi'an, China, in 2020, as a jointly trained Ph.D. student with the University of Science and Technology of China (USTC). He is currently an Associate Research Fellow with the Northwest Institute of Nuclear Technology, Xi'an, China, where he previously served as an Assistant Research Fellow (2020--2024). His research interests include physics-inspired artificial intelligence (physics for AI) and the physics of AI reliability, with emphasis on adversarial robustness and hallucinations in large language models. He received the Eighth China Youth Science and Technology Innovation Award in 2013.
\end{IEEEbiography}

\begin{IEEEbiography}[{\includegraphics[width=1in,height=1.25in,clip,keepaspectratio]{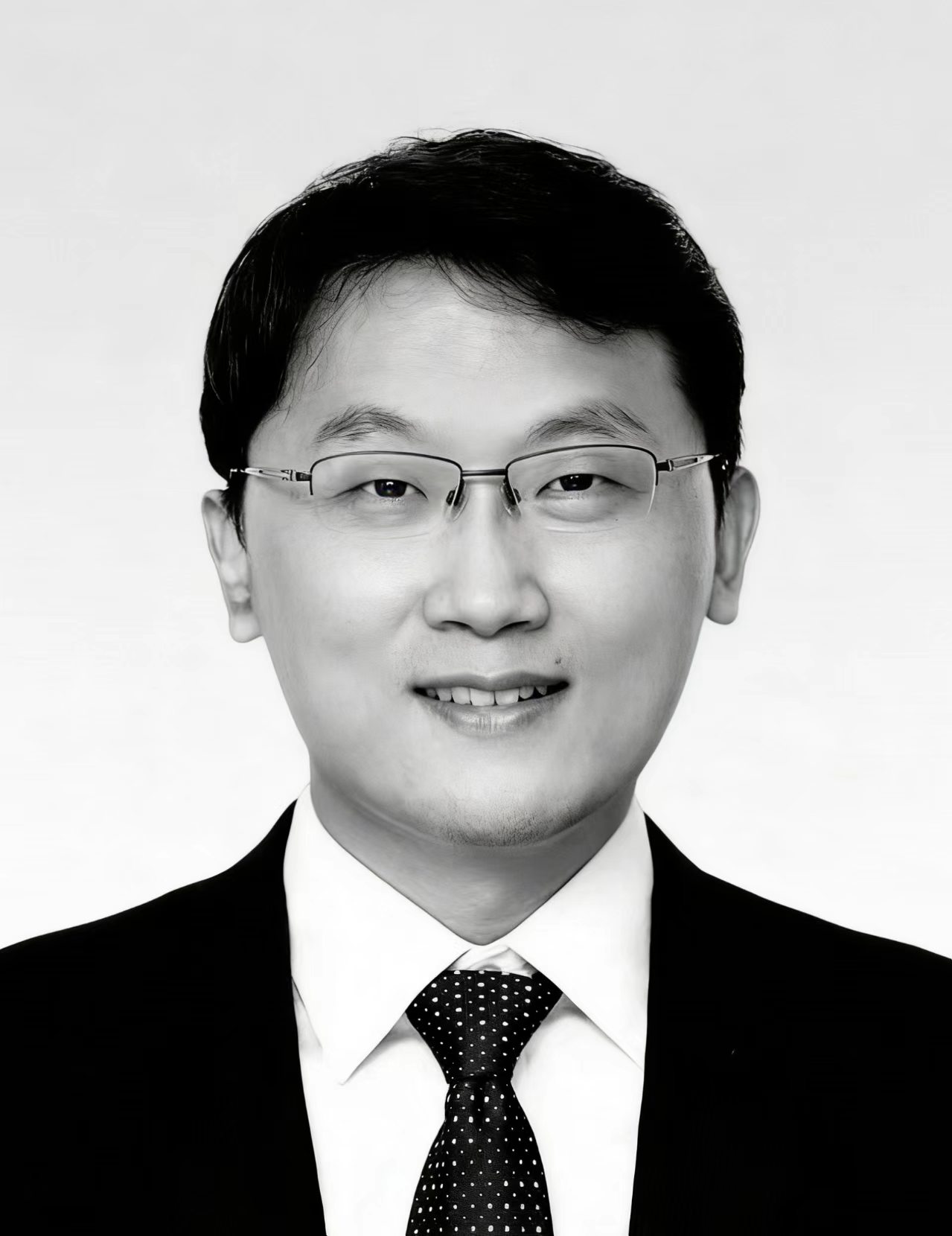}}]{Jun Zhu}
Jun Zhu received the B.S. and Ph.D. degrees from Tsinghua University, Beijing, China, in 2005 and 2009, respectively. He was a postdoctoral fellow in the Machine Learning Department at Carnegie Mellon University, Pittsburgh, PA, USA, from 2009 to 2011. He is currently a Professor at Tsinghua University, where he also serves as the Associate Dean of the Institute for Artificial Intelligence. His research interests primarily focus on developing statistical machine learning methods, including Bayesian methods, deep learning, and generative models, to understand scientific and engineering data.
\end{IEEEbiography}

\begin{IEEEbiography}[{\includegraphics[width=1in,height=1.25in,clip,keepaspectratio]{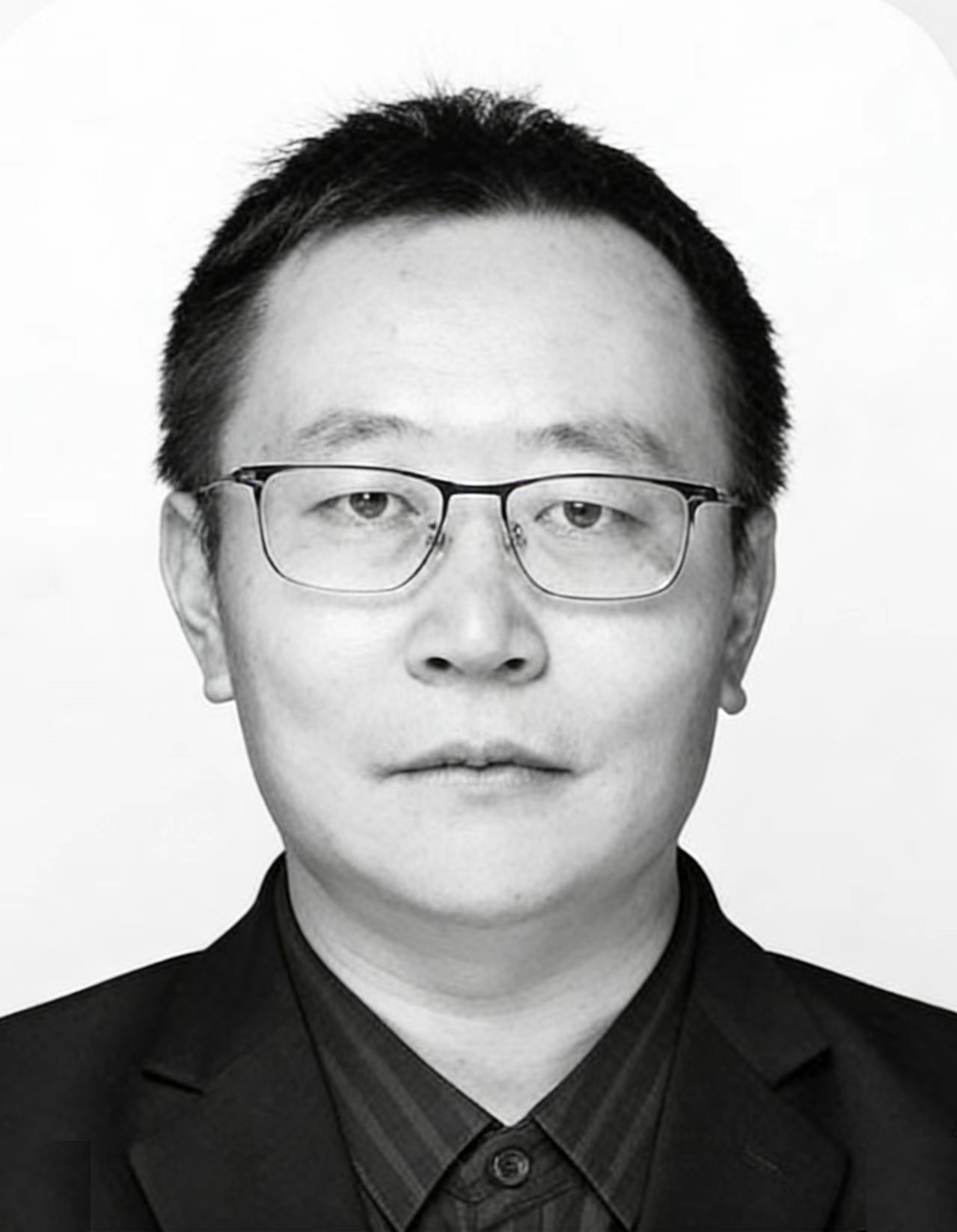}}]{Deyu Meng}
Deyu Meng received the B.S and Ph.D. degrees from Xi’an Jiaotong University, Xi’an, China, in 2001 and 2008, respectively. He was a Visiting Scholar with Carnegie Mellon University, Pittsburgh, PA, USA, from 2012 to 2014. He is currently a professor with the School of Mathematics and Statistics, Xi’an Jiaotong University. His research interests include machine learning, computer vision, and artificial intelligence, with a focus on foundational methodologies such as probabilistic modeling, and interpretable neural networks.
\end{IEEEbiography}

% \vfill

\end{document}

\title{Supplementary Material for Neural Uncertainty Principle: A Unified View of Adversarial Fragility and LLM Hallucination}

\author{Dong-Xiao~Zhang,~Hu~Lou,~Jun-Jie~Zhang,~Jun~Zhu,~Deyu~Meng%
\thanks{Dong-Xiao Zhang, Hu Lou, and Jun-Jie Zhang contributed equally to this work.}%
\thanks{Corresponding authors: Jun-Jie Zhang and Deyu Meng (e-mail: zjacob@mail.ustc.edu.cn; dymeng@mail.xjtu.edu.cn).}%
\thanks{Dong-Xiao Zhang, Hu Lou, and Jun-Jie Zhang are with the Northwest Institute of Nuclear Technology, Xi'an, Shaanxi 710024, China.}%
\thanks{Jun Zhu is with Department of Computer Science and Technology, Institute for AI, BNRist Center, THBI Lab, Tsinghua-Bosch Joint Center for ML, Tsinghua University, Beijing 10084, China (e-mail: dcszj@mail.tsinghua.edu.cn).}%
\thanks{Deyu Meng is with the School of Mathematics and Statistics and the Ministry of Education Key Lab of Intelligent Networks and Network Security, Xi'an Jiaotong University, Xi'an, Shaanxi 710049, China.}%
}

\maketitle

\begin{abstract}
This supplementary material is designed as a guided companion to the main paper, with a compact roadmap to the theory, interpretation, and experimental details behind the Neural Uncertainty Principle (NUP).

Section~\ref{sec:neural_uncertainty_relation} serves as the main entry point. It introduces the loss-induced state, the directional operator pair, the Robertson--Schr\"odinger form of the NUP, the computable bridges to ordinary input-gradient statistics, and the resulting testable predictions for Exp.~1--6.

To help the reader interpret the formalism correctly, Sections~\ref{supp:foundations}--\ref{supp:analogy} clarify how the supplement is organized, why the CC-Probe is representation-dependent, and which parts of the quantum formalism are used only as mathematical structure rather than as physical assumptions.

The full derivations are collected in Section~\ref{app:supp_theory}, including the domain assumptions, the general Robertson--Schr\"odinger inequality, its cosine/correlation form, the specialization to the neural directional pair, and the identities that connect operator moments to the real gradient field and motivate the CC-Probe.

Section~\ref{app:protocal_detail} functions as the protocol hub for the empirical part of the paper. It summarizes the setup of Exp.~1--6, provides the detailed vision and language protocols, and clarifies the implementation role of ConjMask and LogitReg in the released code. Finally, Section~\ref{sec:lcr_verification} reports an appendix-only additional verification experiment based on Latent Conjugate Regularization (LCR), included as an auxiliary causal robustness check outside the main Exp.~1--6 pipeline.

Overall, the supplement is organized to support four reading needs: a concise theory-facing overview, clarification of key interpretive assumptions, complete derivations, and experiment-by-experiment implementation detail.
\end{abstract}

\begin{IEEEkeywords}
Neural Uncertainty Principle, Conjugate Correlation, Adversarial Robustness, LLM Hallucination, Explainable AI.
\end{IEEEkeywords}

\newpage

\section{Formalism of the Neural Uncertainty Principle}
\label{sec:neural_uncertainty_relation}

We provide the formalism of NUP via a loss-weighted operator viewpoint that makes \emph{input--gradient coupling} an explicit state variable, yielding a unified diagnostic signal across the saturation (vision) and slack (LLM) regimes:
\emph{boundary stress}, where overly strong coupling leads to instability,
and \emph{under-conditioning}, where weak coupling leaves the model insufficiently grounded.

\subsection{Neural objects, loss-induced density, and operator statistics}
\label{subsec:nur_operators}

\paragraph{Models, predictions, and losses}
Let $f_\theta$ denote a neural model with parameters $\theta$.
Given an input representation $x$, the model outputs an object $f_\theta(x)$.

When the task involves selecting an outcome from a discrete label set $\mathcal Y$,
we denote the model's predictive distribution by
\[
\pi_\theta(y\mid x), \qquad y\in\mathcal Y,
\]
where $y$ is a \emph{label variable} ranging over $\mathcal Y$.
For a $K$-class classifier, the network outputs logits $z_\theta(x)\in\mathbb R^K$ and
\[
\pi_\theta(y\mid x)=\mathrm{softmax}(z_\theta(x))_y,\qquad y\in\{1,\dots,K\}.
\]
For language models, $f_\theta(x)$ can be the next-token logits and $\pi_\theta(\cdot\mid x)$
is the corresponding vocabulary softmax distribution.

We use $\mathcal L(\cdot)$ for a scalar loss.
Given a \emph{target/condition} index $c$ (e.g., the ground-truth label in classification,
or the target token/sequence in language modeling), we define the per-sample loss as
\begin{equation}
\label{eq:Lc_def_main}
\mathcal{L}_c(x)\;:=\;\mathcal{L}\!\big(f_\theta(x),c\big).
\end{equation}

\paragraph{Input--gradient pair (coordinate vs.\ sensitivity)}
The central objects of our framework are not parameters $\theta$ but the input-level pair
\[
x \in \mathbb{R}^d
\quad\text{and}\quad
p(x)\in\mathbb{R}^d,
\]
where
\begin{equation}
\label{eq:grad_field_def_main}
p(x)\;:=\;\nabla_x \mathcal{L}_c(x)
\end{equation}
is the gradient of the loss with respect to the input.
This gradient is a concrete measure of \emph{sensitivity}: it describes how the loss changes under a small input perturbation.
In discriminative vision models, $p(x)$ captures how sharp the local decision boundary is around $x$.
In language models for inference screening, $x$ denotes the (input-layer) prompt-embedding tensor, and $p(x)$ is
the gradient of the loss with respect to these input embeddings, capturing prompt sensitivity prior to decoding.

\paragraph{Loss-induced density and state expectations}
We construct a loss-induced density that emphasizes sharp decision boundaries between concepts.
Define the normalized loss-phase state
\begin{equation}
\label{eq:loss_phase_state_main}
\psi_c(x)
\;:=\;
A_c(x)\exp\!\big(i\alpha\,\mathcal{L}_c(x)\big),
\end{equation}
where
\begin{equation}
\label{eq:Ac_bc_main}
A_c(x):=\frac{|\mathcal{L}_c(x)|}{\sqrt{\beta_c}},
\qquad
\beta_c:=\int_{\mathcal{X}}\mathcal{L}_c(x)^2\,dx,
\end{equation}
so that $|\psi_c(x)|^2=A_c(x)^2=\mathcal{L}_c(x)^2/\beta_c$ weights samples with larger loss%
\footnote{
We use the Hilbert space $\mathcal H := L^2(\mathcal X,dx)$ with $dx$ a reference (Lebesgue) measure on a continuous ambient input space $\mathcal X\subseteq\mathbb R^d$.
The $dx$-integrals are an analytic idealization that enables integration by parts; in implementations they are replaced by empirical averages over a minibatch (or a class-conditioned subset when $c$ is fixed).
Importantly, the \emph{CC-Probe} used in experiments does \emph{not} require explicitly forming $\psi_c$ or $\beta_c$, nor evaluating $\partial_u\log A_c$; these quantities are only used to connect the operator form to standard input gradients.
We assume $\beta_c\in(0,\infty)$ so that $\psi_c\in L^2(\mathcal X)$ and $\|\psi_c\|_2=1$.
When $\partial_u\log A_c$ appears, we assume $|\mathcal L_c(x)|>0$ almost everywhere under $|\psi_c|^2dx$; numerically we use $\log(|\mathcal L_c(x)|+\epsilon_A)$ with a tiny $\epsilon_A>0$.}. Here $\alpha>0$ is a scalar phase-scale parameter that only rescales the phase-gradient drift term (Eq.~\eqref{eq:app_pu_on_psi}).
Since our empirical CC-Probe is a cosine similarity (scale-invariant), one may set $\alpha=1$ without loss of generality; we keep $\alpha$ to make the drift--amplitude decomposition explicit.
Intuitively, this construction focuses our statistics on the ``frontier'' where classification errors and brittle behavior manifest.

Fix condition $c$ and the state $\psi_c$ from \eqref{eq:loss_phase_state_main}.
For scalar functions $g(x)$, define the loss-induced expectation
\[
\mathbb{E}_c[g]:=\int_{\mathcal{X}} g(x)\,|\psi_c(x)|^2\,dx.
\]
For operators $\hat{O}$ on $\mathcal H$, define the state expectation
\[
\langle \hat{O}\rangle_c := \langle \psi_c,\hat{O}\psi_c\rangle
=\int_{\mathcal{X}}\psi_c(x)^\ast (\hat{O}\psi_c)(x)\,dx.
\]
For multiplication operators (e.g., $\hat{x}_u$), these coincide:
$\langle \hat{x}_u\rangle_c=\mathbb{E}_c[x_u]$.

\subsection{The neural uncertainty relation and its cosine form}
\label{subsec:nur_relation}

\paragraph{Directional reduction and interpretability}
Inputs $x$ and gradients $p(x)$ are both vectors in $\mathbb{R}^d$.
To avoid ambiguous ``vector-valued'' uncertainty statements, we analyze one-dimensional projections.
Let $u\in\mathbb{R}^d$ with $\|u\|_2=1$ denote a direction that selects a scalar coordinate of interest.
Define the projected variables
\begin{equation}
\label{eq:direction_defs_main}
\begin{aligned}
x_u:=&u^\top x,
\\
p_u(x):=&u^\top p(x)=u^\top\nabla_x\mathcal{L}_c(x)=\partial_u\mathcal{L}_c(x).
\end{aligned}
\end{equation}
In practice, $u$ may represent a salient feature direction, an embedding direction, or a random probe direction.

\paragraph{Projected operators and commutator}
For a fixed unit direction $u\in\mathbb R^d$,
we consider the canonical pair of directional operators
\begin{equation}
\label{eq:operators_u_main}
(\hat{x}_u g)(x):=(u^\top x)\,g(x),
\qquad
(\hat{p}_u g)(x):=-i\,\partial_u g(x),
\end{equation}
where $\partial_u:=u^\top\nabla_x$.

These operators are taken on a common dense domain $\mathcal D\subset L^2(\mathcal X)$ on which $\hat p_u$ is self-adjoint (or essentially self-adjoint) and for which boundary terms vanish so that integration by parts is valid.
Then for any sufficiently smooth test function $g(x)$, we have\footnote{We use the standard abbreviations
$[\hat{A},\hat{B}]:=\hat{A}\hat{B}-\hat{B}\hat{A}$ (commutator) and
$\{\hat{A},\hat{B}\}:=\hat{A}\hat{B}+\hat{B}\hat{A}$ (anticommutator).}
\begin{align}
[\hat x_u,\hat p_u] g
&= \hat x_u(-i\partial_u g) - (-i\partial_u)((u^\top x) g) \nonumber \\
&= -i (u^\top x)\partial_u g
   + i\big[(u^\top x)\partial_u g + g\big] \nonumber \\
&= i g.
\end{align}
Therefore the commutator holds on $\mathcal D$,
\begin{equation}
\label{eq:commutator_main}
[\hat x_u,\hat p_u]=i\mathbb I,
\end{equation}
where $\mathbb I$ denotes the identity operator. 

Consequently, for any normalized state $\psi_c\in\mathcal D$,
\begin{equation}
\big|\langle\psi_c,[\hat x_u,\hat p_u]\psi_c\rangle\big|
=\big|\langle\psi_c,i\mathbb I\,\psi_c\rangle\big|
=1.
\end{equation}

\textit{Remark.}
If $\mathcal X$ is bounded and boundary terms do not vanish for the chosen self-adjoint extension of $-i\partial_u$,
the \emph{formal} commutator \eqref{eq:commutator_main} still holds on $\mathcal D$,
but the \emph{state expectation} $\langle[\hat x_u,\hat p_u]\rangle_c$ may deviate from $i$.
All subsequent cosine/volume forms remain valid with the commutator scale
$\kappa:=\tfrac12\big|\langle[\hat x_u,\hat p_u]\rangle_c\big|$. 

The operator $\hat p_u=-i\partial_u$ acts on functions in $L^2(\mathcal X)$ and should not be confused with the real scalar field $p_u(x)=\partial_u\mathcal L_c(x)$ in Eq.~\eqref{eq:direction_defs_main}. The link is created by the loss-phase choice in Eq.~\eqref{eq:loss_phase_state_main}, which makes $\hat p_u\psi_c$ contain a real drift term proportional to $p_u(x)$ (Eq.~\eqref{eq:app_pu_on_psi}).

\paragraph{Centered variables, dispersions, and operator covariance}
For any unit direction $u$, define the centered scalar coordinate
$\Delta x_u:=x_u-\mathbb{E}_c[x_u]$ and the centered operators
$\Delta\hat{x}_u:=\hat{x}_u-\langle \hat{x}_u\rangle_c$ and
$\Delta\hat{p}_u:=\hat{p}_u-\langle \hat{p}_u\rangle_c$.
Here $\Delta x_u$ denotes the centered scalar random variable, whereas
$\Delta\hat x_c(u)=\sqrt{\mathbb{E}_c[(\Delta x_u)^2]}$ denotes the corresponding dispersion (standard deviation).
The (loss-weighted) dispersions are
\begin{equation}
\label{eq:disp_defs_main}
\begin{aligned}
(\Delta \hat{x}_c(u))^2
&:=\mathbb{E}_c[(\Delta x_u)^2], \\
(\Delta \hat{p}_c(u))^2
&:=\big\|\big(\hat{p}_u-\langle \hat{p}_u\rangle_c\big)\psi_c\big\|_{L^2(\mathcal X)}^2,
\end{aligned}
\end{equation}
and the symmetrized (operator) covariance is
\begin{equation}
\label{eq:cov_op_main}
\mathrm{Cov}_c(\hat{x}_u,\hat{p}_u)
:=\frac12\Big\langle \Delta\hat{x}_u\Delta\hat{p}_u+\Delta\hat{p}_u\Delta\hat{x}_u\Big\rangle_c.
\end{equation}

\noindent\textit{Notation alignment with the main text.}
When the condition $c$ is fixed, we occasionally abbreviate
$\Delta \hat{x}_c(u)$ and $\Delta \hat{p}_c(u)$ as
$\Delta \hat{x}_u$ and $\Delta \hat{p}_u$,
respectively, to match the notation used in the main text.

\textit{Intuitively,}
because $\mathbb E_c[\cdot]$ is taken under the loss-induced density $|\psi_c(x)|^2\propto \mathcal L_c(x)^2$,
$(\Delta \hat{x}_c(u))^2$ measures how widely the \emph{high-loss (boundary-relevant)} mass
is spread along the coordinate $x_u$.
Similarly, $(\Delta \hat{p}_c(u))^2$ measures the dispersion of the corresponding sensitivity operator along $u$,
and $\mathrm{Cov}_c(\hat x_u,\hat p_u)$ quantifies how these two dispersions are coupled under the same loss weighting.
Here and below, for a scalar function $g(x)$ we write
$\mathrm{Var}_c(g):=\mathbb E_c[(g-\mathbb E_c[g])^2]$.

\begin{theorem}[Neural Uncertainty Relation]
\label{thm:nur_rs}
For the canonical pair $(\hat x_u,\hat p_u)$ and the state $\psi_c$, assume
$\psi_c\in\mathrm{Dom}(\hat x_u)\cap\mathrm{Dom}(\hat p_u)$ with finite dispersions
$0<\Delta\hat x_c(u)<\infty$ and $0<\Delta\hat p_c(u)<\infty$.
Then
\begin{equation}
\label{eq:nur_rs}
(\Delta \hat{x}_c(u))^2(\Delta \hat{p}_c(u))^2 \ge
\frac14\big|\langle\psi_c,[\hat{x}_u,\hat{p}_u]\psi_c\rangle\big|^2
+\mathrm{Cov}_c(\hat{x}_u,\hat{p}_u)^2,
\end{equation}
where $\mathrm{Cov}_c(\hat{x}_u,\hat{p}_u):=
\frac12\langle\psi_c,\{\Delta\hat x_u,\Delta\hat p_u\}\psi_c\rangle$
is the symmetrized covariance.
\end{theorem}

\noindent\textit{Remark.}
This is exactly the Robertson--Schr\"odinger form used in the main text,
but written here with explicit domain assumptions and the notation
$\Delta\hat x_c(u),\Delta\hat p_c(u)$.
When the condition $c$ is fixed, we abbreviate them as
$\Delta\hat x_u,\Delta\hat p_u$ to match the main text.

\paragraph{Cosine (correlation) form and the effective feasible volume}
Define the operator correlation coefficient
\begin{equation}
\label{eq:rho_def_main}
\rho_c(u)
:=
\frac{\mathrm{Cov}_c(\hat{x}_u,\hat{p}_u)}{\Delta \hat{x}_c(u)\,\Delta \hat{p}_c(u)}
\in[-1,1].
\end{equation}
Then \eqref{eq:nur_rs} is equivalent to
\begin{equation}
\label{eq:nur_cosine}
(\Delta \hat{x}_c(u))^2(\Delta \hat{p}_c(u))^2\bigl(1-\rho_c(u)^2\bigr)
\;\ge\;
\frac14\big|\langle[\hat{x}_u,\hat{p}_u]\rangle_c\big|^2,
\end{equation}
and with \eqref{eq:commutator_main},
\begin{equation}
\label{eq:nur_cosine_canonical}
\Delta \hat{x}_c(u)\,\Delta \hat{p}_c(u)\;\ge\;\frac{1}{2\sqrt{1-\rho_c(u)^2}}.
\end{equation}

We interpret the left-hand side of \eqref{eq:nur_cosine} as an \emph{effective feasible conjugate volume}:
\begin{equation}
\label{eq:V_def_main}
\begin{aligned}
V_c(u)
&:=
\Delta \hat{x}_c(u)\,\Delta \hat{p}_c(u)\,\sqrt{1-\rho_c(u)^2},
\\
\kappa&:=\frac12\big|\langle[\hat{x}_u,\hat{p}_u]\rangle_c\big|.
\end{aligned}
\end{equation}
so \eqref{eq:nur_cosine} becomes the hard constraint $V_c(u)\ge \kappa$.

To connect Eq.~\eqref{eq:nur_cosine} to the mixed-axis form used in the main text, define
\[
\hat m_u(\lambda):=\hat x_u+\lambda \hat p_u,
\qquad
\Delta \hat m_u^\star:=\min_{\lambda}\Delta \hat m_u(\lambda).
\]
A direct minimization gives
\begin{equation}
\label{eq:mixed_axis_identity_si}
(\Delta \hat m_u^\star)^2
=
(\Delta \hat{x}_c(u))^2\bigl(1-\rho_c(u)^2\bigr).
\end{equation}

\begin{lemma}[Mixed-axis form of NUP]
\label{lem:mixed_axis_nup_si}
With $\Delta \hat m_u^\star$ defined by Eq.~\eqref{eq:mixed_axis_identity_si},
Theorem~\ref{thm:nur_rs} implies
\begin{equation}
\label{eq:mixed_axis_nup_si}
\boxed{\ \Delta \hat m_u^\star\,\Delta \hat p_c(u)\ge \kappa\ }.
\end{equation}
\end{lemma}

\noindent\textit{Under the canonical commutator in Eq.~\eqref{eq:commutator_main}, this reduces to }
\[
\Delta \hat m_u^\star\,\Delta \hat p_c(u)\ge \tfrac12,
\]
\textit{which is exactly the main-text mixed-axis NUP after the notational identification
$\Delta \hat p_c(u)\leftrightarrow \Delta \hat p_u$.}

A convenient summary statistic is the slack ratio\footnote{Equality in \eqref{eq:nur_rs} (hence $S_c(u)=1$) occurs when the Cauchy--Schwarz step is tight, i.e., when $\Delta\hat x_u\psi_c$ and $\Delta\hat p_u\psi_c$ are linearly dependent.
This does not force $|\rho_c(u)|$ to be large: both $\rho_c(u)=0$ and $\rho_c(u)\neq 0$ can be compatible with saturation.
What $|\rho_c(u)|$ controls is the effective volume factor $\sqrt{1-\rho_c(u)^2}$ in \eqref{eq:V_def_main}: for fixed dispersions, larger $|\rho_c(u)|$ reduces $V_c(u)$ and therefore moves the system closer to violating the bound, motivating alignment observables as boundary-stress indicators.}
\begin{equation}
\label{eq:slack_ratio_main}
S_c(u):=\frac{V_c(u)}{\kappa}\ge 1,
\end{equation}
which distinguishes near-saturation ($S_c(u)\approx 1$) from slack-dominated operation ($S_c(u)\gg 1$). 

\subsection{Making the operator statistics computable: drift, variance, and covariance bridges}
\label{subsec:nup_bridges}

We provide exact identities that connect the operator quantities
$\Delta \hat p_c(u)$, $\mathrm{Cov}_c(\hat x_u,\hat p_u)$, and $\rho_c(u)$
to ordinary loss-weighted statistics of the real input-gradient $p(x)=\nabla_x\mathcal L_c(x)$.
These bridges are used to motivate computable proxies such as the CC-Probe.

\paragraph{Loss-phase drift identity (link to the real gradient projection)}
Recall from Eq.~\eqref{eq:direction_defs_main} that $p_u(x)=\partial_u\mathcal L_c(x)$.
A direct differentiation of $\psi_c(x)=A_c(x)e^{i\alpha\mathcal{L}_c(x)}$ yields
\begin{equation}
\label{eq:app_pu_on_psi}
\hat{p}_u\psi_c
=-i\partial_u\psi_c
=\Big(\alpha A_c(x)\,p_u(x) - i\,\partial_u A_c(x)\Big)e^{i\alpha\mathcal{L}_c(x)}.
\end{equation}
This identity makes explicit why the factor $i$ in $\hat{p}_u=-i\partial_u$ is convenient: it exposes a \emph{real} drift term proportional to $p_u(x)$, plus an orthogonal amplitude term via $\partial_u A_c$.

\paragraph{Variance decomposition and conservative proxies}
Using \eqref{eq:app_pu_on_psi} and $|\psi_c|^2=A_c^2$, one obtains
\begin{equation}
\label{eq:app_var_decomp}
\begin{aligned}
\|\hat{p}_u\psi_c\|_{L^2(\mathcal X)}^2
=&
\alpha^2\int_{\mathcal X}A_c(x)^2\,p_u(x)^2\,dx
+\int_{\mathcal X}(\partial_u A_c(x))^2\,dx\\
=&
\alpha^2\,\mathbb E_c[p_u(x)^2]+\mathbb E_c\!\big[(\partial_u\log A_c(x))^2\big].
\end{aligned}
\end{equation}
Likewise, for the centered operator one has
\begin{equation}
\label{eq:app_centered_var_decomp}
\begin{aligned}
(\Delta \hat{p}_c(u))^2
= &
\alpha^2\,\mathrm{Var}_c\!\big(p_u(x)\big)\\
&+\mathbb{E}_c\!\Big[\big(\partial_u \log A_c(x)-\mathbb{E}_c[\partial_u \log A_c]\big)^2\Big],
\end{aligned}
\end{equation}
where $\mathrm{Var}_c(p_u):=\mathbb{E}_c\!\big[(p_u-\mathbb{E}_c[p_u])^2\big]$
\footnote{Under the same boundary/decay assumptions, so that $\int_{\mathcal X}\partial_u(A_c(x)^2)\,dx=0$, we have
$\mathbb E_c[\partial_u\log A_c]=\int_{\mathcal X}A_c(x)^2\,\partial_u\log A_c(x)\,dx=\int_{\mathcal X}A_c(x)\,\partial_u A_c(x)\,dx
=\frac12\int_{\mathcal X}\partial_u(A_c(x)^2)\,dx
=0.$}.
The second term in \eqref{eq:app_centered_var_decomp} is nonnegative, implying the conservative bound
\begin{equation}
\label{eq:disp_bridge_main}
\begin{aligned}
(\Delta \hat{p}_c(u))^2
=&
\alpha^2\,\mathrm{Var}_c\!\big(p_u(x)\big)
+\mathrm{Var}_c\!\big(\partial_u\log A_c(x)\big)\\
\;\ge\;&
\alpha^2\,\mathrm{Var}_c\!\big(p_u(x)\big).
\end{aligned}
\end{equation}

\paragraph{Operator-to-scalar covariance reduction}
\label{subsec:op_cov_reduction}
For scalar functions $g(x),h(x)$ define the loss-weighted scalar covariance
\begin{equation}
\label{eq:sc_cov_def}
\begin{aligned}
\mathrm{Cov}^{\mathrm{sc}}_c(g,h)
:=&
\mathbb E_c\!\big[(g-\mathbb E_c[g])(h-\mathbb E_c[h])\big],
\\
\mathrm{Var}_c(g):=&\mathrm{Cov}^{\mathrm{sc}}_c(g,g),
\end{aligned}
\end{equation}
where $\mathbb E_c[\cdot]=\int_{\mathcal X}(\cdot)\,|\psi_c(x)|^2\,dx$.

\begin{theorem}[Covariance reduction under the loss-induced state]
\label{thm:cov_reduction_si}
Let $\psi_c(x)=A_c(x)e^{i\alpha\mathcal L_c(x)}$ with
$A_c(x)^2=|\psi_c(x)|^2=w_c(x)$ as defined in Eq.~\eqref{eq:loss_phase_state_main}.
Assume that boundary terms vanish for integration by parts along direction $u$, and let
$\hat x_u$ and $\hat p_u=-i\partial_u$ be defined as in Eq.~\eqref{eq:operators_u_main}
on a common dense domain $\mathcal D$.
Then the symmetrized operator covariance satisfies the exact identity
\begin{equation}
\label{eq:op_cov_equals_scalar_cov}
\mathrm{Cov}_c(\hat x_u,\hat p_u)
=
\alpha\,\mathrm{Cov}^{\mathrm{sc}}_c(x_u,p_u).
\end{equation}
\end{theorem}

\begin{proof}
Write $\Delta x_u:=x_u-\mathbb{E}_c[x_u]$ and recall the centered operators
$\Delta\hat x_u:=\hat x_u-\langle\hat x_u\rangle_c\mathbb{I}$ and
$\Delta\hat p_u:=\hat p_u-\langle\hat p_u\rangle_c\mathbb{I}$.
Using \eqref{eq:app_pu_on_psi},
\[
\hat p_u\psi_c
=\Big(\alpha A_c(x)\,p_u(x)-i\,\partial_u A_c(x)\Big)e^{i\alpha\mathcal L_c(x)}.
\]
A direct substitution yields
\[
\langle \Delta \hat x_u\psi_c,\Delta \hat p_u\psi_c\rangle
=
\alpha\,\mathbb E_c[\Delta x_u\,p_u]
\;-\; i\,\mathbb E_c[\Delta x_u\,\partial_u\log A_c].
\]
By definition, the symmetrized covariance is the real part:
\begin{align*}
\mathrm{Cov}_c(\hat x_u,\hat p_u)
&=\mathrm{Re}\,\langle \Delta \hat x_u\psi_c,\Delta \hat p_u\psi_c\rangle\\
&=\alpha\,\mathbb E_c[\Delta x_u\,\Delta p_u]\\
&=\alpha\,\mathrm{Cov}^{\mathrm{sc}}_c(x_u,p_u),
\end{align*}
which proves the claim.
\end{proof}

\noindent\textit{Remark.}
Theorem~\ref{thm:cov_reduction_si} is an \emph{exact} bridge for the covariance term in the RS inequality.
The dispersion $\Delta\hat p_c(u)$, however, also contains the nonnegative amplitude-variation term
$\mathrm{Var}_c(\partial_u\log A_c)$ (Eq.~\eqref{eq:disp_bridge_main}),
which attenuates the operator correlation $\rho_c(u)$ relative to the purely scalar correlation.
This will be made explicit by Eq.~\eqref{eq:rho_attenuation}.

\paragraph{Attenuation bound for the correlation factor}
Define the loss-weighted scalar correlation (when the denominator is nonzero)
\[
\mathrm{Corr}^{\mathrm{sc}}_c(g,h)
:=
\frac{\mathrm{Cov}^{\mathrm{sc}}_c(g,h)}{\sqrt{\mathrm{Var}_c(g)\,\mathrm{Var}_c(h)}}.
\]
Combining Theorem~\ref{thm:cov_reduction_si} with the first line of \eqref{eq:disp_defs_main} and \eqref{eq:disp_bridge_main} gives
\begin{equation}
\label{eq:rho_attenuation}
\begin{aligned}
\rho_c(u)
=&
\frac{\alpha\,\mathrm{Cov}^{\mathrm{sc}}_c(x_u,p_u)}{\Delta\hat x_c(u)\,\Delta\hat p_c(u)}\\
=&
\frac{\mathrm{Corr}^{\mathrm{sc}}_c(x_u,p_u)}
{\sqrt{1+\mathrm{Var}_c(\partial_u\log A_c)/(\alpha^2\mathrm{Var}_c(p_u))}},
\end{aligned}
\end{equation}
and in particular $|\rho_c(u)|\le \big|\mathrm{Corr}^{\mathrm{sc}}_c(x_u,p_u)\big|$.

\noindent
\textit{Remark.} A technical subtlety is the domain/self-adjointness of the momentum operator: \(\hat p_u=-i\partial_u\) is (essentially) self-adjoint only under appropriate boundary/decay conditions. The phase choice in \eqref{eq:loss_phase_state_main} does \emph{not} make \(\hat p_u\) self-adjoint; rather, it ensures that applying \(\hat p_u\) to \(\psi_c\) exposes a \emph{real drift term} proportional to the real gradient projection $p_u(x)$, plus an orthogonal amplitude term via $\partial_u A_c$ in \eqref{eq:app_pu_on_psi}. This is what enables the operator-to-scalar covariance reduction and motivates a real-valued cosine proxy.

\paragraph{The empirical CC-Probe}
The RS cosine factor $\sqrt{1-\rho_c(u)^2}$ in \eqref{eq:nur_cosine} depends on the operator correlation $\rho_c(u)$.
Equation~\eqref{eq:rho_attenuation} shows that $\rho_c(u)$ is governed by the loss-weighted \emph{scalar} correlation $\mathrm{Corr}^{\mathrm{sc}}_c(x_u,p_u)$, possibly attenuated by the nonnegative amplitude term $\mathrm{Var}_c(\partial_u\log A_c)$.
By Theorem~\ref{thm:cov_reduction_si}, the covariance channel entering $\rho_c(u)$ reduces exactly to the \emph{real} input--gradient coupling $\mathrm{Cov}^{\mathrm{sc}}_c(x_u,p_u)$.
We now introduce a cheap \emph{single-sample} proxy for this coupling channel.

Define population-centered vectors under the loss-induced density:
\begin{equation}
\begin{aligned}
\bar x:=&x-\mu_c,\qquad \bar p:=p(x)-\nu_c,
\end{aligned}
\end{equation}
where 
\begin{equation}
\mu_c:=\mathbb E_c[x], \qquad \nu_c:=\mathbb E_c[p(x)].
\end{equation}
Then for any unit $u$, the centered projections satisfy
\[
x_u-\mathbb E_c[x_u]=u^\top\bar x,\qquad
p_u-\mathbb E_c[p_u]=u^\top\bar p,
\]
and therefore the scalar covariance appearing in Theorem~\ref{thm:cov_reduction_si} admits the simple form
\begin{equation}
\label{eq:cov_dir_centered}
\mathrm{Cov}^{\mathrm{sc}}_c(x_u,p_u)
=\mathbb E_c\!\big[(u^\top\bar x)(u^\top\bar p)\big].
\end{equation}
That is, for each direction $u$, the covariance channel is the loss-weighted average (over samples) of the directional products $(u^\top\bar x)(u^\top\bar p)$.

For a \emph{fixed sample} (i.e., fixed $\bar x,\bar p$), draw a random direction $u\sim\mathrm{Unif}(\mathbb S^{d-1})$ and define
$X_u:=u^\top\bar x$ and $P_u:=u^\top\bar p$.
Using isotropy $\mathbb E_u[uu^\top]=\tfrac{1}{d}\mathbb I$, we obtain
\begin{equation}
\label{eq:dot_as_dir_average}
\begin{aligned}
\mathbb E_u[X_uP_u]=&\frac{1}{d}\,\bar x^\top \bar p,\\
\mathbb E_u[X_u^2]=&\frac{1}{d}\|\bar x\|_2^2,\\
\mathbb E_u[P_u^2]=&\frac{1}{d}\|\bar p\|_2^2.
\end{aligned}
\end{equation}
Consequently, the Pearson correlation over random directions equals the centered cosine similarity, yielding a canonical per-sample alignment score.

\begin{lemma}[Centered directional-probe correlation and cosine]
\label{lem:cc_probe_identity}
Fix any nonzero centered vectors $\bar x,\bar p\in\mathbb R^d$.
Let $u\sim\mathrm{Unif}(\mathbb S^{d-1})$ and define
$X_u:=u^\top\bar x$ and $P_u:=u^\top\bar p$.
Then
\begin{equation}
\label{eq:cc_probe_identity}
\mathrm{Corr}_u(X_u,P_u)
=
\frac{\bar x^\top \bar p}{\|\bar x\|_2\,\|\bar p\|_2}
=
\cos(\bar x,\bar p).
\end{equation}
\end{lemma}

\noindent\textit{Remark.}
This is the centered generalization of the main-text identity
$\mathrm{Corr}_u(x_u,p_u)=\cos(x,p)$.
The centering here is theoretical; in the actual language experiments we implement a
\emph{per-sample mean-centered} version by subtracting the scalar means of the flattened
prompt and gradient vectors (see Eq.~\eqref{eq:exp5_cprompt_code_abs}),
while in vision we use the uncentered practical image-space form below because the inputs
are already normalized.

\noindent
\textit{Definition (CC-Probe).}
Motivated by Lemma~\ref{lem:cc_probe_identity}, we define the centered CC-Probe as
\begin{equation}
\label{eq:cc_probe_def}
c_{\text{centered}}(x)
:=
\big|\cos(\bar x,\bar p)\big|
=
\frac{\big|\bar x^\top\bar p\big|}{\|\bar x\|_2\,\|\bar p\|_2}.
\end{equation}
The modality-specific instantiations used in the experiments are:
(i) the image-space score
\begin{equation}
\label{eq:cc_probe_img_si}
c_{\text{img}}(x):=\big|\cos(x,p(x))\big|
=\frac{\big|x^\top p(x)\big|}{\|x\|_2\,\|p(x)\|_2}
\end{equation}
for vision, and
(ii) the prompt-side per-sample mean-centered score in
Eq.~\eqref{eq:exp5_cprompt_code_abs} for language.
Both scores are scale-free and require only a single backward pass.

Equation~\eqref{eq:cov_dir_centered} shows that, for any fixed direction $u$, the covariance channel entering $\rho_c(u)$ is the loss-weighted average of $(u^\top\bar x)(u^\top\bar p)$ over samples.
Samples for which this directional product is large across \emph{many} directions will therefore exert a larger influence on typical directional covariances, hence on the operator covariance term via Theorem~\ref{thm:cov_reduction_si}.
Equation~\eqref{eq:dot_as_dir_average} makes this intuition explicit: for a fixed sample, the \emph{expected} directional product over random $u$ is proportional to $\bar x^\top\bar p$, and normalizing by $\|\bar x\|_2\|\bar p\|_2$ yields the direction-agnostic centered CC-Probe $c_{\text{centered}}(x)$.

\noindent
Finally, we emphasize that neither $c_{\text{centered}}(x)$ nor its modality-specific
instantiations are \emph{equal} to $\rho_c(u)$ for a fixed $u$; rather, they are practical
single-sample surrogates for the coupling (covariance) channel that controls the RS bound.

\subsection{Testable predictions}
\label{subsec:nur_regimes}

The NUP motivates a simple idea: \emph{abnormal input--gradient coupling} is a reliable marker of failure modes. We test this idea with the CC-Probe.

\begin{proposition}[Exp.~1 prediction: a ``high-cosine tail'' marks hard/fragile vision samples]
\label{prop:exp1_sep}
As training improves clean accuracy, correctly classified images should show \emph{lower} CC-Probe $c_{\text{img}}$, while misclassified/hard images remain at \emph{higher} $c_{\text{img}}$ (a persistent high-cosine tail). Adversarial errors should concentrate on this high-cosine tail.
\end{proposition}

\begin{proposition}[Exp.~2 prediction: $\pm$FGSM changes CC-Probe in the expected direction]
\label{prop:exp2_causal}
For a fixed vision model, a small gradient-aligned perturbation (+FGSM) should \emph{increase} $c_{\text{img}}$ and reduce accuracy, while a sufficiently small anti-aligned perturbation (--FGSM) should \emph{decrease} $c_{\text{img}}$ and can preserve (or slightly improve) accuracy.
\end{proposition}

\begin{proposition}[Exp.~3--4 prediction: training that suppresses strong $x\!\cdot\!p$ coupling becomes more robust]
\label{prop:exp34_mask}
If vulnerability is driven by a few input components with large coupling scores
(implemented in practice as the channel-wise normalized interaction
$|\tilde x_{c,j}\tilde p_{c,j}|$),
then masking these dominant components during training (ConjMask) should improve robustness to standard gradient attacks (e.g., PGD/APGD-CE) without adversarial training (Exp.~3). If this robustness is loss-dependent, adding a logit stabilizer (LogitReg) should restore robustness under stronger loss-optimized attacks (e.g., APGD-DLR) (Exp.~4).
\end{proposition}

\begin{proposition}[Exp.~5--6 prediction: in LLM prefill, low CC-Probe means under-conditioning and higher hallucination risk]
\label{prop:exp56_llm}
Using only prefill (no decoding), prompts with unusually \emph{low} $c_{\text{prompt}}$ should have higher hallucination risk; therefore $-c_{\text{prompt}}$ should predict hallucination above chance (Exp.~5). Among paraphrased prompts, choosing the one with \emph{higher} $c_{\text{prompt}}$ should more often select the judge-preferred prompt (Exp.~6).
\end{proposition}

\section{Functional-Analytic Foundations and Proof Roadmap}
\label{supp:foundations}

To avoid duplicating the same operator-theoretic material under two different notational schemes, we place the full functional-analytic setup, the Robertson--Schr\"odinger proof, and the drift/variance/covariance bridge derivations in Section~\ref{app:supp_theory}. The remainder of this supplement focuses on representation dependence, the quantum--neural analogy, and detailed experimental protocols.

% =========================================================================
% SECTION S.3: REPRESENTATION (THE EXPANDED SECTION)
% =========================================================================
\section{Representation Dependence and Coordinate Invariance}
\label{supp:representation}

A fundamental question arises regarding the universality of the CC-Probe: \emph{If we change the representation of the input (e.g., from pixels to a latent space), how does the probe value change?} This section uses a differential geometric perspective to clarify the coordinate dependence of the proposed uncertainty relations.

\subsection{Manifolds, Coordinates, and the Physical Analogy}
In physics, a quantum state $|\psi\rangle$ is an abstract vector in Hilbert space, invariant under the choice of basis. However, the \emph{wavefunction} $\psi(x) = \langle x | \psi \rangle$ and the uncertainty product $\Delta x \Delta p$ depend explicitly on the choice of the observable $\hat x$ (the coordinate basis).

Similarly, in deep learning, a sample possesses an abstract semantic content. However, the neural network processes a specific \textbf{representation} $x \in \mathbb R^d$ (pixels, tokens, embeddings). The Neural Uncertainty Principle (NUP) defines the conjugate constraint \emph{relative to the geometry of this representation}.

\subsection{Coordinate Transformation Laws}
Consider a differentiable, invertible change of coordinates $z = \Phi(x)$, where $z$ represents a feature layer or an alternative parameterization. Let $J(x) = \frac{\partial \Phi}{\partial x}$ be the Jacobian matrix.

\paragraph{Transformation of Vectors (Contravariant)}
A small displacement vector $\delta x$ transforms via the Jacobian:
\begin{equation}
    \delta z \approx J(x) \, \delta x.
\end{equation}
Geometrically, $\delta x$ belongs to the tangent space $T_x \mathcal X$.

\paragraph{Transformation of Gradients (Covariant)}
The gradient $p(x) = \nabla_x \mathcal L$ is a one-form (covector) belonging to the cotangent space $T^*_x \mathcal X$. By the chain rule, it transforms via the inverse transpose of the Jacobian:
\begin{equation}
    p_x(x) = J(x)^\top p_z(z).
\end{equation}
Crucially, vectors and covectors transform inversely. This ensures that their duality pairing (the directional derivative) is scalar-invariant locally:
\begin{equation}
    \delta x^\top p_x(x) = \delta x^\top (J^\top p_z) = (J \delta x)^\top p_z = \delta z^\top p_z(z).
\end{equation}

\subsection{Why the CC-Probe varies with Representation}
The CC-Probe $c(x) = |\cos(\bar x, \bar p)|$ involves the \textbf{norms} of the vectors, not just their inner product. The Euclidean norm is not invariant under non-orthogonal transformations.
Let us analyze the cosine in the $z$-space (assuming local linearity for the centered variables $\bar z \approx J \bar x$):
\begin{align}
    \cos(\bar z, \bar p_z) &= \frac{\bar z^\top \bar p_z}{\|\bar z\|_2 \|\bar p_z\|_2} \nonumber \\
    &\approx \frac{(J \bar x)^\top (J^{-\top} \bar p_x)}{\|J \bar x\|_2 \|J^{-\top} \bar p_x\|_2} \nonumber \\
    &= \frac{\bar x^\top \bar p_x}{\sqrt{\bar x^\top J^\top J \bar x} \sqrt{\bar p_x^\top (J J^\top)^{-1} \bar p_x}}.
\end{align}
Here, $G = J^\top J$ acts as the local \textbf{metric tensor}. The "aligned" probe in $z$-space is effectively measuring the alignment in $x$-space, but distorted by the condition number of the Jacobian $J$.

\paragraph*{Interpretation}
This coordinate dependence is not a flaw; it is the \textbf{intended functionality}.
\begin{enumerate}
    \item \textbf{Layer-Specific Diagnostics:} We are interested in whether the \emph{current layer's} geometry allows for easy modification of the loss (high sensitivity) via small perturbations in the \emph{current layer's} coordinates.
    \item \textbf{Slack vs. Saturation:} A representation $z$ where $J$ is highly anisotropic (ill-conditioned) might stretch the input space such that the uncertainty bound is far from saturation (Slack regime), whereas the pixel space $x$ might be near saturation.
    \item \textbf{Conclusion:} When applying the CC-Probe, one must explicitly define the "conjugate pair" of interest. In our experiments, we study the \textbf{Pixel-Gradient} pair for vision (robustness) and the \textbf{Embedding-Gradient} pair for LLMs (hallucination), treating them as distinct physical systems defined by their respective coordinate geometries.
\end{enumerate}

% =========================================================================
% SECTION S.4: TABLE 1 DISCUSSION
% =========================================================================
\section{The Quantum--Neural Analogy}
\label{supp:analogy}

Table~\ref{tab:qm_analogy} provides a detailed side-by-side comparison of the structural isomorphism between the formulation of Quantum Mechanics (QM) and our proposed Neural Uncertainty Principle (NUP). In this section, we elaborate on the non-trivial aspects of this mathematical correspondence and clarify that no physical quantum effects are assumed.

\begin{table}[h]
\centering
\caption{Quantum physics vs.\ neural networks. Hats $\hat{\cdot}$ denote operators; no physical quantum effect is assumed. The structural isomorphism allows the translation of the Robertson--Schr\"odinger inequality to the neural domain.}
\label{tab:qm_analogy}
{%
\scriptsize
\setlength{\tabcolsep}{4pt} % 稍微增加列间距
\renewcommand{\arraystretch}{1.3} % 稍微增加行高以提高可读性
\setcellgapes{2pt}\makegapedcells

\noindent\makebox[\textwidth][c]{%
% 在单栏模式下，建议将宽度设置为 0.75\textwidth 或 0.8\textwidth 以防字体过大
\resizebox{\textwidth}{!}{%
\begin{tabular}{l|p{0.41\textwidth}|p{0.41\textwidth}}
\toprule
\textbf{Role} & \textbf{Quantum physics} & \textbf{Neural networks}\\
\midrule

Model
& \makecell[l]{$\psi,\ \hat H$}
& \makecell[l]{$f_\theta,\ \mathcal L_c(x)$}\\

State
& \makecell[l]{$\psi(x)$}
& \makecell[l]{$\psi_c(x)=A_c(x)e^{i\alpha\mathcal L_c(x)}$}\\

Density
& \makecell[l]{$|\psi(x)|^2$}
& \makecell[l]{$|\psi_c(x)|^2=A_c(x)^2$}\\

Conjugate pair
& \makecell[l]{$(\hat x_u g)(x)=(u^\top x)\,g(x)$\\
$(\hat p_u g)(x)=-i\,\partial_u g(x)$}
& \makecell[l]{$(\hat x_u g)(x)=(u^\top x)\,g(x)$\\
$(\hat p_u g)(x)=-i\,\partial_u g(x)$}\\

Commutator scale
& \makecell[l]{$[\hat x_u,\hat p_u]=i\mathbb I$\\
$\kappa=\tfrac12$}
& \makecell[l]{$[\hat x_u,\hat p_u]=i\mathbb I$ \\
$\kappa=\tfrac12\big|\langle[\hat x_u,\hat p_u]\rangle_c\big|$}\\

Correlation
& \makecell[l]{$\rho(u)=\dfrac{\mathrm{Cov}(\hat x_u,\hat p_u)}{\Delta\hat x(u)\,\Delta\hat p(u)}$}
& \makecell[l]{$\rho_c(u)=\dfrac{\mathrm{Cov}_c(\hat x_u,\hat p_u)}{\Delta\hat x_c(u)\,\Delta\hat p_c(u)}$}\\

Uncertainty relation
& \makecell[l]{$\Delta\hat x(u)\,\Delta\hat p(u)\,\sqrt{1-\rho(u)^2}$\\
$\ge\ \kappa$}
& \makecell[l]{$\Delta\hat x_c(u)\,\Delta\hat p_c(u)\,\sqrt{1-\rho_c(u)^2}$\\
$\ge\ \kappa$}\\

Neural sample probe
& \makecell[l]{(not applicable)}
& \makecell[l]{$\bar x=x-\mu_c,\ \bar p=p(x)-\nu_c$\\
$|\cos(\bar x,\bar p)|=\dfrac{|\bar x^\top\bar p|}{\|\bar x\|_2\,\|\bar p\|_2}$}\\

\bottomrule
\end{tabular}%
}%
}%
} % end local group
\end{table}

\paragraph{Mathematical tool, not physical hypothesis}
We emphasize that the analogy in Table~\ref{tab:qm_analogy} is purely mathematical. We do not assume any physical quantum effect in neural networks; rather, we use the Hilbert-space formalism as a rigorous signal-processing language for the loss manifold.

\paragraph{Engineered phase and sensitivity extraction}
In quantum mechanics, the phase $S(x)$ of a wavefunction $\psi = R e^{iS/\hbar}$ generates the momentum field $\nabla S$. In our framework, we engineer the phase to be $\alpha \mathcal L_c(x)$, so that the operator $\hat p_u=-i\partial_u$ exposes the loss-gradient signal $\nabla \mathcal L_c$ as a real sensitivity channel. This is the key mathematical device behind the quantum--neural analogy.

\paragraph{Conjugate incompatibility}
The NUP implies a fundamental incompatibility between localizing input features and localizing gradient information. A state with very small $\Delta \hat x$ must have large $\Delta \hat p$, while a state with very small $\Delta \hat p$ must be spread out in $\Delta \hat x$. The CC-Probe measures how this trade-off manifests in trained networks and helps distinguish near-bound (saturation) from high-slack regimes.

% ============================================================
% Supplementary Material (Appendix):
% Operator foundations + computable bridges for NUP / CC-Probe
% ============================================================

\section{Operator foundations and computable bridges}
\label{app:supp_theory}

Here we provide a self-contained derivation of the operator-theoretic statements used in the main text, with particular emphasis on \emph{computability}: how operator quantities reduce to loss-weighted statistics of the real input-gradient $p(x)=\nabla_x\mathcal L_c(x)$.
We proceed in three steps:
(i) recall the general Robertson--Schr\"odinger (RS) inequality and its cosine/correlation form,
(ii) specialise to the neural directional pair $(\hat x_u,\hat p_u)$,
and (iii) derive exact bridges (drift / variance / covariance reductions) that motivate empirical proxies such as the CC-Probe.
Throughout, hats $\hat{\cdot}$ denote operators and $\mathbb I$ denotes the identity operator.

\subsection{Functional-analytic setup and standing assumptions}
\label{app:fa_setup}

\paragraph{State space}
Let $\mathcal X\subseteq\mathbb R^d$ denote the continuous input space in which the coordinate $x$ lives (e.g., a bounded subset of $\mathbb R^d$ or all of $\mathbb R^d$).
We work in the complex Hilbert space $\mathcal H:=L^2(\mathcal X)$ with inner product
\[
\langle f,g\rangle := \int_{\mathcal X} f(x)^\ast g(x)\,dx,
\qquad
\|f\|:=\sqrt{\langle f,f\rangle}.
\]

\paragraph{Domains and boundary conditions}
All differential operators are understood on a common dense domain $\mathcal D\subset\mathcal H$ such that integration by parts is valid and boundary terms vanish.
Typical choices include $\mathcal D=C_c^\infty(\mathcal X)$ (compact support), periodic boundary conditions on a torus, or sufficiently fast decay at infinity when $\mathcal X=\mathbb R^d$.
Whenever we invoke self-adjointness (or essential self-adjointness) of $-i\partial_u$, it is with respect to such a domain/self-adjoint extension.

\paragraph{Commutator expectation on bounded domains}
On bounded $\mathcal X$, some self-adjoint extensions of $-i\partial_u$ may introduce boundary terms; in that case the \emph{formal} commutator $[\hat x_u,\hat p_u]=i\mathbb I$ still holds on $\mathcal D$, but the \emph{state expectation} $\langle [\hat x_u,\hat p_u]\rangle$ may deviate from $1$.
Our cosine-form uncertainty relation remains valid with the commutator scale $\kappa:=\tfrac12|\langle[\hat x_u,\hat p_u]\rangle|$; see \S\ref{app:cosine_rs}.

\subsection{Variance of observables and covariance notation}
\label{app:operators}

Let $\hat A$ be a self-adjoint operator on $\mathcal H$ with dense domain containing $\psi$.
For a normalized state $\psi\in\mathcal H$ with $\|\psi\|=1$, define the expectation and centered operator
\[
\langle \hat A\rangle_\psi := \langle \psi,\hat A\psi\rangle,
\qquad
\Delta_\psi \hat A := \hat A - \langle\hat A\rangle_\psi\,\mathbb I,
\]
and the variance/dispersion
\[
(\Delta_\psi \hat A)^2 := \langle \psi,(\Delta_\psi\hat A)^2\psi\rangle
= \|\Delta_\psi \hat A\,\psi\|^2.
\]
For two observables $\hat A,\hat B$, define commutator and anticommutator
\[
[\hat A,\hat B]:=\hat A\hat B-\hat B\hat A,
\qquad
\{\hat A,\hat B\}:=\hat A\hat B+\hat B\hat A,
\]
and the symmetrized (operator) covariance
\begin{equation}
\label{eq:cov_def_general_app}
\mathrm{Cov}_\psi(\hat A,\hat B)
:=\frac12\,\langle \psi,\{\Delta_\psi\hat A,\Delta_\psi\hat B\}\psi\rangle.
\end{equation}

\subsection{Robertson--Schr\"odinger uncertainty inequality}
\label{app:rs_derivation}

\begin{theorem}[Robertson--Schr\"odinger uncertainty inequality]
\label{thm:rs_general_app}
Let $\hat A$ and $\hat B$ be self-adjoint operators on $\mathcal H$ and $\psi\in\mathcal H$ a normalized state with finite dispersions
$0<\Delta_\psi \hat A<\infty$ and $0<\Delta_\psi \hat B<\infty$.
Then
\begin{equation}
\label{eq:rs_general_app}
(\Delta_\psi \hat A)^2 (\Delta_\psi \hat B)^2
\;\ge\;
\frac{1}{4}\,\big|\langle\psi,[\hat A,\hat B]\psi\rangle\big|^2
\;+\;
\mathrm{Cov}_\psi(\hat A,\hat B)^2.
\end{equation}
\end{theorem}

\begin{proof}
Let $\xi:=\Delta_\psi\hat A\,\psi$ and $\eta:=\Delta_\psi\hat B\,\psi$.
Then $\|\xi\|^2=(\Delta_\psi\hat A)^2$ and $\|\eta\|^2=(\Delta_\psi\hat B)^2$.
By Cauchy--Schwarz,
\[
\|\xi\|^2\|\eta\|^2 \ge |\langle \xi,\eta\rangle|^2
= \big|\langle \psi,\Delta_\psi\hat A\,\Delta_\psi\hat B\,\psi\rangle\big|^2.
\]
Decompose $\Delta_\psi\hat A\,\Delta_\psi\hat B$ into symmetric and antisymmetric parts:
\[
\Delta_\psi\hat A\,\Delta_\psi\hat B
=
\frac12[\Delta_\psi\hat A,\Delta_\psi\hat B]
+\frac12\{\Delta_\psi\hat A,\Delta_\psi\hat B\}.
\]
Since constants commute with operators, $[\Delta_\psi\hat A,\Delta_\psi\hat B]=[\hat A,\hat B]$.
Define
\begin{equation}
\begin{aligned}
C:=&\frac12\langle\psi,\{\Delta_\psi\hat A,\Delta_\psi\hat B\}\psi\rangle
=\mathrm{Cov}_\psi(\hat A,\hat B),
\\
iD:=&\frac12\langle\psi,[\hat A,\hat B]\psi\rangle,
\end{aligned}
\end{equation}
where $C\in\mathbb R$ (self-adjoint) and $D\in\mathbb R$ (skew-adjoint gives a purely imaginary expectation).
Then $\langle \xi,\eta\rangle=C+iD$, hence
\[
|\langle \xi,\eta\rangle|^2 = C^2 + D^2
= \mathrm{Cov}_\psi(\hat A,\hat B)^2
+\frac14\big|\langle\psi,[\hat A,\hat B]\psi\rangle\big|^2.
\]
Combining with Cauchy--Schwarz gives \eqref{eq:rs_general_app}.
\end{proof}

By discarding the covariance term in \eqref{eq:rs_general_app} one obtains the familiar Heisenberg-type bound
\begin{equation}
\label{eq:heisenberg_general_app}
(\Delta_\psi \hat A)^2(\Delta_\psi \hat B)^2
\;\ge\;
\frac14\big|\langle\psi,[\hat A,\hat B]\psi\rangle\big|^2.
\end{equation}
In \S\ref{app:nup_specialization_supp} we instantiate this with $(\hat A,\hat B)=(\hat x_u,\hat p_u)$ and $\psi=\psi_c$.

\subsection{Cosine/correlation form}
\label{app:cosine_rs}

Define the operator correlation coefficient (when dispersions are nonzero)
\begin{equation}
\label{eq:rho_general_def_app}
\rho_\psi(\hat A,\hat B)
:=
\frac{\mathrm{Cov}_\psi(\hat A,\hat B)}{\Delta_\psi\hat A\,\Delta_\psi\hat B}
\in[-1,1].
\end{equation}
Substituting $\mathrm{Cov}_\psi(\hat A,\hat B)=\rho_\psi(\hat A,\hat B)\Delta_\psi\hat A\,\Delta_\psi\hat B$
into \eqref{eq:rs_general_app} yields the equivalent cosine form
\begin{equation}
\label{eq:rs_cosine_form_app}
(\Delta_\psi\hat A)^2(\Delta_\psi\hat B)^2\bigl(1-\rho_\psi(\hat A,\hat B)^2\bigr)
\;\ge\;
\frac14\,\big|\langle\psi,[\hat A,\hat B]\psi\rangle\big|^2.
\end{equation}
Geometrically, letting $\xi:=\Delta_\psi\hat A\,\psi$ and $\eta:=\Delta_\psi\hat B\,\psi$,
\[
\rho_\psi(\hat A,\hat B)=\frac{\mathrm{Re}\langle\xi,\eta\rangle}{\|\xi\|\,\|\eta\|},
\]
i.e., the cosine of the angle between $\xi$ and $\eta$ in $\mathcal H$ (taking real part).

\subsection{Specialisation to the neural directional pair}
\label{app:nup_specialization_supp}

\paragraph{Neural loss and input-gradient}
Fix a condition/target index $c$ and a scalar loss $\mathcal L_c(x)$ (see main text).
Define the input-gradient field
\[
p(x):=\nabla_x\mathcal L_c(x)\in\mathbb R^d.
\]
For a unit direction $u\in\mathbb S^{d-1}$ define the projections
\begin{equation}
\label{eq:direction_defs_app}
\begin{aligned}
x_u:=&u^\top x,
\\
p_u(x):=&u^\top p(x)=\partial_u\mathcal L_c(x),
\\
\partial_u:=&u^\top\nabla_x.
\end{aligned}
\end{equation}

\paragraph{Loss-phase state and loss-induced expectation}
Define the loss-phase state
\begin{equation}
\label{eq:loss_phase_state_app}
\begin{aligned}
\psi_c(x):=&A_c(x)\exp\!\big(i\alpha\,\mathcal L_c(x)\big),
\\
A_c(x):=&\frac{|\mathcal L_c(x)|}{\sqrt{\beta_c}},
\\
\beta_c:=&\int_{\mathcal X}\mathcal L_c(x)^2\,dx,
\end{aligned}
\end{equation}
so that $\|\psi_c\|_2=1$ and $|\psi_c(x)|^2=A_c(x)^2$.
For scalar functions $g(x)$ define the loss-induced expectation
\[
\mathbb E_c[g]:=\int_{\mathcal X} g(x)\,|\psi_c(x)|^2\,dx,
\]
and for operators $\hat O$ define $\langle\hat O\rangle_c:=\langle\psi_c,\hat O\psi_c\rangle$.

\noindent\textit{Remark.}
The phase $\exp(i\alpha\mathcal L_c)$ is the simplest choice whose directional derivative exposes the real gradient projection $p_u=\partial_u\mathcal L_c$ through Lemma~\ref{lem:drift_identity}.
Moreover, choosing $A_c\propto|\mathcal L_c|$ makes $|\psi_c|^2\propto \mathcal L_c^2$, emphasizing high-loss (boundary-relevant) regions.
Other monotone loss-weightings and phase choices are possible, but the linear phase keeps the drift coefficient constant and yields the clean covariance reduction in Lemma~\ref{lem:op_cov_reduction_app}.

\paragraph{Directional operators}
Define the canonical directional pair on $\mathcal D$:
\begin{equation}
\label{eq:dir_ops_app}
(\hat x_u g)(x):=(u^\top x)\,g(x),
\qquad
(\hat p_u g)(x):=-i\,\partial_u g(x).
\end{equation}
Under the standing assumptions of \S\ref{app:fa_setup}, the formal commutator holds on $\mathcal D$:
\begin{equation}
\label{eq:commutator_app}
[\hat x_u,\hat p_u]=i\mathbb I.
\end{equation}

Indeed, for any $g\in\mathcal D$,
\begin{equation}
\begin{aligned}
([\hat x_u,\hat p_u]g)(x)
=&(u^\top x)(-i\partial_u g)+i\partial_u\big((u^\top x)g\big)\\
=&i(\partial_u(u^\top x))g\\
=&i g,
\end{aligned}
\end{equation}
hence $[\hat x_u,\hat p_u]=i\mathbb I$ on $\mathcal D$.

\paragraph{Neural dispersions and covariance}
Define centered operators $\Delta\hat x_u:=\hat x_u-\langle\hat x_u\rangle_c\mathbb I$ and
$\Delta\hat p_u:=\hat p_u-\langle\hat p_u\rangle_c\mathbb I$ and the dispersions
\begin{equation}
\label{eq:disp_defs_app}
\begin{aligned}
(\Delta\hat x_c(u))^2 :=& \|\Delta\hat x_u\,\psi_c\|^2
=\mathbb E_c\big[(x_u-\mathbb E_c[x_u])^2\big],
\\
(\Delta\hat p_c(u))^2 :=& \|\Delta\hat p_u\,\psi_c\|^2.
\end{aligned}
\end{equation}
The symmetrized operator covariance is
\begin{equation}
\label{eq:cov_op_app}
\mathrm{Cov}_c(\hat x_u,\hat p_u)
:=\frac12\langle\psi_c,\{\Delta\hat x_u,\Delta\hat p_u\}\psi_c\rangle.
\end{equation}

\paragraph{Neural RS uncertainty relation (operator form)}
Applying Theorem~\ref{thm:rs_general_app} with $(\hat A,\hat B)=(\hat x_u,\hat p_u)$ and $\psi=\psi_c$ gives
\begin{equation}
\label{eq:nur_rs_app}
(\Delta\hat x_c(u))^2(\Delta\hat p_c(u))^2
\;\ge\;
\frac14\big|\langle\psi_c,[\hat x_u,\hat p_u]\psi_c\rangle\big|^2
+\mathrm{Cov}_c(\hat x_u,\hat p_u)^2,
\end{equation}
and the cosine form
\begin{equation}
\label{eq:nur_cosine_app}
\begin{aligned}
(\Delta\hat x_c(u))^2(\Delta\hat p_c(u))^2\bigl(1-\rho_c(u)^2\bigr)
\;\ge\;
\frac14\big|\langle[\hat x_u,\hat p_u]\rangle_c\big|^2,
\\
\rho_c(u):=\frac{\mathrm{Cov}_c(\hat x_u,\hat p_u)}{\Delta\hat x_c(u)\,\Delta\hat p_c(u)}.
\end{aligned}
\end{equation}
If $\langle[\hat x_u,\hat p_u]\rangle_c=1$ then the canonical bound follows immediately.

\subsection{Making the operator statistics computable}
\label{app:bridges}

This subsection derives exact identities that connect $(\Delta\hat p_c(u))^2$ and $\mathrm{Cov}_c(\hat x_u,\hat p_u)$
to loss-weighted statistics of the \emph{real} projected gradient $p_u(x)=\partial_u\mathcal L_c(x)$.
These bridges justify using real-valued alignment observables as proxies.

\subsubsection{Loss-phase drift identity}
\label{app:drift_identity}

\begin{lemma}[Loss-phase drift identity]
\label{lem:drift_identity}
Let $\psi_c(x)=A_c(x)e^{i\alpha\mathcal L_c(x)}$ and $\hat p_u=-i\partial_u$ as in \eqref{eq:dir_ops_app}.
Then
\begin{equation}
\label{eq:pu_on_psi_app}
\hat p_u\psi_c
=-i\partial_u\psi_c
=\Big(\alpha A_c(x)\,p_u(x)-i\,\partial_u A_c(x)\Big)e^{i\alpha\mathcal L_c(x)}.
\end{equation}
\end{lemma}

\begin{proof}
Differentiate $\psi_c=A_c e^{i\alpha\mathcal L_c}$:
\[
\partial_u\psi_c=(\partial_u A_c)e^{i\alpha\mathcal L_c}+A_c\cdot i\alpha(\partial_u\mathcal L_c)e^{i\alpha\mathcal L_c}.
\]
Multiplying by $-i$ gives \eqref{eq:pu_on_psi_app} with $p_u=\partial_u\mathcal L_c$.
\end{proof}

\noindent
\textit{Interpretation.}
The choice $\hat p_u=-i\partial_u$ is convenient because it exposes a \emph{real} drift term $\alpha A_c p_u$ plus an orthogonal amplitude term $-i\,\partial_uA_c$.

\subsubsection{Variance decomposition for $\hat p_u$}
\label{app:variance_decomp}

\begin{proposition}[Variance decomposition and a conservative lower bound]
\label{prop:variance_decomp_app}
Under the standing assumptions of \S\ref{app:fa_setup}, the following identities hold:
\begin{align}
\label{eq:p_norm_app}
\|\hat p_u\psi_c\|^2
&=
\alpha^2\,\mathbb E_c[p_u(x)^2]+\mathbb E_c\!\big[(\partial_u\log A_c(x))^2\big],
\\
\label{eq:centered_p_var_app}
(\Delta\hat p_c(u))^2
&=
\alpha^2\,\mathrm{Var}_c(p_u)
+\mathrm{Var}_c\!\big(\partial_u\log A_c\big)
\;\ge\;
\alpha^2\,\mathrm{Var}_c(p_u),
\end{align}
where $\mathrm{Var}_c(g):=\mathbb E_c[(g-\mathbb E_c[g])^2]$ and $\mathbb E_c[\partial_u\log A_c]=0$ under vanishing boundary terms.
\end{proposition}

\begin{proof}
From Lemma~\ref{lem:drift_identity}, with $|\psi_c|^2=A_c^2$,
\begin{equation}
\begin{aligned}
\|\hat p_u\psi_c\|^2
=&\int_{\mathcal X}\Big(\alpha^2A_c^2p_u^2+(\partial_uA_c)^2\Big)\,dx\\
=&\alpha^2\mathbb E_c[p_u^2]+\mathbb E_c[(\partial_u\log A_c)^2].
\end{aligned}
\end{equation}
Centering yields
\begin{equation}
\begin{aligned}
(\Delta\hat p_c(u))^2=&\|\hat p_u\psi_c\|^2-|\langle\hat p_u\rangle_c|^2\\
=&\alpha^2\mathrm{Var}_c(p_u)+\mathrm{Var}_c(\partial_u\log A_c),
\end{aligned}
\end{equation}
where the second term is nonnegative.
Finally, under vanishing boundary terms,
\begin{equation}
\begin{aligned}
\mathbb E_c[\partial_u\log A_c]
=&\int A_c^2\,\partial_u\log A_c\,dx\\
=&\int A_c\,\partial_uA_c\,dx\\
=&\frac12\int \partial_u(A_c^2)\,dx\\
=&0.
\end{aligned}
\end{equation}
\end{proof}

\subsubsection{Operator-to-scalar covariance reduction}
\label{app:cov_reduction}

For scalar functions $g,h$, define the loss-weighted scalar covariance
\begin{equation}
\label{eq:sc_cov_def_app}
\mathrm{Cov}^{\mathrm{sc}}_c(g,h)
:=\mathbb E_c\!\big[(g-\mathbb E_c[g])(h-\mathbb E_c[h])\big],
\qquad
\mathrm{Var}_c(g):=\mathrm{Cov}^{\mathrm{sc}}_c(g,g).
\end{equation}

\begin{lemma}[Operator covariance reduces to scalar covariance]
\label{lem:op_cov_reduction_app}
Let $\psi_c(x)=A_c(x)e^{i\alpha\mathcal L_c(x)}$ with real $A_c(x)\ge0$, and let $\hat x_u,\hat p_u$ be defined by \eqref{eq:dir_ops_app} on a common dense domain $\mathcal D$ such that integration by parts is valid.
Then
\begin{equation}
\label{eq:op_cov_equals_scalar_cov_app}
\mathrm{Cov}_c(\hat x_u,\hat p_u)
=\alpha\,\mathrm{Cov}^{\mathrm{sc}}_c(x_u,p_u).
\end{equation}
\end{lemma}

\begin{proof}
Let $\Delta x_u:=x_u-\mathbb E_c[x_u]$ and recall $\Delta\hat x_u=\hat x_u-\langle\hat x_u\rangle_c\mathbb I$, $\Delta\hat p_u=\hat p_u-\langle\hat p_u\rangle_c\mathbb I$.
Using Lemma~\ref{lem:drift_identity},
\[
\hat p_u\psi_c=\Big(\alpha A_c\,p_u-i\,\partial_uA_c\Big)e^{i\alpha\mathcal L_c}.
\]
A direct substitution gives
\[
\langle \Delta\hat x_u\psi_c,\Delta\hat p_u\psi_c\rangle
=
\alpha\,\mathbb E_c[\Delta x_u\,p_u]
-i\,\mathbb E_c[\Delta x_u\,\partial_u\log A_c].
\]
By definition \eqref{eq:cov_op_app}, the symmetrized covariance is the real part:
\begin{equation}
\begin{aligned}
\mathrm{Cov}_c(\hat x_u,\hat p_u)
=&\mathrm{Re}\langle \Delta\hat x_u\psi_c,\Delta\hat p_u\psi_c\rangle\\
=&\alpha\,\mathbb E_c[\Delta x_u\,p_u]\\
=&\alpha\,\mathrm{Cov}^{\mathrm{sc}}_c(x_u,p_u).
\end{aligned}
\end{equation}
\end{proof}

\subsubsection{Attenuation bound for the correlation factor}
\label{app:attenuation}

Define the loss-weighted scalar correlation (when denominators are nonzero)
\[
\mathrm{Corr}^{\mathrm{sc}}_c(g,h)
:=
\frac{\mathrm{Cov}^{\mathrm{sc}}_c(g,h)}{\sqrt{\mathrm{Var}_c(g)\,\mathrm{Var}_c(h)}}.
\]

\begin{corollary}[Attenuation bound]
\label{cor:attenuation_app}
Combining Lemma~\ref{lem:op_cov_reduction_app} with Proposition~\ref{prop:variance_decomp_app} yields
\begin{equation}
\label{eq:rho_attenuation_app}
\begin{aligned}
\rho_c(u)
=&
\frac{\mathrm{Corr}^{\mathrm{sc}}_c(x_u,p_u)}
{\sqrt{1+\mathrm{Var}_c(\partial_u\log A_c)/(\alpha^2\mathrm{Var}_c(p_u))}},
\\
|\rho_c(u)|\le &\big|\mathrm{Corr}^{\mathrm{sc}}_c(x_u,p_u)\big|.
\end{aligned}
\end{equation}
\end{corollary}

\noindent
\textit{Remark.}
The phase choice in $\psi_c=A_c e^{i\alpha\mathcal L_c}$ does \emph{not} make $\hat p_u=-i\partial_u$ self-adjoint.
Self-adjointness is ensured by the domain/boundary assumptions (\S\ref{app:fa_setup}).
The role of the phase is instead computational: it ensures $\hat p_u\psi_c$ decomposes into a real drift term proportional to $p_u(x)$ and an orthogonal amplitude term, enabling the exact reductions above.

\subsection{From operator correlation to a per-sample CC-Probe}
\label{app:cc_probe_link}

The cosine factor in \eqref{eq:nur_cosine_app} depends on $\rho_c(u)$, hence on $\mathrm{Cov}_c(\hat x_u,\hat p_u)$.
By Lemma~\ref{lem:op_cov_reduction_app}, this operator covariance is exactly $\alpha$ times a loss-weighted scalar covariance, suggesting a computable alignment proxy.

\paragraph{Centered vectors and random-direction identity}
Define loss-weighted means (componentwise)
\[
\mu_c:=\mathbb E_c[x]\in\mathbb R^d,
\qquad
\nu_c:=\mathbb E_c[p(x)]\in\mathbb R^d,
\].

Let $u\sim\mathrm{Unif}(\mathbb S^{d-1})$ and define $X_u:=u^\top x$, $P_u:=u^\top p$.
Using isotropy $\mathbb E_u[uu^\top]=\tfrac1d\mathbb I$,
\begin{equation}
\label{eq:dir_avg_identity_app}
\begin{aligned}
\mathbb E_u[X_uP_u]=&\frac1d\,x^\top p,
\\
\mathbb E_u[X_u^2]=&\frac1d\|x\|_2^2,
\\
\mathbb E_u[P_u^2]=&\frac1d\|p\|_2^2.
\end{aligned}
\end{equation}
Since $\mathbb E_u[u]=0$, we have $\mathbb E_u[X_u]=\mathbb E_u[P_u]=0$, hence
\begin{equation}
\label{eq:cos_as_dir_corr_app}
\begin{aligned}
\mathrm{Corr}_u(X_u,P_u)
=&
\frac{\mathbb E_u[X_uP_u]}{\sqrt{\mathbb E_u[X_u^2]\mathbb E_u[P_u^2]}}\\
=&
\frac{ x^\top p}{\| x\|_2\,\| p\|_2},\\
\big|\cos( x,p)\big|=&\big|\mathrm{Corr}_u(X_u,P_u)\big|.
\end{aligned}
\end{equation}

\paragraph{Interpretation as a per-sample probe}
Equation \eqref{eq:cos_as_dir_corr_app} shows that $|\cos(x,p)|$ is the magnitude of the random-direction correlation between projected coordinate and projected sensitivity for that sample.
Aggregated under the loss-induced expectation $\mathbb E_c[\cdot]$, these centered couplings control $\mathrm{Cov}^{\mathrm{sc}}_c(x_u,p_u)$, and hence (via \eqref{eq:op_cov_equals_scalar_cov_app}) the operator covariance term in $\rho_c(u)$.
This motivates $|\cos(x,p)|$ as a per-sample CC-Probe.

\section{Detailed Experimental Protocols}
\label{app:protocal_detail}

We provide a comprehensive summary of the core experimental protocols for Exp.~1 through Exp.~6 in Table~\ref{tab:all_exp_protocols}, covering dataset versions, specific model architecture adaptations, optimizer choices, and key training hyperparameters.

\begin{table*}[t!]
\centering
\caption{\textbf{Experimental Protocols (Exp.~1--Exp.~6).} This table summarizes the protocols across diagnostic (Exp.~1--2), mitigation (Exp.~3--4), and LLM application (Exp.~5--6) experiments. Standard Random Crop and Horizontal Flip are used in vision training unless otherwise specified. \(k\)=kernel size, \(s\)=stride; ``Stem'' denotes the first convolutional layer. The \textbf{LR / WD} column reports learning rate and weight decay. Unless otherwise noted, experiments follow the released-script seeds (Exp.~2 uses the parser default seed \(=0\); Exps.~1,3--6 use seed \(=42\)). Exp.~3 and Exp.~4 are summarized to match the released \texttt{mask\_train.py} implementation; launcher artifacts should be interpreted with care.}
\label{tab:all_exp_protocols}
\scriptsize
\setlength{\tabcolsep}{2.6pt}
\renewcommand{\arraystretch}{1.28}
\resizebox{\textwidth}{!}{%
\begin{tabular}{l|l|l|p{0.22\textwidth}|l|c|c|p{0.23\textwidth}}
\toprule
\textbf{Exp.} & \textbf{Task} & \textbf{Dataset} & \textbf{Architecture \& Adaptations} & \textbf{Optimizer} & \textbf{\makecell{Epochs / \\ Batch Size}} & \textbf{\makecell{LR / \\ WD}} & \textbf{Key Configuration} \\
\midrule

\multirow{2}{*}{\textbf{1}} &
\multirow{2}{*}{\shortstack[l]{\textbf{Diagnostic:}\\Correlation\\Dynamics}} &
\makecell[l]{CIFAR-10 \\ ($32{\times}32$)} &
\textbf{ResNet-18, DenseNet-121:} Stem \(7{\times}7,s{=}2 \rightarrow 3{\times}3,s{=}1\). \newline
\textbf{EffNet-B0:} stem stride \(2 \rightarrow 1\), first MBConv dw-stride \(\rightarrow 1\). \newline
\textbf{ViT-Tiny:} patch size \(16{\times}16 \rightarrow 4{\times}4\). &
AdamW + CosineAnnealingLR &
\makecell{200 \\ (300 for EffNet-B0) \\ 512} &
\makecell{\(1 \times 10^{-3}\) \\ \(1 \times 10^{-4}\)} &
Tracks \(\bar{c}_{\text{img}}\) every 10 epochs on held-out data. Compares clean-correct vs.\ clean-incorrect samples using gradients w.r.t.\ normalized inputs. \\
\cmidrule{3-8}
& &
\makecell[l]{Tiny-ImageNet-200 \\ ($64{\times}64$)} &
\textbf{ResNet-50:} stem \(7{\times}7 \rightarrow 3{\times}3\), remove MaxPool. \newline
\textbf{DenseNet-121:} stem \(7{\times}7 \rightarrow 3{\times}3\), remove pool0. \newline
\textbf{EffNet-B4, Swin-Tiny:} ImageNet-pretrained; the released Exp.~1 notebooks instantiate the active dataloaders at native \(64{\times}64\) resolution (not an active \(224{\times}224\) pipeline). &
AdamW + CosineAnnealingLR &
\makecell{200 \\ 256} &
\makecell{\(1 \times 10^{-3}\) \\ \makecell[l]{\(1\times10^{-4}\) (most)\\ \(5\times10^{-2}\) (Swin)}} &
Validates cross-architecture and cross-dataset consistency. Same probe schedule (\(K{=}10\)) on the official validation split. \\
\midrule

\multirow{2}{*}{\textbf{2}} &
\multirow{2}{*}{\shortstack[l]{\textbf{Diagnostic:}\\Signed-FGSM\\Intervention (IFA)}} &
\makecell[l]{CIFAR-10 \\ ($32{\times}32$)} &
\textbf{ResNet-18:} modified stem, FC \(\rightarrow 10\). \newline
\textbf{ViT-Tiny:} \texttt{vit\_tiny\_patch4\_32}. \newline
\textbf{Vim-Tiny:} \texttt{vim\_tiny\_patch4\_32}. &
AdamW + CosineAnnealingLR &
\makecell{201 \\ 64} &
\makecell{\(1 \times 10^{-3}\) \\ \(1 \times 10^{-4}\)} &
FGSM attack with \(\epsilon_{\text{train}}=8/255\); final runs use no mixup. Evaluation sweeps signed FGSM magnitudes and reports accuracy plus cosine statistics by clean correctness. \\
\cmidrule{3-8}
& &
\makecell[l]{ImageNet-100 \\ ($224{\times}224$)} &
\textbf{ResNet-18:} modified stem, FC \(\rightarrow 100\). \newline
\textbf{ViT-Tiny:} \texttt{vit\_tiny\_patch16\_224}. \newline
\textbf{Vim-Tiny:} \textbf{Vim-Tiny (patch16, 224)}. &
AdamW + CosineAnnealingLR &
\makecell{201 \\ 128} &
\makecell{\(5 \times 10^{-4}\) \\ \(1 \times 10^{-4}\)} &
Same FGSM attack setting as CIFAR-10. Validation-time signed-FGSM sweeps evaluate intervention effects across perturbation directions and magnitudes. \\
\midrule

\textbf{3} &
\shortstack[l]{\textbf{Robustness:}\\Mask-Only\\Training} &
\makecell[l]{CIFAR-10 \\ ($32{\times}32$)} &
\textbf{ResNet-18:} modified stem, FC \(\rightarrow 10\). \newline
\textbf{ViT-Tiny:} \texttt{vit\_tiny\_patch4\_32}. \newline
\textbf{EfficientNet-B0:} stem stride \(2 \rightarrow 1\), first MBConv dw-stride \(\rightarrow 1\). &
AdamW + warmup cosine &
\makecell{201 \\ 64} &
\makecell[l]{\textbf{ResNet-18 / EffNet-B0:} \(1 \times 10^{-3}\) / \(1 \times 10^{-4}\) \\ \textbf{ViT-Tiny:} \(5 \times 10^{-4}\) / \(5 \times 10^{-2}\)} &
Initialized from clean checkpoints. Uses \texttt{mask\_pelnaty\_ablation} with \texttt{--no\_logits\_loss} and \texttt{--no\_random\_loss}, i.e., \(\mathcal{L}=\mathcal{L}_{\text{mask}}\) only. The released run commands use \texttt{ratio\_base}=0.25 for ResNet-18 / EfficientNet-B0 and \texttt{ratio\_base}=0.85 for ViT-Tiny. \\
\midrule

\textbf{4} &
\shortstack[l]{\textbf{Robustness:}\\Mask + LogitReg} &
\makecell[l]{CIFAR-10 \\ ($32{\times}32$)} &
Same architectures as Exp.~3: \textbf{ViT-Tiny}, \textbf{ResNet-18}, and \textbf{EfficientNet-B0}. &
AdamW + warmup cosine &
\makecell{201 \\ 64} &
\makecell[l]{\textbf{ResNet-18 / EffNet-B0:} \(1 \times 10^{-3}\) / \(1 \times 10^{-4}\) \\ \textbf{ViT-Tiny:} \(5 \times 10^{-4}\) / \(5 \times 10^{-2}\)} &
Initialized from clean checkpoints. Uses \texttt{mask\_pelnaty\_ablation} with both \texttt{logits\_loss} and \texttt{random\_loss}. All runs share the same masking operator as Exp.~3. Use the centered-logit variance penalty on clean and masked logits, together with KL consistency to a per-sample mean-baseline auxiliary view. The consistency branch is enabled after epoch \(5\). The released run commands use \texttt{ratio\_base}=0.25 for ResNet-18 / EfficientNet-B0 and \texttt{ratio\_base}=0.85 for ViT-Tiny. \\
\midrule

\textbf{5} &
\shortstack[l]{\textbf{LLM Risk:}\\Hallucination\\Detection} &
\makecell[l]{Benchmark-500 \\ (Math/Reasoning)} &
\textbf{Model:} DeepSeek-Coder-7B-Instruct-v1.5. \newline
\textbf{Judge:} 5-LLM panel (Claude, DeepSeek, Gemini, GPT, Grok). &
N/A & N/A & N/A &
Decoding-free prefill probe based on prompt embedding--gradient cosine. Evaluated with AUROC against hallucination labels; compared with entropy, NLL, and margin baselines. \\
\midrule

\textbf{6} &
\shortstack[l]{\textbf{LLM App:}\\Prompt\\Selection} &
\makecell[l]{Perturbation-100 \\ (5 Paraphrases)} &
\textbf{Model:} DeepSeek-Coder-7B-Instruct-v1.5. \newline
\textbf{Data:} 100 problems, each with 5 semantic variants. &
N/A & N/A & N/A &
Evaluates four prefill-only risk scores (Risk-Cos, Risk-Margin, Risk-Entropy, and Risk-Loss) as independent selection rules before decoding. Reports Top-1 hit rate, mean regret, Best-vs-Rest AUROC, and Spearman rank correlation. \\
\bottomrule
\end{tabular}
}
\end{table*}

\subsection{Exp.~1: CC-Probe Separates Correct and Incorrect Samples During Training}
\label{app:exp1_protocol}

\paragraph{Objective and hypothesis}
Exp.~1 quantifies how the conjugate correlation
\(
\bar c_\text{img}
\)
evolves during standard training, and whether it separates correctly classified vs.\ misclassified samples.
Our working hypothesis is: as training improves generalization, correctly classified samples concentrate at low coupling
(\(
\bar c_\text{img}
\to 0\)), while misclassified samples exhibit a persistent high-correlation tail.

\paragraph{Datasets and evaluation split}
We use CIFAR-10 (\(32\times 32\)) with the standard train/test split.
We use Tiny-ImageNet-200 (\(64\times 64\), 200 classes) with the official training set and the official validation set.
For Tiny-ImageNet-200, validation labels are constructed by parsing \texttt{val\_annotations.txt} to build deterministic
\((\text{image}, y)\) pairs.

\paragraph{Input preprocessing and augmentation}
The probe uses the same tensor \(x\) that is fed to the model. In particular, \(\nabla_x\mathcal{L}\) is taken
\emph{w.r.t. the normalized input tensor} (after mean/std normalization).
\begin{itemize}[leftmargin=*]
    \item \textbf{CIFAR-10:} training uses RandomCrop(32, padding=4) + RandomHorizontalFlip; testing uses only normalization.
    \item \textbf{Tiny-ImageNet-200:} in the released Exp.~1 notebooks, the active dataloaders are instantiated at native \(64\times64\) resolution and then normalized. ResNet-50 and DenseNet-121 use native-resolution dataloaders throughout.
    \item \textbf{Implementation note for EfficientNet-B4 / Swin-Tiny:} the released notebooks also contain a later alternative \(224\times224\) transform block for experimentation, but the datasets/loaders used by the reported runs are created earlier and are not rebuilt afterward. Therefore, the reported Exp.~1 results for EfficientNet-B4 and Swin-Tiny use native \(64\times64\) dataloader inputs, while still initializing the models from ImageNet-pretrained weights.
\end{itemize}

\paragraph{Architectures and adaptations}
We evaluate representative CNN/Transformer backbones.
\begin{itemize}[leftmargin=*]
    \item \textbf{CIFAR-10:}
    ResNet-18: stem conv \(7{\times}7, s{=}2 \to 3{\times}3, s{=}1\).
    DenseNet-121: \texttt{features.conv0} \(7{\times}7, s{=}2 \to 3{\times}3, s{=}1\).
    EfficientNet-B0: \texttt{conv\_stem} stride \(2\to 1\) and first MBConv depthwise stride \(2\to 1\).
    ViT-Tiny: patch size \(16{\times}16 \to 4{\times}4\) with \texttt{img\_size}=32.
    \item \textbf{Tiny-ImageNet-200:}
    ResNet-50: stem \(7{\times}7 \to 3{\times}3\) and remove maxpool.
    DenseNet-121: stem \(7{\times}7 \to 3{\times}3\) and remove \texttt{pool0}.
    EfficientNet-B4 and Swin-Tiny are initialized with ImageNet-pretrained weights; other backbones are trained from scratch.
\end{itemize}

\paragraph{Training hyperparameters and reproducibility}
All experiments use AdamW and a cosine annealing learning-rate schedule.
We fix random seed = 42 and enable deterministic CuDNN for reproducibility.
\begin{itemize}[leftmargin=*]
    \item \textbf{CIFAR-10:} epochs = 200 (EffNet-B0: 300), batch size = 512, learning rate \(=10^{-3}\),
    weight decay \(=10^{-4}\).
    \item \textbf{Tiny-ImageNet-200:} epochs = 200, batch size = 256, learning rate \(=10^{-3}\).
    Weight decay \(=10^{-4}\) for ResNet-50 / DenseNet-121 / EfficientNet-B4, and \(=5\times 10^{-2}\) for Swin-Tiny.
\end{itemize}

\paragraph{Probe definition (per-sample absolute cosine)}
Given a held-out mini-batch \((x,y)\), we compute the loss \(\mathcal{L}(f_\theta(x),y)\) and obtain the input gradient
\(g=\nabla_x\mathcal{L}\).
Flattening each sample to vectors in \(\mathbb{R}^d\), we compute:
\begin{equation}
\label{eq:exp1_probe_def}
c_i
=
\left|
\frac{\langle x_i, g_i\rangle}{\|x_i\|_2\|g_i\|_2 + \varepsilon}
\right|,
\qquad
\varepsilon=10^{-8}.
\end{equation}

\paragraph{Evaluation schedule and reported statistics}
We run the probe every \(K=10\) epochs (epochs 10, 20, \dots) on the held-out set.
At each checkpoint, we report:
(i) Top-1 accuracy,
(ii) \(\bar c_\text{img}=\mathbb{E}[c_i]\) for correctly classified samples,
(iii) \(\bar c_\text{img}\) for incorrectly classified samples.
These curves visualize how the ``stress tail'' (high \(c\)) concentrates on misclassified samples as training proceeds.

\begin{algorithm}[t]
\caption{Training-Time Conjugate Correlation Probe (Exp.~1)}
\label{alg:exp1_probe}
\small
\begin{algorithmic}[1]
\REQUIRE Model \(f_\theta\), train loader \(\mathcal{D}_{\mathrm{train}}\), held-out loader \(\mathcal{D}_{\mathrm{hold}}\),
AdamW optimizer, cosine scheduler, probe interval \(K=10\), \(\varepsilon=10^{-8}\).
\FOR{epoch \(=1\) to \(E\)}
    \STATE \textbf{Train:} one epoch on \(\mathcal{D}_{\mathrm{train}}\) using cross-entropy loss.
    \STATE Step cosine scheduler.
    \IF{epoch mod \(K = 0\)}
        \STATE Set \(f_\theta\) to eval; initialize sums/counters for corr/incorr.
        \FOR{mini-batch \((x,y)\) in \(\mathcal{D}_{\mathrm{hold}}\)}
            \STATE \(\texttt{x.requires\_grad}\leftarrow\texttt{True}\) (clone to avoid side-effects).
            \STATE Forward: \(z=f_\theta(x)\), \(\hat{y}=\arg\max(z)\).
            \STATE Loss: \(\ell=\mathrm{CE}(z,y)\); backprop to get \(g=\nabla_x\ell\).
            \STATE Flatten: \(x\leftarrow\mathrm{view}(x,B,-1)\), \(g\leftarrow\mathrm{view}(g,B,-1)\).
            \STATE Normalize: \(\tilde{x}=x/(\|x\|_2+\varepsilon)\), \(\tilde{g}=g/(\|g\|_2+\varepsilon)\).
            \STATE Compute \(c = \big|\sum_j \tilde{x}_j\odot\tilde{g}_j\big|\) (per sample).
            \STATE If \(\hat{y}=y\), accumulate into \(S_{\mathrm{corr}}\); else into \(S_{\mathrm{incorr}}\).
        \ENDFOR
        \STATE Output Top-1, \(\bar c_\text{img-corr}\), \(\bar c_\text{img-incorr}\).
    \ENDIF
\ENDFOR
\end{algorithmic}
\end{algorithm}

\paragraph{Analysis purpose}
Exp.~1 provides an empirical basis for using \(\bar c_\text{img}\) as a reliability observable:
(i) it confirms that high conjugate coupling concentrates on failure states (misclassification),
(ii) it motivates later causal manipulation (Exp.~2) and mitigation mechanisms (Exp.~3--4) by treating high \(\bar c_\text{img}\)
as a measurable ``boundary stress'' marker.

\subsection{Protocol for Exp.~2: Stepping Along/Against the Gradient Moves the CC-Probe as Predicted}
\label{app:exp2_details_revised}

\paragraph{Goal (causal test of NUP)}
Exp.~2 performs a controlled, inference-only intervention to test the directional prediction of the NUP:
moving an input \emph{toward} the decision boundary (loss ascent) should increase conjugate coupling and degrade reliability,
while moving \emph{away} (loss descent) should relieve coupling and preserve/improve reliability.
We quantify this by measuring both accuracy and the conjugate correlation observable on perturbed inputs.

\paragraph{Models and checkpoints}
We evaluate the dedicated Exp.~2 checkpoints released with the \texttt{exp2/Train} and \texttt{exp2/Test} code. The released training code contains a one-step FGSM-augmentation branch with \(\epsilon_{\text{train}}=8/255\), but the evaluation notebooks load the \texttt{withATmix00} checkpoints whose training configuration uses \texttt{at\_mixprob}=0. In the released \texttt{train\_one\_epoch\_with\_AT} implementation, adversarial samples are injected only when the attack branch is active and a Bernoulli draw satisfies the mixing probability; therefore, for the reported \texttt{withATmix00} checkpoints, no adversarial samples are actually mixed into training. Unless otherwise stated, we report results from the final checkpoint of each model. Accordingly, Exp.~2 should be understood as an inference-time signed-FGSM intervention on effectively clean-trained checkpoints, rather than as an adversarial-training study. Different architectures share the same signed-FGSM evaluation pipeline; only the model constructor and checkpoint path differ.

\paragraph{Datasets and evaluation split}
We run Exp.~2 on the standard test/validation split of CIFAR-10 ($32\times32$) and ImageNet-100 ($224\times224$).
All perturbations are applied to the \emph{actual input tensor} received by the model at inference time (i.e., the output of the dataloader).
We use a single $L_\infty$ step and ensure the perturbed input remains in the valid input range via clipping.

\paragraph{Bidirectional FGSM intervention}
For each clean input $x$ with label $y$, define the one-step FGSM direction
\[
d(x) = \mathrm{sign}\!\left(\nabla_x \mathcal{L}(f_\theta(x), y)\right).
\]
We generate a perturbed input by
\begin{equation}
\label{eq:exp2_bifgsm}
x_{\epsilon,\eta} = \mathrm{clip}\big(x + \eta \cdot \epsilon \cdot d(x)\big),\qquad \eta\in\{+1,-1\}.
\end{equation}
Here $\eta=+1$ is the standard adversarial step (loss ascent), while $\eta=-1$ is the anti-aligned ``relief'' step (loss descent).

\paragraph{Perturbation grids}
We evaluate a grid of $\epsilon$ values:
\begin{itemize}[nosep, leftmargin=*]
    \item \textbf{CIFAR-10 (+FGSM):} $\epsilon \in \{0.1,0.2,0.3,0.4,0.5\}/255$.
    \item \textbf{CIFAR-10 (-FGSM):} $\epsilon \in \{0.01,0.1,0.2,0.4,1.0\}/255$ (implemented as negative \texttt{eps}).
    \item \textbf{ImageNet-100 (+FGSM):} $\epsilon \in \{0.01,0.1,0.2,0.4,1.0\}/255$.
    \item \textbf{ImageNet-100 (-FGSM):} $\epsilon \in \{0.01,0.1,0.2,0.4,1.0\}/255$ (implemented as negative \texttt{eps}).
\end{itemize}
The two directional grids are not strictly symmetric: CIFAR-10 includes an additional larger relief step (\(1.0/255\)) and ImageNet-100 uses the same magnitude grid for both directions. Accordingly, we interpret Exp.~2 primarily as a directional probe, rather than as a perfectly paired equal-\(\epsilon\) comparison at every step size.

\paragraph{Probe signal on perturbed inputs}
For each perturbed input $x_{\epsilon,\eta}$, we \emph{recompute} the gradient at the perturbed point:
\[
p_{\epsilon,\eta} := \nabla_{x_{\epsilon,\eta}} \mathcal{L}(f_\theta(x_{\epsilon,\eta}), y).
\]
We flatten $x_{\epsilon,\eta}$ and $p_{\epsilon,\eta}$ into vectors and compute the cosine similarity
\begin{equation}
\label{eq:exp2_cos}
c_{\epsilon,\eta}(x) = \left|\cos\!\left(x_{\epsilon,\eta},\,p_{\epsilon,\eta}\right)\right|
= \frac{\left|x_{\epsilon,\eta}^\top p_{\epsilon,\eta}\right|}{\|x_{\epsilon,\eta}\|_2\,\|p_{\epsilon,\eta}\|_2}.
\end{equation}

\paragraph{Evaluation metrics}
For each $(\epsilon,\eta)$, we report:
\begin{itemize}[nosep, leftmargin=*]
    \item \textbf{Clean Top-1 accuracy} $\mathrm{Acc}_{\mathrm{clean}}$: model predictions on $x$.
    \item \textbf{FGSM Top-1 accuracy} $\mathrm{Acc}_{\mathrm{FGSM}}(\epsilon,\eta)$: predictions on $x_{\epsilon,\eta}$.
    \item \textbf{Expected coupling} $\bar c_{\epsilon,\eta}=\mathbb{E}[c_{\epsilon,\eta}(x)]$ over the evaluation set.
\end{itemize}

\begin{algorithm}[t]
\caption{Bidirectional FGSM Evaluation with Conjugate Correlation Probe (Exp.~2)}
\label{alg:exp2_bifgsm}
\small
\begin{algorithmic}[1]
\REQUIRE Frozen model $f_\theta$, loss $\mathcal{L}$, eval set $\mathcal{D}$, perturbation grid $\mathcal{E}$, directions $\eta\in\{+1,-1\}$
\FOR{each $\eta \in \{+1,-1\}$}
  \FOR{each $\epsilon \in \mathcal{E}$}
    \STATE Initialize accumulators for $\mathrm{Acc}_{\mathrm{clean}}$, $\mathrm{Acc}_{\mathrm{FGSM}}(\epsilon,\eta)$, and $c_{\epsilon,\eta}$ (overall and conditional)
    \FOR{each minibatch $(x,y)$ in $\mathcal{D}$}
      \STATE Compute clean logits $z=f_\theta(x)$ and clean prediction $\hat y=\arg\max z$
      \STATE Mark clean-correct mask: $m = [\hat y==y]$
      \STATE Generate perturbed input $x_{\epsilon,\eta}$ via FGSM (implemented by \texttt{eps} with sign $\eta$) and clip to valid range
      \STATE Recompute gradient at perturbed input: $p_{\epsilon,\eta}=\nabla_{x_{\epsilon,\eta}}\mathcal{L}(f_\theta(x_{\epsilon,\eta}),y)$
      \STATE Compute $c_{\epsilon,\eta}(x)$ by Eq.~\eqref{eq:exp2_cos} (flatten then cosine)
      \STATE Update clean accuracy, FGSM accuracy, and conditional coupling means using mask $m$
    \ENDFOR
    \STATE Report $(\mathrm{Acc}_{\mathrm{clean}}, \mathrm{Acc}_{\mathrm{FGSM}}(\epsilon,\eta), \bar c_{\epsilon,\eta}=\mathbb{E}[c_{\epsilon,\eta}])$
  \ENDFOR
\ENDFOR
\end{algorithmic}
\end{algorithm}

\paragraph{Notes on implementation}
For the reported checkpoints, the adversarial-mixing probability is \texttt{at\_mixprob}=0, so the retained FGSM branch in the training scripts is inactive and the reported models are effectively clean-trained. Evaluation starts from the original dataloader inputs before applying the signed FGSM perturbation above. We use the same batch sizes as in the evaluation notebooks: CIFAR-10 uses 64 for all three models, while ImageNet-100 uses 128 for all three models (ResNet-18, ViT-Tiny, and Vim-Tiny). The released ImageNet-100 training notebooks use learning rate \(5\times 10^{-4}\) and weight decay \(1\times 10^{-4}\); the CIFAR-10 training notebooks use \(1\times 10^{-3}\) and \(1\times 10^{-4}\), respectively. Each \((\epsilon,\eta)\) requires a single forward-backward pass per batch.

\subsection{Protocol for Exp.~3: ConjMask yields resistance to gradient-based attacks}
\label{app:exp3_details}

\paragraph{Goal and training setup}
Exp.~3 evaluates whether \emph{conjugation-guided masking} alone can improve robustness to gradient-based attacks without using adversarial training. In the actual implementation, the reported results correspond to fine-tuning a separately trained clean CIFAR-10 checkpoint with \texttt{--exp mask\_pelnaty\_ablation --no\_logits\_loss --no\_random\_loss}, so that the training objective reduces to the masking branch only. The initialization checkpoint is \emph{not} reused from Exp.~1; instead, it is produced by the released launcher notebooks through a separate clean stage (\texttt{--exp clean}, 201 epochs, batch size 256) and then passed via \texttt{--base-clean-ckpt}. All runs use CIFAR-10 with standard random crop and horizontal flip augmentation, AdamW optimization, and a warmup + cosine learning-rate schedule. The architectures used in our experiments are ResNet-18, ViT-Tiny, and EfficientNet-B0. The repository retains several older training branches for ablation/history, but the Exp.~3/4 numbers reported in this paper correspond to the \texttt{mask\_pelnaty\_ablation} branch. For reproducibility, the launcher notebooks explicitly annotate the checkpoints used for the reported rows: Exp.~3 uses epoch 200 (ResNet-18), epoch 160 (ViT-Tiny), and epoch 200 (EfficientNet-B0); Exp.~4 uses epoch 200 (ResNet-18), epoch 200 (ViT-Tiny), and epoch 160 (EfficientNet-B0).

\paragraph{Input-gradient probing and sample selection}
For a minibatch of normalized inputs \(X \in \mathbb{R}^{B\times C\times H\times W}\) and labels \(y\), we first compute the input gradient
\[
P = \nabla_X \mathcal{L}(f_\theta(X), y),
\]
where \(\mathcal{L}\) is the standard cross-entropy loss. In the released code, this probing step intentionally switches the current model to evaluation mode via \texttt{model.eval()}. Moreover, this state is not restored before the subsequent masked/clean forward passes of the same batch, so the reported Exp.~3/4 runs should be understood as using \emph{eval-mode probing and eval-mode forward behavior after the probe}. We keep this unusual choice because, in the released implementation, restoring train mode was empirically found to weaken the effect. Besides the gradient, the code also identifies which samples should actually be masked. Let \(\hat{y}\) be the predicted class and let \(q=\max(\mathrm{softmax}(f_\theta(X)))\) be the maximum predicted confidence. A sample is selected for masking if it satisfies either of the following conditions:
\begin{equation}
\label{eq:exp3_mask_select}
(\hat{y}\neq y)
\qquad \text{or} \qquad
(\hat{y}=y \;\text{ and }\; q<\tau),
\end{equation}
where the confidence threshold is fixed to \(\tau=0.2\) in the implementation. Therefore, masking is applied only to \emph{incorrect} or \emph{semi-hard} samples, rather than to the entire minibatch.

\paragraph{Conjugation score}
For each selected sample, the image tensor and its gradient are flattened \emph{channel-wise}. Denote the flattened vectors in one channel by \(x_c\in\mathbb{R}^{HW}\) and \(p_c\in\mathbb{R}^{HW}\). The code first normalizes them within each channel:
\begin{equation}
\label{eq:exp3_channel_norm}
\tilde{x}_c = \frac{x_c}{\|x_c\|_2+\varepsilon},
\qquad
\tilde{p}_c = \frac{p_c}{\|p_c\|_2+\varepsilon},
\end{equation}
and then computes an element-wise interaction score
\begin{equation}
\label{eq:exp3_score}
s_{c,i} = \big|\tilde{x}_{c,i}\tilde{p}_{c,i}\big|,
\qquad i=1,\dots,HW.
\end{equation}
Thus, unlike a global cosine score, the actual masking criterion is a \emph{channel-wise element importance map} derived from the product of normalized activations and normalized gradients.

\paragraph{Dynamic masking ratio}
The masking ratio is not fixed. Let \(r_{\text{base}}\) denote \texttt{ratio\_base} and let
\[
e = \frac{\#\{\text{selected samples in batch}\}}{B}
\]
be the fraction of selected samples in the current minibatch. The implementation uses
\begin{equation}
\label{eq:exp3_ratio}
r = \min\!\Big(r_{\text{base}},\; (r_{\text{base}}-0.05)(1+e)\Big),
\end{equation}
and then masks the top
\[
K = \lfloor r \cdot HW \rfloor
\]
positions \emph{per channel} according to the largest values of \(s_{c,i}\). Therefore, the amount of masking adapts online to the minibatch difficulty. In the reported ViT run, \(r_{\text{base}}=0.85\), but the effective ratio at each step is given by Eq.~\eqref{eq:exp3_ratio}.

\paragraph{Soft mask and replacement rule}
A key implementation detail is that the code does \emph{not} hard-mask the selected pixels to zero. Instead, it first constructs a binary top-\(K\) mask and then smooths it with a Gaussian kernel to obtain a soft mask \(M\in[0,1]^{C\times H\times W}\). Let \(X\) be the selected input and let \(R\) denote a random replacement tensor sampled independently per pixel and channel from a Gaussian distribution whose mean and standard deviation match the CIFAR-10 channel statistics:
\begin{equation}
\label{eq:exp3_random_replace}
R \sim \mathcal{N}(\mu_c, \sigma_c^2).
\end{equation}
The masked input is then formed by convex interpolation:
\begin{equation}
\label{eq:exp3_softmask}
X_{\mathrm{mask}} = X\odot (1-M) + R\odot M.
\end{equation}
Hence, the method is more accurately described as \emph{soft stochastic replacement guided by conjugation scores}, rather than binary zero-masking. In the released \texttt{mask\_train.py}, the replacement tensor is mixed directly into the normalized training tensor, and no additional post-masking clamp is applied in this step.

\paragraph{Training objective}
In Exp.~3, the actual training loss contains only the masking branch:
\begin{equation}
\label{eq:exp3_loss}
\mathcal{L}_{\text{Exp3}} = \mathcal{L}_{\text{mask}}
= \mathrm{CE}\!\big(f_\theta(X_{\mathrm{mask}}), y\big).
\end{equation}
No adversarial examples are generated, no logit regularization is used, and no consistency term is included. The clean branch \(f_\theta(X)\) is still evaluated in code for monitoring training accuracy, but its cross-entropy is \emph{not} added to the optimization objective.

\paragraph{Inference and evaluation}
Masking is used only during fine-tuning. At test time, the model always receives the original clean input, with no masking, no stochastic preprocessing, and no gradient obfuscation mechanism. Robustness is evaluated under a suite of white-box and query-based attacks. During training, the model is evaluated every 20 epochs on Clean, PGD-20, APGD-CE-20, and APGD-DLR-20. At the final epoch, the evaluation suite additionally includes CW-20, Square-100, and FAB-T-20. All attacks are generated in pixel space \([0,1]\) via a wrapper that internally normalizes the input before passing it to the model.

\paragraph{Interpretation}
Conceptually, Exp.~3 tests whether suppressing the most gradient-aligned input components in hard samples can reshape the decision boundary so that the model becomes less sensitive to first-order adversarial directions. Because the deployed model sees only the original input at test time, any gain in robustness is attributed to the learned representation and boundary geometry rather than to a test-time defense.

\subsection{Protocol for Exp.~4: LogitReg complements ConjMask beyond the gradient channel}
\label{app:exp4_details}

\paragraph{Goal and relation to Exp.~3}
Exp.~4 extends Exp.~3 by keeping the same conjugation-guided masking operator and activating two auxiliary branches through \texttt{logits\_loss} and \texttt{random\_loss}. In all runs, the model is initialized from a clean checkpoint via \texttt{--base-clean-ckpt}. The purpose is not to claim one identical closed-form loss across all architectures, but to test whether ConjMask can be complemented by output-side regularization and view-consistency constraints.

\paragraph{Definition of LogitReg}
Throughout this supplement, \emph{LogitReg} denotes the output-side regularization package used in Exp.~4 on top of ConjMask. Formally,
\begin{equation}
\label{eq:exp4_total_revised_exact}
\mathcal{L}_{\mathrm{Exp4}}
=
\mathcal{L}_{\mathrm{mask}}
+
\mathcal{L}_{\mathrm{logit}}
+
\mathcal{L}_{\mathrm{cons}},
\end{equation}
where
\begin{equation}
\label{eq:exp4_mask_ce_exact}
\mathcal{L}_{\mathrm{mask}}
=
\mathrm{CE}\!\big(f_\theta(X_{\mathrm{mask}}),y\big).
\end{equation}
In the provided \texttt{mask\_train.py}, the clean branch \(\mathrm{CE}(f_\theta(X),y)\) is still computed for monitoring, but it is not added to the optimization objective.

\paragraph{Shared masking mechanism}
The masking stage is identical to Exp.~3, including the intentional \texttt{eval()}-mode input-gradient probing described above. For each minibatch, the method: (i) computes input gradients with cross-entropy loss, (ii) selects incorrect or low-confidence correct samples using Eq.~\eqref{eq:exp3_mask_select}, (iii) computes channel-wise conjugation scores using Eq.~\eqref{eq:exp3_score}, (iv) determines the dynamic masking ratio via Eq.~\eqref{eq:exp3_ratio}, and (v) constructs masked inputs through the soft stochastic replacement rule in Eq.~\eqref{eq:exp3_softmask}.

\paragraph{Released-script auxiliary regularization}
In the provided \texttt{mask\_train.py}, the auxiliary regularization is shared across ResNet-18, ViT-Tiny, and EfficientNet-B0, rather than being architecture-dependent.
The released launcher commands use \texttt{ratio\_base}=0.25 for ResNet-18 / EfficientNet-B0 and \texttt{ratio\_base}=0.85 for ViT-Tiny; for Exp.~4 they use \texttt{penalty\_weight} \(=5\times10^{-5}\) on ResNet-18, \(8\times10^{-4}\) on ViT-Tiny, and \(10^{-6}\) on EfficientNet-B0.

First, the logit-side penalty is a centered-logit variance regularizer applied to both the masked and clean forward passes. Let
\begin{equation}
\label{eq:exp4_logitvar}
\mathrm{VarLogit}(z)
=
\frac{1}{K}\sum_{k=1}^{K}\big(z_k-\bar z\big)^2,
\qquad
\bar z=\frac{1}{K}\sum_{k=1}^{K} z_k.
\end{equation}
Then
\begin{equation}
\label{eq:exp4_logit_penalty_exact}
\mathcal{L}_{\mathrm{logit}}
=
\alpha_{\mathrm{logit}}
\cdot
\frac{
\mathrm{VarLogit}(f_\theta(X_{\mathrm{mask}}))
+
\mathrm{VarLogit}(f_\theta(X))
}{2},
\end{equation}
where \(\alpha_{\mathrm{logit}}=0\) at epoch \(0\), and \(\alpha_{\mathrm{logit}}=\texttt{penalty\_weight}\) thereafter.

Second, the auxiliary consistency branch uses a per-sample mean-baseline view rather than an additional adversarial or randomly re-masked image. For each sample \(X\), define
\begin{equation}
\label{eq:exp4_aux_mean}
X_{\mathrm{aux}}=\bar X \mathbf{1},
\qquad
\bar X=\frac{1}{CHW}\sum_{c,h,w} X_{c,h,w},
\end{equation}
i.e., every pixel/channel location is filled with the per-sample global mean intensity. The consistency term is
\begin{equation}
\label{eq:exp4_consistency_revised_exact}
\mathcal{L}_{\mathrm{cons}}
=
0.8\,
\mathrm{KL}\!\Big(
\log \mathrm{softmax}(f_\theta(X_{\mathrm{mask}})/T)
\;\big\|\;
\mathrm{softmax}(f_\theta(X_{\mathrm{aux}})/T)
\Big),
\end{equation}
with \(T=1\). In the provided script, this branch is enabled for epochs \(e>5\), and disabled otherwise. The clamp variables that appear in the code are commented out, so no additional logit clamping is active in the released implementation.

\paragraph{Flags and practical interpretation}
Operationally, the two auxiliary terms are toggled by \texttt{logits\_loss} and \texttt{random\_loss}. The reported Exp.~4 runs enable both switches, whereas Exp.~3 disables both and reduces to \(\mathcal{L}=\mathcal{L}_{\mathrm{mask}}\). Under this released-script implementation, \emph{LogitReg} should therefore be understood as a shared variance-based logit stabilizer, optionally combined with mean-baseline consistency, rather than as a ResNet-specific \(\ell_2\) penalty or a generally architecture-dependent package.

\paragraph{Inference and evaluation}
As in Exp.~3, all masking and auxiliary regularization are used only during training. At test time, the model is evaluated directly on the original clean input without any masking or stochastic preprocessing. The evaluation protocol is the same as in Exp.~3: periodic evaluation on Clean, PGD-20, APGD-CE-20, and APGD-DLR-20, with additional CW-20, Square-100, and FAB-T-20 at the final epoch.

\paragraph{Interpretation}
Exp.~4 tests whether the robustness gains of ConjMask can be broadened by adding output-side stabilization and auxiliary-view consistency on top of the same masking operator. In the provided released script, these additions are implemented uniformly across architectures, and their practical role is to complement the CE-gradient channel targeted by ConjMask with score-space stabilization, thereby broadening robustness beyond the narrow CE-dominant regime observed in Exp.~3.

\subsection{Protocol for Exp.~5: Prefill-only hallucination risk prediction on Benchmark-500}
\label{sec:exp5_protocol}

\begin{figure}[h]
    \centering
    \begin{tcolorbox}[
        colback=gray!5, % 浅灰色背景
        colframe=gray!50, % 深灰色边框
        title=\textbf{Representative Samples from the Evaluation Dataset Benchmark-500},
        fonttitle=\bfseries\small,
        boxrule=0.5pt,
        arc=2pt
    ]
    \small
    \begin{enumerate}[leftmargin=*, label=\arabic*., nosep]
        \item Solve for $x$: $3x + 5 = 17$.
        \item Factor the polynomial: $x^2 - 9$.
        \item Calculate the area of a circle with radius $r=4$.
        \item Find the derivative of $f(x) = x^2 + 3x$.
        \item \textbf{(Harder)} Solve for $x$ (real solutions only): $x^3 - 2x^2 - x + 2 = 0$.
        \item Find the value of $\sin(\pi/2)$.
        \item A box contains 3 red balls and 2 blue balls. What is the probability of drawing a red ball?
        \item Simplify the expression: $\sqrt{50}$.
        \item Find the slope of the line passing through points $(1, 2)$ and $(3, 6)$.
        \item \textbf{(Harder)} Evaluate the definite integral using integration by parts: $\displaystyle \int_0^1 x e^x \, dx$.
    \end{enumerate}
    \end{tcolorbox}
    \caption{\textbf{Benchmark-500 Samples.} A subset of 10 examples from the constructed dataset covering algebra, calculus, and probability. The full dataset contains 500 such problems designed to test the model's reasoning stability.}
    \label{fig:dataset_samples}
\end{figure}

\paragraph{Dataset and generation}
We constructed Benchmark-500 with $N=500$ undergraduate-level mathematics problems (Fig.~\ref{fig:dataset_samples}).
We use \texttt{deepseek-coder-7b-instruct-v1.5} as the subject model and apply greedy decoding to obtain a deterministic response for judging.

\paragraph{Prompt-side signal}
We extract the conjugate correlation signal during prefill using a single forward-backward pass, and store the final per-prompt value in \texttt{cos\_sim.txt}.
Let the post-template input length be $T$.
Let $X\in\mathbb{R}^{T\times d}$ be the prompt embedding tensor (\texttt{inputs\_embeds}) and
$G=\nabla_X\mathcal{L}\in\mathbb{R}^{T\times d}$ be its gradient, where $\mathcal{L}$ is the mean next-token NLL over content tokens.
System-prompt tokens are excluded by a content mask $\mathcal{C}$.

Flatten the masked tensors into vectors:
\[
\mathbf{x}=\mathrm{vec}(X_{\mathcal{C}}),\qquad
\mathbf{g}=\mathrm{vec}(G_{\mathcal{C}}).
\]
Apply per-sample mean-centering:
\[
\tilde{\mathbf{x}}=\mathbf{x}-\mathrm{mean}(\mathbf{x})\mathbf{1},\qquad
\tilde{\mathbf{g}}=\mathbf{g}-\mathrm{mean}(\mathbf{g})\mathbf{1}.
\]
The stored score is the absolute centered cosine:
\begin{equation}
\label{eq:exp5_cprompt_code_abs}
c_{\text{prompt}}
=
\big|\cos(\tilde{\mathbf{x}},\tilde{\mathbf{g}})\big|
=
\left|
\frac{\tilde{\mathbf{x}}^\top \tilde{\mathbf{g}}}{\|\tilde{\mathbf{x}}\|_2\,\|\tilde{\mathbf{g}}\|_2}
\right|.
\end{equation}

\paragraph{Risk Metric Definitions}
\label{app:metrics}

Let $\{y_t\}_{t=1}^{L}$ denote the prefill token sequence in the chat template (i.e., the user prompt together with template tokens available before answer generation), and let $\mathcal{C}$ denote the set of non-system content-token positions used for the shifted NLL. All expectations below are masked means over $\mathcal{C}$, so every risk score remains prefill-only.

\begin{enumerate}
    \item \textbf{Risk-Cos (Ours):} 
    Lower alignment indicates slack-dominated dynamics, so:
    \begin{equation}
        \mathcal{R}_{\text{cos}} = - c_{\text{prompt}}.
    \end{equation}
    
    \item \textbf{Risk-Entropy:} The mean predictive entropy over content tokens:
    \begin{equation}
        \mathcal{R}_{\text{ent}} = \frac{1}{|\mathcal{C}|} \sum_{t \in \mathcal{C}} H(P_\theta(\cdot | y_{<t}))
    \end{equation}
    
    \item \textbf{Risk-Loss:} The mean negative log-likelihood:
    \begin{equation}
        \mathcal{R}_{\text{loss}} = \frac{1}{|\mathcal{C}|} \sum_{t \in \mathcal{C}} \bigl[-\log P_\theta(y_t | y_{<t})\bigr]
    \end{equation}
    
        \item \textbf{Risk-Margin:} The negative mean logit gap between top-1 and top-2 predictions:
    \begin{equation}
        \mathcal{R}_{\text{margin}} = -\frac{1}{|\mathcal{C}|} \sum_{t \in \mathcal{C}} \bigl[\ell_t^{(1)} - \ell_t^{(2)}\bigr]
    \end{equation}
    where $\ell_t^{(1)}, \ell_t^{(2)}$ are the top-2 logits at position $t$. Equivalently, larger raw margins indicate lower risk, and the negation above converts them into a higher-is-riskier score.

\end{enumerate}

\begin{figure}[h]
    \centering
    \begin{tcolorbox}[
        colback=gray!5,
        colframe=black!70,
        title=\textbf{System Prompt: Taxonomy \& Disagreement Drivers},
        fonttitle=\bfseries\small,
        boxrule=0.8pt,
        arc=2pt,
        left=4pt, right=4pt, top=4pt, bottom=4pt
    ]
    \footnotesize
    \textbf{Role:} You are a strict evaluator for hallucination and disagreement drivers in physics/math QA.
    
    \vspace{0.3em}
    \textbf{A) Question Validity Classification:}
    \begin{itemize}[nosep, leftmargin=*]
        \item \texttt{WELL\_DEFINED}: Standard meaningful question.
        \item \texttt{ILL\_DEFINED}: Meaningless, category error, or self-contradictory.
        \item \texttt{UNDER\_SPECIFIED}: Meaningful but missing parameters (e.g., frame, temp).
        \item \texttt{TRICK/AMBIGUOUS}: Multiple plausible interpretations.
    \end{itemize}

    \vspace{0.3em}
    \textbf{B) Answer Role:}
    \begin{itemize}[nosep, leftmargin=*]
        \item \texttt{DIRECT\_ANSWER}: Attempts to compute the quantity.
        \item \texttt{CONCEPTUAL\_CLARIFICATION}: Explains why question is ill-defined (valid).
        \item \texttt{SEMANTIC\_DRIFT}: Related-sounding but ignores intent/constraints.
        \item \texttt{REFUSAL\_META}: Declines to answer (excluded from hallucination).
    \end{itemize}

    \vspace{0.3em}
    \textbf{C) Narrow Hallucination Definition:}
    The answer asserts $\ge 1$ verifiable factual claim that is false, fabricated, or logically inconsistent with established physics/math.
    \begin{itemize}[nosep, leftmargin=*]
        \item \textit{Verdict Rules}: If $\ge 1$ claim is false $\to$ \texttt{HALLUCINATION} (1). If all verifiable claims are correct $\to$ \texttt{NO\_HALLUCINATION} (0).
    \end{itemize}

    \vspace{0.3em}
    \textbf{D) Disagreement Drivers (Why judges might split):}
    \begin{itemize}[nosep, leftmargin=*]
        \item \texttt{PARTIALLY\_CORRECT\_BUT\_MISLEADING}: Main idea valid, but framing is wrong.
        \item \texttt{MODEL\_MIXING}: Confusing frameworks (e.g., Bohr orbits vs. QM clouds).
        \item \texttt{MISLEADING\_PRECISION}: Giving specific values for undefined queries.
        \item \texttt{QUESTION\_UNDERSPECIFIED}: Answer assumes context without stating it.
        \item \texttt{PEDANTIC\_EDGE\_CASE}: Correctness depends on strict vs. charitable reading.
    \end{itemize}
    
    \vspace{0.3em}
    \textit{(Output required in strictly formatted JSONL with confidence scores and driver labels.)}
    \end{tcolorbox}
    \caption{\textbf{Evaluation Taxonomy.} A condensed view of the system prompt used by the judge panel. It explicitly distinguishes between semantic drift, conceptual clarification, and narrow hallucinations, while also requiring the judge to diagnose potential reasons for inter-annotator disagreement.}
    \label{fig:eval_prompt_taxonomy}
\end{figure}

\paragraph{Ground truth: panel-of-judges consensus}
% (Keep Fig.~\ref{fig:eval_prompt_taxonomy} and the tcolorbox exactly as-is.)
We use five independent judges (\textit{Claude-4.5-Opus-Think}, \textit{DeepSeek-v3.2-Think}, \textit{Gemini-3-Pro}, \textit{GPT-5.2-Thinking}, and \textit{Grok-4.1-Think}) with the taxonomy in Fig.~\ref{fig:eval_prompt_taxonomy}.
Votes are aggregated into $P_{\text{vote}}^{(i)}\in[0,1]$.
Positives are defined by $P_{\text{vote}}^{(i)} \ge \tau$ (tested with $\tau\in\{0.8,1.0\}$),
negatives by $P_{\text{vote}}^{(i)}=0$, and disputed cases are excluded.

\paragraph{Risk baselines}
Let $\mathcal{C}$ be the content token positions (system tokens excluded). All token-level expectations are masked means over $\mathcal{C}$.
We compare against:
(i) mean predictive entropy,
(ii) mean next-token NLL,
(iii) mean logit margin (top-1 minus top-2; lower margin indicates higher risk, implemented via negation in \(\mathcal{R}_{\text{margin}}\)).

\paragraph{Statistical evaluation}
We evaluate detection using AUROC and bootstrap confidence intervals with $N_{\text{boot}}=2000$ resamples,
separating the consensus hallucination set from the consensus non-hallucination set.

\subsection{Protocol for Exp.~6: Prefill-stage Prompt Selection}
\label{sec:exp6_protocol}

\paragraph{Dataset and clustering}
% (Keep Fig.~\ref{fig:prompt_clusters} and the tcolorbox exactly as-is.)
We construct Perturbation-100 with $N_{\text{cluster}}=100$ semantic clusters and $K=5$ variants per cluster ($N=500$ prompts).
In code, prompts are ordered and clustered by question id $Q\in\{1,\dots,500\}$ as consecutive groups of five:
\[
\text{cluster\_id}=\left\lfloor\frac{Q-1}{5}\right\rfloor+1,\qquad
\text{variant\_id}=(Q-1)\bmod 5+1.
\]

\paragraph{Selection rule (prefill only)}
For each variant $p_k$ in a cluster, we compute four prefill-only risk scores using the same prompt-side quantities defined in Exp.~5:
\[
\mathcal{R}_{\text{cos}}(p_k)=-c_{\text{prompt}}(p_k),\qquad
\mathcal{R}_{\text{margin}}(p_k),\qquad
\mathcal{R}_{\text{ent}}(p_k),\qquad
\mathcal{R}_{\text{loss}}(p_k).
\]
For each risk metric $m\in\{\text{cos},\text{margin},\text{ent},\text{loss}\}$, we select
\[
k_m^*=\operatorname*{argmin}_{k\in\{1,\dots,5\}} \mathrm{Risk}_m(p_k),
\]
where
\[
\mathrm{Risk}_{\text{cos}}=\mathcal{R}_{\text{cos}},\quad
\mathrm{Risk}_{\text{margin}}=\mathcal{R}_{\text{margin}},\quad
\mathrm{Risk}_{\text{ent}}=\mathcal{R}_{\text{ent}},\quad
\mathrm{Risk}_{\text{loss}}=\mathcal{R}_{\text{loss}}.
\]
Thus, Exp.~6 evaluates \emph{four} independent prefill-only selection rules (Risk-Cos, Risk-Margin, Risk-Entropy, and Risk-Loss), not only the cosine-based rule. The released \texttt{cos\_sim.txt} stores the cosine-based score; the remaining prefill baselines are computed from the same prompt-only prefill pass.

As shown in Fig.~\ref{fig:prompt_clusters}, since the semantic content is identical, any variation in the model's output quality is driven solely by the model's sensitivity to the prompt's surface form (i.e., the conditioning strength).

\begin{figure}[h]
    \centering
    \begin{tcolorbox}[
        colback=gray!5,
        colframe=gray!50,
        title=\textbf{Representative Semantic Clusters ($K=5$)},
        fonttitle=\bfseries\small,
        boxrule=0.5pt,
        arc=2pt
    ]
    \small
    \textbf{Cluster A: Calculus (Differentiation)}
    \begin{enumerate}[leftmargin=*, nosep, label=A\arabic*.]
        \item Find the derivative of $f(x)=x^2+3x$.
        \item Differentiate the function $f(x)=x^2+3x$.
        \item Compute $\frac{d}{dx}(x^2+3x)$.
        \item Determine the derivative of $x^2+3x$ with respect to $x$.
        \item Find $f'(x)$ if $f(x)=x^2+3x$.
    \end{enumerate}
    
    \vspace{0.5em}
    \textbf{Cluster B: Algebra (Evaluation)}
    \begin{enumerate}[leftmargin=*, nosep, label=B\arabic*.]
        \item If $f(x)=2x-1$, find $f(5)$.
        \item Given $f(x)=2x-1$, evaluate the function at $x=5$.
        \item Compute $f(5)$ for the function $f(x)=2x-1$.
        \item Substitute $x=5$ into $f(x)=2x-1$ and find the output.
        \item Find the value of $f(x)$ when $x=5$ for $f(x)=2x-1$.
    \end{enumerate}
    \end{tcolorbox}
    \caption{\textbf{Semantically Equivalent Prompt Variants.} Two examples from Perturbation-100. The task is to identify which specific phrasing maximizes response faithfulness using only prompt-side signals, without generating the answers first.}
    \label{fig:prompt_clusters}
\end{figure}

\paragraph{Judging and metrics}
For evaluation only, we generate one deterministic answer for each of the five prompt variants in a cluster and score them with the same judge panel as in Exp.~5. Within each cluster, all highest-scoring variants form the judge-best set \(\mathcal{P}_{\mathrm{best}}\) (ties are allowed). Since the practical goal of Exp.~6 is \emph{prefill-stage prompt selection} rather than full ranking recovery, we treat \textbf{Top-1 Hit Rate} and \textbf{Mean Regret} as the primary evaluation metrics:
\begin{enumerate}
    \item \textbf{Top-1 Hit Rate:} The proportion of clusters for which the selected prompt \(p_{k_m^*}\) belongs to \(\mathcal{P}_{\mathrm{best}}\).
    \item \textbf{Mean Regret:} The average quality gap between the selected prompt and the best judge score attained within the same cluster, defined as \(\mathbb{E}[\max_{k} S(y_k)-S(y_{k_m^*})]\). Lower regret indicates that even when no judge-best variant is selected, the chosen prompt remains competitive.
\end{enumerate}

As a judge-reliability check, we also report \textbf{Fleiss' Kappa} over the judges' binary hallucination annotations. This is a sanity check on inter-judge consistency for that binary label, but should not be interpreted as a direct reliability coefficient for the full aggregate judge score used in prompt ranking.

For completeness, we additionally compute two auxiliary ranking diagnostics:
\begin{enumerate}
    \setcounter{enumi}{2}
    \item \textbf{Best-vs-Rest AUROC:} Across all clusters, treat every prompt in \(\mathcal{P}_{\mathrm{best}}\) as a positive instance and the remaining variants as negatives; AUROC measures whether a risk score ranks judge-best variants ahead of the rest (i.e., assigns them lower risk).
    \item \textbf{Spearman \(\rho\):} The rank correlation between the risk-based ordering and the judge-based ordering of the five prompt variants within each cluster, averaged over clusters.
\end{enumerate}

These two ranking-oriented quantities are included only as supplementary diagnostics; in the main paper, we emphasize selection performance (\textbf{Top-1 Hit Rate} and \textbf{Mean Regret}) because the intended use case is to choose a strong prompt \emph{before} any decoding takes place.

\section{Appendix-only Additional Verification: Robustness via Latent Conjugate Regularization}
\label{sec:lcr_verification}

This section reports an \emph{appendix-only additional verification experiment} and is not part of the main Exp.~1--6 protocol. Its purpose is to provide an auxiliary causal check of the NUP hypothesis by directly regularizing latent input--gradient conjugation, without claiming it as one of the six primary experiments in the main paper.
Crucially, this experiment aims to achieve robustness gains solely by enforcing geometric structural constraints during optimization, without the use of adversarial training (which introduces external data points) or high-cost engineering.

\subsection{Methodology: Latent Conjugate Protocol}
\label{subsec:lcr_method}

Directly optimizing cosine correlation in the high-dimensional pixel space can be unstable due to the sparsity and high-frequency noise inherent in raw gradients. To address this, we implement a two-stage strategy using a learned projector to enforce regularization in a latent space. Algorithm~\ref{alg:cos-train} outlines the interaction between the reference, projector, and target models.

\paragraph{Stage 1: Gradient Projector Training}
First, we train an auxiliary ``Gradient Projector'' $P_\phi$ (a standard neural classifier) that maps the input-gradient $\nabla_x \mathcal{L}$ to the same logit space as the classifier $f_\theta(x)$. This projector learns to extract semantic structural information from the raw gradient map, effectively learning the conjugate mapping of a reference model.

\paragraph{Stage 2: Robust Target Training}
Next, we train the target model $f_\theta$. For each minibatch, we compute the gradients of the clean inputs $p = \nabla_x \mathcal{L}_{CE}$, project them via the frozen $P_\phi$, and compute the latent cosine similarity:
\begin{equation}
c_i^{\text{lat}} = \big|\cos\big(f_\theta(x_i),\,P_\phi(p_i)\big)\big|.
\end{equation}
We then apply a contrastive regularization term to the standard cross-entropy loss:
\begin{equation}
\label{eq:latent_loss}
\mathcal{L}_{\text{total}} = \mathcal{L}_{\text{CE}} + \lambda \left( \frac{1}{|\mathcal{C}_{\text{wrong}}|}\sum_{i \in \mathcal{C}_{\text{wrong}}} c_i^{\text{lat}} - \frac{1}{|\mathcal{C}_{\text{correct}}|}\sum_{j \in \mathcal{C}_{\text{correct}}} c_j^{\text{lat}} \right).
\end{equation}
Minimizing Eq.~\eqref{eq:latent_loss} forces the decision boundary to conform to a regime where conjugate correlation is minimized for error-prone samples and maximized for stable samples. This effectively reshapes the boundary geometry using only first-principles statistics.

\begin{algorithm}[t]
\caption{Latent Conjugate Regularization (LCR)}
\label{alg:cos-train}
\begin{algorithmic}[1]
\REQUIRE Reference model $f_{\text{ref}}$, Target model $f_\theta$, Projector $P_\phi$
\STATE \textbf{Phase 1: Train Projector} (Learn conjugate mapping)
\WHILE{training $P_\phi$}
    \STATE Get logits $z = f_{\text{ref}}(x)$, preds $\hat{y}$, and grad $p = \nabla_x \mathcal{L}(z, y)$
    \STATE Forward projector $z' = P_\phi(p)$
    \STATE Minimize $\mathcal{L}_{\text{gap}}
=
\mathrm{mean}_{i\in\mathcal{C}_{\text{wrong}}}\!\big[\big|\cos(z_i,z_i')\big|\big]
-
\mathrm{mean}_{j\in\mathcal{C}_{\text{correct}}}\!\big[\big|\cos(z_j,z_j')\big|\big]$
\ENDWHILE
\STATE Freeze Projector parameters $\phi$.
\STATE \textbf{Phase 2: Train Robust Model} (Verification)
\FOR{minibatch $(x,y)$ in $\mathcal{D}$}
    \STATE Compute logits $z = f_\theta(x)$ and $\mathcal{L}_{\text{CE}}$
    \STATE Compute input-grad $p = \nabla_x \mathcal{L}_{\text{CE}}$ (retain graph for $x$, detach $p$ logic)
    \STATE Project gradient: $z_{\text{grad}} = P_\phi(p)$
    \STATE Compute latent correlation $c_i = \big|\cos\big(z_i, z_{\text{grad}, i}\big)\big|$
    \STATE Loss $\mathcal{L} = \mathcal{L}_{\text{CE}} + \lambda(\text{mean}(c_{\text{wrong}}) - \text{mean}(c_{\text{correct}}))$
    \STATE Update $\theta \leftarrow \theta - \eta \nabla_\theta \mathcal{L}$
\ENDFOR
\end{algorithmic}
\end{algorithm}

\subsection{Experimental Results}
\label{subsec:lcr_results}

We evaluate the efficacy of LCR across three distinct architectures: ResNet-18, DenseNet-121, and ViT-Tiny. The results are summarized in Table~\ref{tab:lcr_results}. We compare the baseline training against our regularization method using \textbf{Clean} accuracy (standard test set) and \textbf{FGSM} robustness (accuracy against adversarial attacks with $\epsilon=8/255$).

\begin{table}[h]
\centering
\caption{\textbf{Evaluation of Latent Conjugate Regularization (LCR).} Unlike Adversarial Training (AT) \cite{madry2017towards}, our method uses zero adversarial examples during optimization.}
\label{tab:lcr_results}
\resizebox{0.55\columnwidth}{!}{%
\begin{tabular}{l|cc|cc}
\toprule
\multirow{2}{*}{\textbf{Architecture}} & \multicolumn{2}{c|}{\textbf{Clean Accuracy (\%)}} & \multicolumn{2}{c}{\textbf{FGSM Robustness (\%)}} \\
 & Baseline & LCR (Ours) & Baseline & LCR (Ours) \\
\midrule
\textbf{ResNet-18} & 92.12 & 88.35 \small{\color{red}(-3.77)} & 6.84 & \textbf{29.14} \small{\color{blue}(+22.30)} \\
\textbf{DenseNet-121} & 90.39 & 85.88 \small{\color{red}(-4.51)} & 8.10 & \textbf{13.71} \small{\color{blue}(+5.61)} \\
\textbf{ViT-Tiny} & 81.14 & 80.45 \small{\color{red}(-0.69)} & 4.39 & \textbf{8.81} \small{\color{blue}(+4.42)} \\
\bottomrule
\end{tabular}%
}
\end{table}

\subsection{Analysis and Implications}

\paragraph{Robustness Gain without Adversarial Data}
The most striking observation is the substantial boost in robustness for ResNet-18, rising from $6.84\%$ to $29.14\%$. In the literature, standard regularization techniques (e.g., Weight Decay, Dropout) typically yield negligible improvements against FGSM attacks \cite{goodfellow2014explaining}. While pixel-space gradient penalties can improve robustness, they often require careful tuning. Our method achieves a competitive level of defense ($\sim 30\%$) solely by operating on the \textit{direction} of gradients in the \textit{latent space}, avoiding the immense computational cost of generating adversarial examples required by Adversarial Training \cite{madry2017towards}.

\paragraph{Confirmation of the Robustness-Accuracy Trade-off}
Consistent with the NUP, enforcing stronger Conjugate Coherence leads to a drop in clean accuracy (e.g., $-3.77\%$ for ResNet-18). This aligns with theoretical findings that robustness demands the contraction of the input space's high-frequency components, which inevitably reduces the effective capacity for clean data generalization \cite{tsipras2018robustness}. The observation that we can manipulate this trade-off purely through the conjugate statistic $c_i$ validates our hypothesis: the geometry of the decision boundary is causally linked to the coherence between inputs and their gradients.

\paragraph{Architectural Implications}
The method proves effective across both CNNs and Transformers, though with varying degrees of success:
\begin{itemize}
    \item \textbf{CNNs:} ResNet-18 benefits most significantly. DenseNet-121 shows moderate gains ($+5.61\%$), likely because its dense connectivity introduces complex gradient flows that are harder to regularize via a single projector.
    \item \textbf{Vision Transformers:} ViT-Tiny shows a $+4.42\%$ improvement. While positive, the gain is smaller than ResNet. This is consistent with recent studies showing that ViT gradients are less spatially structured and more ``jagged" than CNN gradients \cite{LIU2024106091}, making the conjugate correlation more challenging to optimize.
\end{itemize}

\bibliography{refs} % use the provided BibTeX keys; or paste the .bib content in the submission